\documentclass{article}
\usepackage{arxiv}
\usepackage{natbib}
\usepackage{upgreek}
\usepackage[utf8]{inputenc} % allow utf-8 input
\usepackage[T1]{fontenc}    % use 8-bit T1 fonts
\usepackage[hidelinks]{hyperref}
\usepackage{xurl}
\hypersetup{breaklinks=true}
\usepackage{booktabs}       % professional-quality tables
\usepackage{amsfonts}       % blackboard math symbols
\usepackage{nicefrac}       % compact symbols for 1/2, etc.
\usepackage{microtype}      % microtypography
\usepackage{xcolor}         % colors
\usepackage{enumitem}
\usepackage{multirow}
\usepackage{amsmath}
\usepackage{amssymb}
\usepackage{dsfont}
\usepackage{subcaption}
\usepackage{graphicx}
\usepackage[export]{adjustbox}
\usepackage{amsthm}
\usepackage{thm-restate}
\usepackage[linesnumbered,ruled,vlined]{algorithm2e}
\usepackage[capitalize,noabbrev]{cleveref}
\usepackage{bm}
\usepackage{wrapfig}

\newcommand{\smallblacksquare}{\scalebox{0.5}{$\blacksquare$}}

\newcommand{\nump}{p}
\newcommand{\numuser}{m}
\newcommand{\numdim}{d}
\newcommand{\userind}{i}
\newcommand{\EE}{\ensuremath{\mathbf{E}}}

\newcommand{\RR}{\mathbb{R}}
\newcommand{\Iv}{\mathbf{I}}
\newcommand{\gv}{\mathbf{g}}
\newcommand{\dv}{\mathbf{d}}
\newcommand{\sv}{\mathbf{s}}
\newcommand{\wv}{\mathbf{w}}
\newcommand{\zv}{\mathbf{z}}
\newcommand{\xv}{\mathbf{x}}
\newcommand{\fv}{\mathbf{f}}
\newcommand{\Dc}{\mathcal{D}}
\newcommand{\lambdav}{\bm{\lambda}}
\newcommand{\epsilonv}{\bm{\epsilon}}
\newcommand{\xiv}{\bm{\xi}}
\newcommand{\pa}[1]{\mathsf{A}_{#1}}
\newcommand{\prox}[1]{\mathsf{P}_{#1}}
\newcommand{\env}[1]{\mathsf{M}_{#1}}
\newcommand{\inner}[2]{\left\langle{{#1},{#2}}\right\rangle}
\newcommand{\where}{\mbox{where}}
\newcommand{\zero}{\mathbf{0}}

\newtheorem{lemma}{Lemma}

\newtheorem{definition}{Definition}

\newcommand{\argmin}{\mathop{\textrm{argmin}}}
\newcommand{\fedavg}{\texttt{FedAvg}\xspace}
\newcommand{\MOM}{\texttt{MoM}\xspace}
\newcommand{\fedavgn}{\texttt{FedAvg-n}\xspace}
\newcommand{\FL}{\texttt{FL}\xspace}
\newcommand{\FA}{\fedavg}
\newcommand{\fedprox}{\texttt{FedProx}\xspace}
\newcommand{\FP}{\fedprox}
\newcommand{\fedmom}{\texttt{FedMom}\xspace}
\newcommand{\fedMGDA}{\texttt{FedMGDA}\xspace}
\newcommand{\fedMGDAn}{\texttt{FedMGDA+}\xspace}

\newcommand{\fedMGDAprox}{\texttt{MGDA-Prox}\xspace}
\newcommand{\fedMGDAproxn}{\texttt{MGDA-Prox}\xspace}
\newcommand{\FM}{\fedMGDA}
\newcommand{\FMp}{\texttt{FedMGDA+}\xspace}

\newcommand{\qffl}{\texttt{q-FedAvg}\xspace}

\newcommand{\afl}{\texttt{AFL}\xspace}
\newcommand{\fedma}{\texttt{FedMA}\xspace}
\newcommand{\pfnm}{\texttt{PFNM}\xspace}
\newcommand{\B}{\boldmath}

\newtheorem{manualtheoreminner}{Theorem}
\newenvironment{manualtheorem}[1]{%
  \renewcommand\themanualtheoreminner{#1}%
  \manualtheoreminner
}{\endmanualtheoreminner}

\newcommand{\vs}{\textit{vs.}}
\newcommand{\st}{\textit{s.t.}}
\newcommand{\ie}{\textit{i.e.}}
\newcommand{\eg}{\textit{e.g.}}
\newcommand{\etc}{\textit{etc.}}

\title{Federated Learning Meets Multi-objective Optimization}

\author{Zeou~Hu, Kiarash~Shaloudegi, Guojun~Zhang, and~Yaoliang~Yu% <-this % stops a space
\thanks{Zeou Hu and Yaoliang Yu are with the Cheriton School of Computer Science, University of Waterloo, Waterloo, ON, N2L 3G1. E-mail: \{zeou.hu,yaoliang.yu\}@uwaterloo.ca; Kiarash Shaloudegi is with Amazon Advertising. Work was done while at Huawei Noah Ark's Lab,
Montreal, QC, H3N 1X9. E-mail: kiarashs@amazon.com; Guojun Zhang is with Huawei Noah Ark's Lab, Montreal, QC, H3N 1X9. E-mail: guojun.zhang@huawei.com}
}

\begin{document}

\maketitle

\begin{abstract}
Federated learning has emerged as a promising, massively distributed way to train a joint deep model over large amounts of edge devices while keeping private user data strictly on device. In this work, motivated from ensuring fairness among users and robustness against malicious adversaries, we formulate federated learning as multi-objective optimization and propose a new algorithm \FMp that is guaranteed to converge to Pareto stationary solutions. \FMp is simple to implement, has fewer hyperparameters to tune, and refrains from sacrificing the performance of any \emph{participating} user. We establish the convergence properties of \FMp and point out its connections to existing approaches. Extensive experiments on a variety of datasets confirm that \FMp compares favorably against state-of-the-art.
\end{abstract}
%\begin{IEEEkeywords}

\section{Introduction}\label{sec:intro}

D{eep} learning has achieved impressive successes on a number of domain applications, thanks largely to innovations on algorithmic and architectural design, and equally importantly to the tremendous amount of computational power one can harness through GPUs, computer clusters and dedicated software and hardware. Edge devices, such as smart phones, tablets, routers, car devices, home sensors, \etc, due to their ubiquity and moderate computational power, impose new opportunities and challenges for deep learning. On the one hand, edge devices have direct access to privacy sensitive data that users may be reluctant to share (with say data centers), and they are much more powerful than their predecessors, capable of conducting  a significant amount of on-device computations. On the other hand, edge devices are largely heterogeneous in terms of capacity, power, data, availability, communication, memory, \etc, posing new challenges beyond conventional in-house training of machine learning models. Thus, a new paradigm, known as federated learning (\FL) \citep{McMahanMRHA17} that aims at harvesting the prospects of edge devices, has recently emerged. Developing new \FL algorithms and systems on edge devices has since become a hot research topic in machine learning.

From the beginning of its birth, \FL has close ties to  conventional  distributed optimization. However, \FL emerged from the pressing need to address news challenges in the mobile era that  existing distributed optimization algorithms were not designed for \emph{per se}. We mention the following characteristics of \FL that are most relevant to our work, and refer to the excellent surveys \citep{LiSTS19,YangLCT19, Kairouz2019s} and the references therein for more challenges and applications in \FL.
\begin{itemize}%[leftmargin=*,topsep=0pt, noitemsep] 
\item \textbf{Non-IID:} Each user's data  can be distinctively different from every other user's, violating the standard iid assumption in statistical learning and posing significant difficulty in formulating the goal in precise mathematical terms \citep{MohriSS19}. The distribution of user data is often severely unbalanced.
\item \textbf{Limited communication:} Communication between each user and a central server is constrained by network bandwidth, device status, user participation incentive, \etc, demanding a thoughtful balance between computation (on each user device) and communication.
\item \textbf{Privacy:} Protecting user (data) privacy is of uttermost importance in \FL. It is thus not possible to share user data (even to a cloud arbitrator), which adds another layer of difficulty in addressing the previous two challenges.
\item \textbf{Fairness:} As argued forcibly in recent work (\eg, \citealt{MohriSS19,LiSBS20}), ensuring fairness among users has become another serious goal in \FL, as it largely determines users' willingness to participate and ensures some degree of robustness against malicious user manipulations.
\item \textbf{Robustness:} \FL algorithms are eventually deployed in the wild hence subject to malicious attacks. Indeed, adversarial attacks (\eg, \citealt{BagdasaryanVHES20,SunKSM19,BhagojiCMC19}) have been constructed recently to reveal vulnerabilities of \FL systems against malicious manipulations at the user side.
\end{itemize}

In this work, motivated from the last two challenges above, \ie fairness and robustness, we propose a new algorithm \FMp that complements and improves existing \FL systems. \FMp is based on multi-objective optimization and is guaranteed to converge to Pareto stationary solutions. \FMp is simple to implement, has fewer hyperparameters to tune, and most importantly \emph{refrains from sacrificing the performance of any participating user}. We demonstrate the superior performance of \FMp under a variety of metrics including accuracy, fairness, and robustness. 

We summarize our contributions as follows:
\begin{itemize}%[leftmargin=*, topsep=0pt, noitemsep]
\item In \S\ref{sec:bg}, based on the proximal average we provide a novel, unifying and revealing interpretation of existing \FL practices.
\item In \S\ref{sec:MOM}, we summarize some background on multi-objective optimization and point out its connections to existing \FL algorithms. We believe this new perspective will yield more fruitful exchanges between the two fields in the future.
\item In \S\ref{sec:tech}, we propose \FMp that complements existing \FL systems while taking robustness and fairness explicitly into its algorithmic design. We prove that \FMp  converges to a Pareto stationary solution under mild assumptions. 
\item In \S\ref{sec:exp}, we perform extensive experiments to validate the competitiveness of \FMp under a variety of desirable metrics, and to illustrate the respective pros and cons of our and alternative algorithms.
\end{itemize}
We discuss more related work in \S\ref{sec:related} and we conclude in \S\ref{sec:con} with some future directions. 

To facilitate reproducibility, we have released our code at: \url{https://github.com/watml/Fed-MGDA}.

\section{Related Work}
\label{sec:related}
%\vspace{-0.3em}
In this section we give a brief review of some recent work that is directly related to ours and put our contributions in context. To start with, \citet{McMahanMRHA17} proposed the first \FL algorithm known as ``Federated Averaging'' (a.k.a., \fedavg{}), which is a  synchronous update scheme that proceeds in several rounds.  At  each round, the  central server sends the current global model to a subset  of users, each of which then uses its respective local data  to update the received model.  
Upon receiving the updated local models from users, the server performs aggregation, such as simple averaging, to update the global model. 
For more discussion on different averaging schemes, see \citet{LiHYWZ20}. \citet{LiSZSTS20} extended \fedavg{}  to better deal with non-i.i.d.~distribution of data, by adding a  ``proximal regularizer''  to the local loss functions and minimizing the Moreau envelope function for each user.  The resulting algorithm \fedprox{}, as pointed out in \S\ref{sec:bg}, is a randomized version of the proximal average algorithm in \citet{Yu13a} and reduces to \fedavg{} when  regularization diminishes.
%, choosing the local solver to be SGD, and using a fixed number of local epochs across devices. 

Analysing \fedavg{} has been a challenging task due to its flexible updating scheme, partial user participation, and non-iid distribution of client data \citet{LiSZSTS20}. The first  theoretical analysis of \fedavg{} for strongly convex and smooth problems with iid and non-iid data appeared in \citet{Stich19} and \citet{LiHYWZ20}, respectively, where the effect of different sampling and averaging schemes on the convergence rate of \fedavg{} was also investigated, leading to the conclusion that such effect becomes particularly important when the dataset is unbalanced and non-iid distributed. In \citet{HuoYGCH20}, \fedavg{} was analyzed for non-convex problems, where \fedavg was formulated as a stochastic gradient-based algorithm with biased gradients, and the convergence of \fedavg with decaying step sizes to stationary points was proved. 
Moreover, \citet{HuoYGCH20} proposed \fedmom, a server-side acceleration based on Nesterov's momentum, and proved again its convergence to stationary points. Lately, \citet{ReddiCZG20} proposed and analyzed federated versions of several popular adaptive optimizers (e.g. \texttt{ADAM}). They generalize the framework of \fedavg{} by decoupling the \FL update scheme into \emph{server optimizer} and \emph{client optimizer}. Interestingly, same as us, \citet{ReddiCZG20} also observed the importance of learning rate decays on both clients and server.% by examining different learning rate schemes.

Recently, an interesting work by \citet{PathakWainwright20} demonstrated theoretically that fixed points reached by \fedavg and \fedprox{} (if exist) need not be stationary points of the original optimization problem, even in convex settings and with deterministic updates. To address this issue, they proposed \texttt{FedSplit} to restore the correct fixed points.
% and enjoys convergence guarantees to them under  convexity assumptions. 
It still remains open, though, if \texttt{FedSplit} can still converge to the correct fixed points under asynchronous and stochastic user updates, both of which are widely adopted in practice and studied here.

Ensuring fairness among users has become a serious goal in \FL since it largely determines users’ willingness to participate in the training process. \citet{MohriSS19} argued that existing \FL algorithms can lead to federated models that are biased toward different users. To solve this issue, \citet{MohriSS19} proposed agnostic federated learning (\afl{}) to improve fairness among users. \afl considers the target distribution as a weighted combination of the user distributions %\ie $\mathcal{D}_{\lambda} = \sum_{i=1}^m \lambda_i \mathcal{D}_i$ for some $\lambda \in \Delta$, 
and optimizes the centralized model for the worse-case realization, %of $\lambda$, 
leading to a saddle-point optimization problem which was solved by a fast stochastic optimization algorithm. On the other hand, based on fair resource allocation in wireless networks, \citet{LiSBS20} proposed q-fair federated learning (\texttt{q-FFL}) to achieve more uniform test accuracy across users. \citet{LiSBS20} further proposed \qffl{} as a communication efficient algorithm to solve \texttt{q-FFL}. 
However, both \afl and \qffl do not explicitly encourage user participation and they suffer from adversarial attacks while our algorithm \FMp is designed to be fair among participants and robust against both additive and multiplicative attacks. 

\fedavg{} 
% \citep{McMahanMRHA17}
relies on a coordinate-wise averaging of   local models to update the global model. According to \citet{WangYSPK20}, in neural network (NN) based models, such coordinate-wise averaging might lead to sub-optimal results due to the permutation invariance of NN parameters. To address this issue, \citet{Yurochkin19a} proposed probabilistic
federated neural matching (\pfnm{}), which is only applicable to fully connected feed-forward networks. The recent work \citep{WangYSPK20} proposed federated matched averaging (\fedma{})  as a layer-wise extension of \pfnm{} to accommodate CNNs and LSTMs. However, the Bayesian non-parametric mechanism in \pfnm{} and \fedma{} may be vulnerable to model poisoning attack \citep{BagdasaryanVHES20,BhagojiCMC19,WangSRVASLP20}, while some simple defences, such as norm thresholding and differential privacy, were discussed in \citet{SunKSM19}. 
 We note that these ideas are complementary to \FMp and we plan to investigate possible integrations of them in future work.

Lastly, we note that there is significant interest in standardizing the benchmarks,  protocols and evaluations in \FL, see for instance \citep{Caldas18,He20}. We have spent significant efforts in adhering to the suggested rules there, by reporting on common datasets, open sourcing our code and including all experimental details.

\section{Problem Setup}
\label{sec:bg}
We recall the federated learning (\FL) framework of \citet{McMahanMRHA17} and point out a simple interpretation that seemingly unifies different implementations. We consider \FL with $\numuser$ users (edge devices), where the $\userind$-th user is interested in minimizing a function $f_{\userind}: \RR^{\numdim} \to \RR, \userind = 1, \ldots, \numuser$, defined on a shared model parameter  $\wv\in \RR^{\numdim}$. Typically, each user function $f_{\userind}$ also depends on the respective  user's local (private) data $\Dc_{\userind}$. 
The main goal in \FL is to \emph{collectively and efficiently} optimize \emph{individual} objectives $\{f_{\userind}\}$ while meeting challenges such as those mentioned in the Introduction (\S\ref{sec:intro}): non-iid distribution of user data, limited communication, user privacy, fairness, robustness, \etc.

\citet{McMahanMRHA17} proposed \FA to optimize the arithmetic average of individual user functions:
\begin{align}
\label{eq:wFL}
\min_{\wv\in \RR^{\numdim}} ~ \pa{\fv,\lambdav}^{0}(\wv), ~~\where~~ \pa{\fv,\lambdav}^{0}(\wv) := \sum_{{\userind}=1}^{\numuser} \lambda_{\userind} f_{\userind}(\wv).
\end{align}
The weights $\lambda_{\userind}$ need to be specified \emph{beforehand}. Typical choices include the dataset size at each user, the ``importance'' of each user, or simply uniform, \ie $\lambda_{\userind} \equiv 1/\numuser$. \FA works as follows: At each round, a (random) subset of users is selected, each of which performs $k$ epochs of local (full or minibatch) gradient descent:
\begin{align}
\mbox{for all $\userind$ in parallel}, ~~ \wv^{{\userind}} \gets \wv^{{\userind}} - \eta \nabla f_{\userind}(\wv^{{\userind}}),
\end{align}
and then the weights are averaged at the server side:
\begin{align}
\wv \gets \sum_{{\userind}} \lambda_{\userind} \wv^{{\userind}},
\end{align}
which is finally broadcast to the users in the next round. The number of local epochs $k$ turns out to be a key factor. Setting $k=1$ amounts to solving \eqref{eq:wFL} by the usual gradient descent  algorithm,
while setting $k=\infty$ (and assuming convergence for each local function $f_{\userind}$) amounts to (repeatedly) averaging the respective minimizers of $f_{\userind}$'s. We now give a new interpretation of \FA that yields insights on what it \emph{optimizes} with an intermediate $k$.

Our interpretation is based on the proximal average \citep{BauschkeGLW08}. Recall that the Moreau envelope and proximal map of a convex\footnote{For nonconvex functions, similar results hold once we address multi-valuedness of the proximal map, see \cite{YuZMX15}.} function $f$ is defined respectively as:
\begin{align}
\env{f}^{\eta}(\wv) &= \min_{\xv} ~\tfrac{1}{2\eta}\|\xv - \wv\|_2^2 + f(\xv), \\ \prox{f}^{\eta}(\wv) &= \argmin_{\xv} ~\tfrac{1}{2\eta}\|\xv - \wv\|_2^2 + f(\xv).
\end{align}
Given a set of convex functions $\fv = (f_1, \ldots, f_{\numuser})$ and positive weights $\lambdav = (\lambda_1, \ldots, \lambda_{\numuser})$ that sum to 1, we define the proximal average as the \emph{unique} function $\pa{\fv, \lambdav}^{\eta}$ such that 
%\begin{align}
$
\prox{\pa{\fv,\lambdav}^{\eta}}^{\eta} = \sum_{\userind} \lambda_{\userind} \prox{f_{\userind}}^{\eta}.
$
%\end{align}
In other words, the proximal map of the proximal average is the average of proximal maps. 
More concretely, \citet{BauschkeGLW08} gave the following explicit, albeit complicated, formula for the proximal average:
\begin{align}
\label{eq:pa}
\!\!\pa{\fv,\lambdav}^{\eta}(\wv) &=\! \min_{\wv_1, \ldots, \wv_m} \sum_{i=1}^m \lambda_i \big[f_i(\wv_i) \!+\! \tfrac{1}{2\eta} \|\wv_i\|_2^2 \big ] \!-\! \tfrac{1}{2\eta}\|\wv\|_2^2, \\
&\qquad ~~\st~~ \sum_{i=1}^m \lambda_i \wv_i = \wv.
\end{align}
From the above formula we can easily derive that 
\begin{align}
\pa{\fv,\lambdav}^{0}(\wv) &:= \lim_{\eta \to 0+} \pa{\fv, \lambdav}^{\eta}(\wv) = \sum\nolimits_{\userind} \lambda_{\userind} f_{\userind}(\wv), \nonumber \\ \pa{\fv,\lambdav}^{\infty}(\wv) &:= \lim_{\eta \to \infty} \pa{\fv, \lambdav}^{\eta}(\wv) = \min_{\sum_{\userind} \lambda_{\userind} \wv_{\userind} = \wv} ~ \sum\nolimits_{\userind} \lambda_{\userind} f_{\userind}(\wv_{\userind}).\nonumber
\end{align}
Interestingly, we can now interpret \FA in two extreme settings as minimizing the proximal average: 
\begin{itemize}
\item \FA with $k=1$ local step is exactly the same as minimizing the proximal average $\pa{\fv,\lambdav}^{0}(\wv)$ with $\eta = 0$. This is clear from the objective \eqref{eq:wFL} of \FA (as our notation already suggests). 
\item \FA with $k=\infty$ local steps is exactly the same as minimizing the proximal average $\pa{\fv,\lambdav}^{\infty}(\wv)$ with $\eta = \infty$. Indeed, 
\begin{align}
\left\{ \min_{\wv} ~\pa{\fv,\lambdav}^{\infty}(\wv) \right\}= \min_{\wv_1, \ldots, \wv_m} \sum_i \lambda_i f_i(\wv_i), 
\end{align}
where the right-hand side decouples and hence $\wv_i$ at optimality is a minimizer of $f_i$ (recall that $\lambdav \geq 0$).
\end{itemize}

Therefore, we may interpret \FA with an intermediate $k$ as minimizing $\pa{\fv, \lambdav}^{\eta}$ with an intermediate $\eta$. More interestingly, if we apply the PA-PG algorithm in \citet[Algo.~$2$]{Yu13a} to minimize $\pa{\fv, \lambdav}^{\eta}$, we obtain the simple update rule
\begin{align}
\label{eq:papg}
\wv \gets \sum\nolimits_{\userind} \lambda_{\userind} \prox{f_{\userind}}^{\eta}(\wv),
\end{align}
where the proximal maps are computed in parallel at the user's side. We note that the recent \FP algorithm \citep{LiSZSTS20} is essentially a randomized version of \eqref{eq:papg}. Crucially, we do not need to evaluate the complicated formula \eqref{eq:pa} as the update \eqref{eq:papg} only requires its proximal map, which by definition is the average of the individual proximal maps (computed by each user separately). Moreover, the difference between the proximal average $\pa{\fv,\lambdav}^{\eta}$ and the arithmetic average $\pa{\fv,\lambdav}^{0}$ can be \emph{uniformly} bounded using the Lipschitz constant of each function $f_{\userind}$ \citep{Yu13a}. Thus, for small step size $\eta$, \FA (with any finite $k$) and \FP all minimize some \emph{approximate} form of the arithmetic average in \eqref{eq:wFL}.

How to set the weights $\lambdav$ in \FA has been a major challenge. In \FL, data is distributed in a highly non-iid and unbalanced fashion, so it is not clear if some chosen arithmetic average in \eqref{eq:wFL} would really satisfy one's actual intention. A second issue with the arithmetic average in \eqref{eq:wFL} is its well-known non-robustness against malicious manipulations, which has been exploited in recent adversarial attacks \citep{BhagojiCMC19}. Instead, Agnostic \FL (\afl \citep{MohriSS19}) aims to optimize the worst-case loss:
\begin{align}
\label{eq:AFL}
\min_{\wv} ~ \max_{\lambdav\in\Lambda} ~\pa{\fv,\lambdav}^{0}(\wv),
\end{align}
where the set $\Lambda$ might cover reality better than any specific $\lambdav$ and provide some minimum guarantee for all users (hence achieving mild fairness). On the other hand, the worst-case loss in \eqref{eq:AFL} is perhaps even more non-robust against adversarial attacks. For instance, adding a positive constant to some loss $f_{\userind}$ can make it dominate the entire optimization process.
The recent work \qffl \citep{LiSBS20} proposes an $\ell_q$ norm interpolation between \FA (essentially $\ell_1$ norm) and \afl (essentially $\ell_\infty$ norm). By tuning $q$, \qffl can achieve better compromise than \FA or \afl.

\section{Multi-objective Minimization (\MOM)}
\label{sec:MOM}
Multi-objective minimization (\MOM) refers to the setting where \emph{multiple} scalar objective functions, possibly incompatible with each other, need to be minimized \emph{simultaneously}. It is also called \textit{vector optimization} \citep{Jahn09} because the objective functions can be combined into a single vector-valued function. In mathematical terms, \MOM can be written as
\begin{align}
\label{eq:MOM}
\min_{\wv\in \RR^{\numdim}} ~\fv(\wv) := \left(f_{1}(\wv), f_{2}(\wv), \ldots, f_{\numuser}(\wv)\right), 
\end{align}
where the minimum is defined wrt the \emph{partial} ordering:
\begin{align}
\label{eq:po}
\fv(\wv) \leq \fv(\zv) \iff \forall \userind = 1, \ldots, \numuser, ~f_{\userind}(\wv) \leq f_{\userind}(\zv). 
\end{align}
(We remind that algebraic operations such as $\leq$ and $+$, when applied to a vector with another vector \emph{or scalar}, are always performed component-wise.) 
Unlike single objective optimization, with multiple objectives it is possible that 
\begin{align}
\fv(\wv) \not\leq \fv(\zv) \mbox{ and } \fv(\zv) \not\leq \fv(\wv),
\end{align}
in which case we say $\wv$ and $\zv$ are not comparable. 

We call $\wv^*$ a Pareto optimal solution of \eqref{eq:MOM} if its objective value $\fv(\wv^*)$ is a minimum element (wrt the partial ordering in \eqref{eq:po}), or equivalently for any $\wv$, $\fv(\wv) \leq \fv(\wv^*)$ implies $\fv(\wv) = \fv(\wv^*)$. In other words, it is not possible to improve \emph{any} component objective in $\fv(\wv^*)$ without compromising some other objective. Similarly, we call $\wv^*$ a \emph{weakly} Pareto optimal solution if there does not exist any $\wv$ such that $\fv(\wv) < \fv(\wv^*)$, \ie, it is not possible to improve \emph{all} component objectives in $\fv(\wv^*)$. Clearly, any Pareto optimal solution is also weakly Pareto optimal but the converse may not hold.

We point out that the optimal solutions in \MOM are usually a set (in general of infinite cardinality) \citep{Mukai80}, and without additional subjective preference information, all Pareto optimal solutions are considered equally good (as they are not comparable against each other). This is \emph{fundamentally} different from the single objective case.

From now on, for simplicity we assume all objective functions are continuously differentiable but not necessarily convex (to accommodate deep models). Finding a (weakly) Pareto optimal solution in this setting is quite challenging (already so in the single objective case). Instead, we will contend with Pareto stationary solutions, namely those that satisfy an intuitive first order necessary condition:
\begin{definition}[Pareto-stationarity, \citealt{Mukai80}]
We call $\wv^{*}$ Pareto-stationary iff \emph{some} convex combination of the gradients $\{\nabla f_{\userind}(\wv^{*})\}$ vanishes, \ie there exists some $\lambdav\geq 0$ such that $\sum_i \lambda_i = 1$ and $\sum_i \lambda_i\nabla f_i(\wv^*) = \zero$.
\label{thm:PS}
\end{definition}
\begin{lemma}[\citealt{Mukai80}]
\label{thm:ParetoStat}
Any Pareto optimal solution is Pareto stationary. Conversely, if all functions are convex, then any Pareto stationary solution is weakly Pareto optimal.
\end{lemma}
Needless to say, the above results reduce to the familiar ones for the single objective case ($\numuser=1$).

There exist many algorithms for finding Pareto stationary solutions. We briefly review three popular ones that are relevant for us, and refer the reader to the excellent monograph \citep{Mietinen98} for more details.

\textbf{Weighted approach.}
Let $\lambdav\in\Delta$ (the simplex) and consider the following single, weighted objective:
\begin{align}
\label{eq:WA}
\min_{\wv} ~~ \sum_{i=1}^m \lambda_i f_i(\wv).
\end{align}
This is essentially the approach taken by \FA, with any (global) minimizer of \eqref{eq:WA} being weakly Pareto optimal (in fact, Pareto optimal if all weights $\lambda_i$ are positive). 
From \Cref{thm:PS} it is clear that any stationary solution of the weighted scalar problem \eqref{eq:WA} is a Pareto stationary solution of the original \MOM \eqref{eq:MOM}. 
Note that the scalarization weights $\lambdav$, once chosen, are fixed throughout. Different $\lambdav$ leads to different Pareto stationary solutions. 

\textbf{$\epsilon$-constraint.} Let $\epsilonv \in \RR^{m-1}$, $\iota \in \{1, \ldots, m\}$ and consider the following constrained scalar problem:
\begin{align}
\label{eq:eps}
\min_{\wv}~ & f_{\iota}(\wv) \\
\mathrm{s.t.}~ & f_i(\wv) \leq \epsilon_i, ~\forall i \ne \iota.
\end{align}
Assuming the constraints are satisfiable, then any (global) minimizer of \eqref{eq:eps} is again weakly Pareto optimal. The $\epsilon$-constraint approach is closely related to the weighted approach above, through the usual Lagrangian reformulation. Both require fixing an $m-1$ dimensional parameter in advance ($\lambdav$ \vs~$\epsilonv$), though.

\textbf{Chebyshev approach.} Let $\sv\in\RR^m$ and consider the minimax problem (where recall that $\Delta$ is the simplex constraint):
\begin{align}
    \min_{\wv} \max_{\lambdav\in\Delta} ~~ \lambdav^\top (\fv(\wv) - \sv).
\end{align}
Again, any (global) minimizer is weakly Pareto optimal. Here $\sv$ is a fixed vector that ideally lower bounds $\fv$. This is essentially the approach taken by \afl \citep{MohriSS19} with $\sv = \zero$.

%\vspace{-.3em}
\section{\FL as Multi-objective Minimization}
\label{sec:tech}
%\vspace{-.3em}
Having introduced both \FL and \MOM, and observed some connections between the two, it is very natural to treat each user function $f_{i}$ in \FL as a separate objective in \MOM and aim to optimize them \emph{simultaneously} as in \eqref{eq:MOM}. This will be the main approach we follow below, which, to the best of our knowledge, has not been formally explored before (despite of the apparent connections that we saw in the previous section, perhaps retrospectively). 
In particular, we will extend the multiple gradient descent algorithm \citep{Mukai80} in \MOM to \FL, draw connections to existing \FL algorithms, and prove convergence properties of our extended algorithm \FMp. 
Very importantly, the notion of Pareto optimality and stationarity immediately enforces fairness among users, as we are discouraged from improving certain users by sacrificing others. 

% \begin{align}
% \label{eq:MoMmain}
% \min_{\wv\in \RR^{\numdim}} ~\fv(\wv) := \left(f_{1}(\wv), f_{2}(\wv), \ldots, f_{\numuser}(\wv)\right), 
% \end{align}
% where the minimum is defined wrt the \emph{partial} ordering:
% \begin{align}
% \label{eq:pomain}
% \fv(\wv) \leq \fv(\zv) \iff \forall \userind = 1, \ldots, \numuser, ~f_{\userind}(\wv) \leq f_{\userind}(\zv). 
% \end{align}
% Unlike single objective optimization, with multiple objectives it is possible that 
% \begin{align}
% \fv(\wv) \not\leq \fv(\zv) \mbox{ and } \fv(\zv) \not\leq \fv(\wv),
% \end{align}
% in which case we say they are not comparable. 

%We call $\wv^*$ a Pareto optimal solution of \eqref{eq:MoMmain} if for all $\wv$, $\fv(\wv) \leq \fv(\wv^*)$ implies $\fv(\wv) = \fv(\wv^*)$. In other words, it is not possible to improve \emph{any} component objective in $\fv(\wv^*)$ without compromising some other objective. For nonconvex functions, we extend the definition to include Pareto stationary points, at which the convex hull of the gradients contains the origin. Importantly for us, the notion of Pareto optimality and stationarity immediately enforces fairness among users, as we are not allowed to improve some user by sacrificing others. 
%We refer the reader to \Cref{sec:MOM} for more basic concepts in \MOM.

To further motivate our development, let us compare to the objective in \afl \citep{MohriSS19}:
\begin{align}
\label{eq:afl}
\min_{\wv} ~ \max_{\lambdav \in \Delta} ~ \lambdav^\top \fv(\wv) \quad \equiv \quad \min_{\wv} ~ \max_{\userind=1, \ldots, \numuser} ~ f_{\userind}(\wv), 
\end{align}
where $\Delta$ denotes the simplex\footnote{To be precise, \afl restricted $\lambdav$ to a subset $\Lambda \subseteq \Delta$. We simply set $\Lambda = \Delta$ to ease the discussion.}. By optimizing the \emph{worst} loss than the \emph{average} loss in \FA, \afl provides some guarantee to all users hence achieving some form of fairness. However, note that \afl's objective \eqref{eq:afl} is not robust against adversarial attacks. In fact, if a malicious user artificially ``inflates'' its loss $f_{\userind}$ (\eg, even by adding/multiplying a constant), it can completely dominate and mislead \afl to solely focus on optimizing its performance. The same issue applies to \qffl \citep{LiSBS20}, albeit with a less dramatic effect if $q$ is small.

\afl's objective \eqref{eq:afl} is very similar to the Chebyshev approach in \MOM (see \Cref{sec:MOM}), which inspires us to propose the following iterative algorithm for solving \eqref{eq:MOM}:
\begin{align}
\label{eq:Cheby}
\tilde\wv_{t+1} = \argmin_{\wv} ~ \max_{\lambdav \in \Delta} ~ \lambdav^\top (\fv(\wv) - \fv(\tilde\wv_t)),
\end{align}
where we \emph{adaptively} ``center'' the user functions using function values from the previous iteration. When the functions $f_i$ are smooth, we apply the quadratic bound to obtain:
\begin{align}
\label{eq:FMp}
\wv_{t+1} = \argmin_{\wv} \max_{\lambdav \in \Delta} \lambdav^\top J^\top_{\fv}(\wv_t)(\wv - \wv_t) 
+ \tfrac{1}{2\upeta}\|\wv- \wv_t\|^2\!,
\end{align}
where $J_{\fv} = [\nabla f_1, \ldots, \nabla f_{\numuser}] \in \RR^{\numdim \times \numuser}$ is the Jacobian and $\upeta>0$ is the step size. Crucially, note that $\fv(\wv_t)$ does not appear in the above bound \eqref{eq:FMp} since we subtracted it off in \eqref{eq:Cheby}. Since \eqref{eq:FMp} is convex in $\wv$ and concave in $\lambdav$  we can swap min with max and obtain the dual:
\begin{align}
\max_{\lambdav\in\Delta} ~ & \min_{\wv} ~ \lambdav^\top J^\top_{\fv}(\wv_t)(\wv - \wv_t) + \tfrac{1}{2\upeta}\|\wv-\wv_t\|^2.
\end{align}
Solving $\wv$ by setting its derivative to $\zero$ we arrive at:
\begin{align}
\label{eq:MGDA}
\wv_{t+1} = \wv_t - \upeta \dv_t, ~~ \dv_t = J_{\fv}(\wv_t) \lambdav_t^*, \\
\where ~~
\lambdav_t^* = \argmin_{\lambdav\in\Delta} ~ \|J_{\fv}(\wv_t)\lambdav\|^2.
\end{align}
Note that $\dv_t$ is precisely the minimum-norm element in the convex hull of the columns (\ie, gradients) in the Jacobian $J_{\fv}$, and finding $\lambdav_t^*$ amounts to solving a simple quadratic program. 
The resulting iterative algorithm in \eqref{eq:MGDA} is known as multiple gradient descent algorithm (MGDA), which has been (re)discovered in \citet{Mukai80,FliegeSvaiter00,Desideri12} and recently applied to multitask learning  in \citet{SenerKoltun18,LinZLZK19} and to training GANs in \citet{AlbuquerqueMDCFM19}. Our concise derivation here reveals some new insights about MGDA, in particular its connection to \afl.

To adapt MGDA to the federated learning setting, we propose the following extensions. 

\textbf{Balancing user average performance and fairness.} We observe that the MGDA update in \eqref{eq:MGDA} resembles \FA, with the crucial difference that MGDA \emph{automatically} tunes the dual weighting variable $\lambdav$ in each step while \FA pre-sets $\lambdav$ based on \emph{a priori} information about the user functions (or simply uniform in lack of such information). Importantly, the direction $\dv_t$ found in MGDA is a common descent direction for \emph{all participating objectives}:
\begin{align}
\vspace{-0.2em}
\label{eq:fair}
\fv(\wv_{t+1}) &\leq \fv(\wv_t) +  J_{\fv}^\top(\wv_t) (\wv_{t+1} - \wv_{t})  
+\tfrac{1}{2\upeta}\|\wv_{t+1} - \wv_{t}\|^2 \nonumber\\
&\leq \fv(\wv_t),
\end{align}
where the first inequality follows from familiar smoothness assumption on $\fv$ while the second inequality follows simply from plugging $\wv= \wv_t$ in \eqref{eq:FMp} and noting that $\wv_{t+1}$ by definition can only decrease \eqref{eq:FMp} even more. It is clear that equality is attained iff $\dv_t = J_{\fv}(\wv_t) \lambdav_t^* = \zero$, \ie, $\wv_t$ is Pareto-stationary (see \Cref{sec:MOM}). In other words, MGDA never sacrifices any participating objective to trade for more sizable improvements over some other objective, something \FA with a fixed weighting $\lambdav$ might attempt to do. On the other hand, \FA with a fixed weighting $\lambdav$ may achieve higher \emph{average performance} under the weighting $\lambdav$. It is natural to introduce the following trade-off between average performance and fairness:
\begin{align}
\label{eq:interpolation}
%\wv_{t+1} = \wv_t - \upeta \dv_t, ~~ \dv_t = J_{\fv}(\wv_t) \lambdav_t^*, \\
%\where
\mbox{ update \eqref{eq:MGDA} with } 
\lambdav_t^* = \!\!\argmin_{\lambdav\in\Delta, \|\lambdav - \lambdav_0\|_\infty \leq \epsilon} \!\|J_{\fv}(\wv_t)\lambdav\|^2\!\!.
\end{align}
Clearly, setting $\epsilon = 0$ recovers \FA with \emph{a priori} weighting $\lambdav_0$ while setting $\epsilon=1$ recovers MGDA where the weighting variable $\lambdav$ is tuned without any restriction to achieve maximal fairness. In practice, with an intermediate $\epsilon \in (0,1)$ we may strike a desirable balance between the two (sometimes) conflicting goals. Moreover, even with the uninformative weighting $\lambdav_0 = \mathbf{1}/m$, using an intermediate $\epsilon$ allows us to upper bound the contribution of each user function to the common direction $\dv_t$ hence achieve some form of robustness against malicious manipulations.

\textbf{Robustness against malicious users through normalization.} Existing work \citep[\eg,][]{BhagojiCMC19,XieKG19} has demonstrated that the \emph{average} gradient in \FA can be easily manipulated by even a single malicious user. While more robust aggregation strategies are studied recently \citep[see \eg,][]{BlanchardEMGS17,YinCKB18,DiakonikolasKKLSS2019}, they do not necessarily maintain the convergence properties of \FMp (\eg finding a common descent direction and converging to a Pareto stationary solution). 
Instead, we propose to simply normalize the gradients from each user to unit length, based on the following considerations: 
(a) Normalizing the (sub)gradient is common for specialists in nonsmooth and stochastic optimization \citep{AnstreicherWolsey09} and sometimes eases step size tuning. 
(b) Solving the weights $\lambdav_t^*$ in \eqref{eq:MGDA} with normalized gradients still guarantees fairness, \ie, the resulting direction $\dv_t$ is descending for all participating objectives (by a completely  similar reasoning as the remark after \eqref{eq:fair}).
(c) Normalization restores robustness against multiplicative ``inflation'' from any malicious user, which, combined with MGDA's built-in robustness against additive ``inflation'' (see \cref{eq:Cheby}), offers reasonable robustness guarantees against adversarial attacks.
%in practice.

\textbf{Balancing communication and on-device computation.} Communication between user devices and the central server is heavily constrained in \FL, due to a variety of reasons mentioned in \S\ref{sec:bg}. On the other hand, modern edge devices are capable of performing reasonable amount of on-device computations. Thus, we allow each user device to perform multiple local updates before communicating its update $\gv = \wv^{0} - \wv^{\scalebox{0.5}{$\smallblacksquare$}}$, namely the difference between the initial $\wv^{0}$ and the final $\wv^{\scalebox{0.5}{$\smallblacksquare$}}$, to the central server. The server then calls the (extended) MGDA to perform a global update, which will be broadcast to the next round of user devices. We note that similar strategy was already adopted in many existing \FL systems \citep[\eg,][]{McMahanMRHA17,LiSBS20,LiSZSTS20}.

\textbf{Subsampling to alleviate non-iid and enhance throughput.} Due to the massive number of edge devices in \FL, it is not realistic to expect most devices to participate at each or even most rounds. Consequently, the current practice in \FL is to select a (different) subset of user devices to participate in each round \citep{McMahanMRHA17}. Moreover, randomly subsampling user devices can also help combat the non-iid distribution of user-specific data  \citep[\eg,][]{McMahanMRHA17,LiHYWZ20}. 
Here we point out an important advantage of our MGDA-based algorithm: its update is along a common descending direction (see \eqref{eq:fair}), meaning  that the objective of any \emph{participating} user can only decrease. We believe this unique property of MGDA provides strong incentive for users to participate in \FL. To our best knowledge, existing \FL algorithms do not provide similar algorithmic incentives. 
Last but not the least, subsampling also solves a degeneracy issue in MGDA: when the number of participating users exceeds the dimension $\numdim$, the Jacobian $J_{\fv}$ has full row-rank hence \eqref{eq:MGDA} achieves Pareto-stationarity in a single iteration and stops making progress. Subsampling removes this undesirable effect and allows different subsets of users to be continuously optimized.

With the above extensions, we summarize our extended algorithm \FMp in \Cref{alg:FM}, and we prove the following convergence guarantees (precise statements and proofs can be found in \Cref{sec:proof}):

\begin{manualtheorem}{1a}
\label{thm:simp}
Let each user function $f_i$ be $L$-Lipschitz smooth and $M$-Lipschitz continuous, and choose step size $\upeta_t$ so that $\sum_t \upeta_t = \infty$ and $\sum_t \sigma_t \upeta_t < \infty$, where $\sigma_t^2 := \EE \| \dv_t - \hat \dv_t \|^2$ with 
\begin{align}
\dv_t &:= J_{\fv}(\wv_t) \lambdav_t, ~ \lambdav_t = \argmin_{\lambdav\in\Delta} \|J_{\fv}(\wv_t) \lambdav\|, \\
\hat \dv_t &:= \hat J_{\fv}(\wv_t) \hat \lambdav_t, ~ \hat\lambdav_t = \argmin_{\lambdav\in\Delta} \|\hat J_{\fv}(\wv_t) \lambdav\|.
\end{align} 
Then, with $k = r = 1$ we have:
\begin{align}
\min_{t=0, \ldots, T} \EE\|J_{\fv}(\wv_t) \lambdav_t\|^2 \to 0. 
\end{align}
\end{manualtheorem}
%\vskip-.5em
Here $k$ is the number of local updates and $r$ is the number of minibatches in each local update. The convergence rate depends on how quickly the ``variance'' term $\sigma_t$ of the stochastic common descent direction $\hat d_t$ diminishes (if at all), which in turn depends on how aggressively we subsample users or how heterogeneous the users are. 
% We remark that a similar result for any $k$ or $r$ still holds: we need only replace the (stochastic) gradient $\hat\nabla f_i(\wv_t)$ in the Jacobian $\hat J_{\fv}(\wv_t)$ with $\wv_t - \wv_{t,i}^{\scalebox{0.5}{$\smallblacksquare$}}$ where $\wv_{t,i}^{\scalebox{0.5}{$\smallblacksquare$}}$ is the local output of the $i$-th user at the $t$-th round. The caveat is that the variance $\sigma_t^2$ becomes harder to control if $k$ and $r$ are large, as $\wv_t - \wv_{t,i}^{\scalebox{0.5}{$\smallblacksquare$}}$ may no longer approximate $\nabla f_i(\wv_t)$ well. With a smaller step size, using a larger $k$ and $r$ may still be beneficial, as we report in our experiments. We believe further analysis is required to formally understand the tradeoff in $k$ and $r$.

For deterministic gradient updates, we can prove convergence even with more local updates (\ie $k>1$):
\begin{manualtheorem}{1b}
Let each user function $f_i$ be $L$-Lipschitz smooth and $M$-Lipschitz continuous. For any number of local updates $k$, if the global step size $\upeta_t \to 0$ with $\sum_t \upeta_t = \infty$, local learning rate $\eta^{l}_t \to 0$ and $\varepsilon_t := \|\lambdav_t - \hat\lambdav_t\| \to 0$, then we have:
\begin{align}
\min_{t=0, \ldots, T} \|J_{\fv}(\wv_t) \lambdav_t\|^2 \to 0. 
\end{align}
\end{manualtheorem}
Please refer to \Cref{sec:proof} for the precise statement of the theorem and its  proof. 
We note that one natural approach to bound the deviation $\varepsilon_t$ is by applying the $\epsilon$-constrained version of \fedMGDA. 
For example, if  $\left\|\lambdav-\lambdav_{0}\right\|_{\infty} \leq \epsilon_t$, and $\epsilon_t$ is bounded, then $\varepsilon_t \leq 2 \sqrt{m} \epsilon_t$ is also bounded. Thus, $\varepsilon_t \to 0$ when $\epsilon_t \to 0$.
Moreover, 
when $k=1$, we do not need the local learning rate $\eta^{l}_t$ to decay for convergence; in addition, if $\varepsilon_t \equiv 0$ (e.g. in \fedavg), then our convergence guarantee reduces to the usual one for gradient descent, which is  expected since we know \fedavg with $k=1,r=1$ is the same as centralized gradient descent.
Lastly, we note that when $k>1$, local learning rate $\eta^{l}_t$ must vanish in order to obtain convergence. This importance of local learning rate decay is also pointed out in \citet{ReddiCZG20}.

When the functions $f_i$ are convex, we can derive a finer result:
\begin{manualtheorem}{2}
Suppose each user function $f_i$ is convex and $M$-Lipschitz  continuous. Suppose at each round \FMp includes a strongly convex user function whose weight is bounded away from 0. Then, with the choice $\upeta_t = \tfrac{2}{c(t+2)}$ and $k=r=1$, we have 
\begin{align}
\EE\|\wv_{t} - \wv_{t}^*\|^2 \leq \tfrac{4M^2}{c^2(t+3)},
\end{align}
and $\wv_t-\wv_t^* \to 0$ almost surely, 
where $\wv_t^*$ is the nearest Pareto stationary solution to $\wv_t$ and $c$ is some constant.
\end{manualtheorem}
A slightly stronger result where we also allow some user functions to be nonconvex can be found in \Cref{sec:proof}. The same results hold if the gradient normalization is bounded away from 0 (otherwise we are already close to Pareto stationarity). For $r, k > 1$, using a similar argument as in \S\ref{sec:bg}, we expect \FMp to optimize some \emph{proxy problem} (such as the proximal average), and we leave the thorough theoretical analysis for future work.

% \begin{figure*}[t]
% \centering
% \begin{minipage}[t]{0.53\textwidth}

% \end{minipage}
% \hfill
% \begin{minipage}[t]{0.43\textwidth}
% \begin{algorithm}[H]
%     \DontPrintSemicolon
% 	\SetKwFunction{client}{\textsc{ClientUpdate}}
%     \SetKwBlock{Repeat}{repeat}{}	
%     \SetKwProg{Fn}{Function}{:}{}
% 	\Fn{\client{$\userind, \wv$}}{
%     	$\wv^{0} \gets \wv$ \;
% 	    \Repeat($k$ epochs){
% 	    \tcp{\small{split local data into $r$ batches}}
% 	    \vspace{.35em}
	   
% 	    $\Dc_{\userind} \to \Dc_{\userind,1} \cup \cdots \cup \Dc_{\userind, r}$
% 	    \vspace{.4em}
	   
% 	    \For{$j \in \{1, \ldots, r\}$}
% 	        {\vskip.3em
% 	        $\wv \gets \wv - \eta \nabla f_{\userind}(\wv; \Dc_{\userind, j})$}
% 	    }
% 	    \KwRet $\gv := \wv^{0} - \wv$ to server
% 	}
% \end{algorithm}
% \end{minipage}
% \end{figure*}

We remark that convergence rate for MGDA, even when restricted to the deterministic case, was only derived recently in \citet{FliegeVV19}. The stochastic case (that we consider here) is much more challenging and our theorems provide one of the first convergence guarantees for \FMp. 
%We want to point out that
We wish to emphasize that \FMp is not just an alternative algorithm for \FL practitioners; it can be used as a post-processing step to enhance existing \FL systems or combined with existing \FL algorithms (such as \FP or \qffl). 
This is particularly appealing with \emph{nonconvex} user functions as MGDA is capable of converging to all Pareto stationary points while approaches such as \FA do not necessarily enjoy this property even when we enumerate the weighting $\lambdav_0$ \citep{Mietinen98}. 
Furthermore, it is possible to find multiple or even enumerate all Pareto optimal solutions (\ie the Pareto front). For instance, we may run \FMp multiple times with different random seeds or initializations. As shown by \citet{LinZLZK19}, we could also incorporate additional linear  constraints in \eqref{eq:MGDA} to encode one's preference and encourage more diverse solutions. However, these techniques become less effective in higher dimensions (\ie when the number of users is large) and in communication limited settings. Practically, the server may dynamically adjust the linear constraints in (22) to steer the algorithm to a more desirable Pareto stationary solution.

Lastly, we mention that finding the common descent direction (\ie Line 6 of \Cref{alg:FM}) is a standard quadratic programming (QP) problem that is solved only at the server side. For moderate number of (sampled) users, it suffices to employ a generic QP solver while for large number of users we could also solve $\lambda$ efficiently using for instance the conditional gradient algorithm \citep{SenerKoltun18}, with per-step complexity proportional to the model dimension and the number of participating users. 
For our experiments below, we used a generic QP sovler and we observed that this overhead is negligible, resulting almost the same overall running time for \fedavg and \fedMGDA.

\begin{algorithm}[t]
    \caption{\FMp}\label{alg:FM}
    \DontPrintSemicolon
	\SetKwFunction{client}{\textsc{ClientUpdate}}
	\For{$t=1, 2, \ldots$}{
    	 choose a subset $I_t$ of $\lceil\nump \numuser\rceil$ clients/users\;
        \For{$\userind \in I_t$}{
            $\gv_{\userind} \gets$ \client{$\userind, \wv_t$}\;
            $\bar{\gv}_{\userind} := \gv_{\userind} / \|\gv_{\userind}\|$ \tcp*{\small{normalize}}
            }
            
        $\lambdav^* \gets \argmin_{\lambdav\in\Delta, \|\lambdav - \lambdav_0\|_\infty \leq \epsilon} ~ \|\sum_i \lambda_i \bar{\gv}_i\|^2$    
        
        \vskip.1em    
    	$\dv_t \gets \sum_i \lambda_i^* \bar{\gv}_i$ \tcp*{\small{common direction}}
    	
    	\vskip.3em
    	choose (global) step size $\upeta_t$

    	$\wv_{t+1} \gets \wv_{t} - \upeta_t \dv_t$ \;
	}
	
	\SetKwFunction{client}{\textsc{ClientUpdate}}
    \SetKwBlock{Repeat}{repeat}{}	
    \SetKwProg{Fn}{Function}{:}{}
	\Fn{\client{$\userind, \wv$}}{
    	$\wv^{0} \gets \wv$ \;
	    \Repeat($k$ epochs){
	    \tcp{\small{split local data into $r$ batches}}

	    $\Dc_{\userind} \to \Dc_{\userind,1} \cup \cdots \cup \Dc_{\userind, r}$

	    \For{$j \in \{1, \ldots, r\}$}
	   {
	       $\wv \gets \wv - \eta \nabla f_{\userind}(\wv; \Dc_{\userind, j})$}
	    }
	    \vskip-.3em
	    \KwRet $\gv := \wv^{0} - \wv$ to server
	}
\end{algorithm}

\section{Experiments}
% (add citation of FedML, add vector 1 to the proof)
% `\FL algorithms are deployed on smart devices where we can only locally train medium-sized models.'
%\vspace{-0.3em}
\label{sec:exp}

\subsection{Experimental setups}
\label{sec:setups}
\begin{table*}[th]
% \footnotesize	
\centering
\caption{Dataset summary \label{table:dataset_stat}}
\begin{tabular}{l|ccccc} \toprule
          Dataset         &  Train Clients &  Train samples   & Test clients & Test samples & Batch size                    \\\midrule
          CIFAR-10 \citep{Krizhevsky09}        &   $100$        & $50000$          & $100$        & $10000$      &    $\{10, \infty\}$            \\
          F-MNIST \citep{XiaoRV17}     &  $100$    & $60000$ & $100$ & $10000$ &  $\{10, \infty\}$  \\
        FEMNIST \citep{Caldas18}       &    $3406$      & $709385$         &  $3406$      & $80011$      & $\{20, \infty\}$            \\           
          Shakespeare \citep{LiSZSTS20}   &   $31$   & $92959$ & $31$ & $23255$ &  $\{10\}$  \\
          Adult \citep{Dua2019}       &  $2$    & $32561$ & $2$& $16281$ &  $\{10\}$  \\
 \bottomrule
\end{tabular}
\vspace{5pt}
\end{table*}

\begin{table*}[th]
% \footnotesize	
\centering
\caption{CIFAR-10 model \label{table:cifar-10_model}}
\begin{tabular}{lcccc} \toprule
          Layer &  Output Shape &  $\#$ of Trainable Parameters & Activation & Hyper-parameters  \\\midrule
           Input & $(3, 32, 32)$ & $0$ &  &  \\
           Conv2d & $(64, 28, 28)$ & $4864$ & ReLU & kernel size =$5$; strides=$(1, 1)$ \\
           MaxPool2d & $(64, 14, 14)$ & $0$ &  & pool size=$(2, 2)$ \\
           LocalResponseNorm & $(64, 14, 14)$ & $0$ &  & size=$2$ \\
           Conv2d & $(64, 10, 10)$ & $102464$ & ReLU & kernel size =$5$; strides=$(1, 1)$ \\
           LocalResponseNorm & $(64, 10, 10)$ & $0$ &  & size=$2$ \\
           MaxPool2d & $(64, 5, 5)$ & $0$ &  & pool size=$(2, 2)$ \\
           Flatten & $1600$ & $0$ & & \\
           Dense &  $384$ & $614784$ & ReLU & \\
           Dense &  $192$ & $73920$ & ReLU & \\
           Dense &  $10$ & $1930$ & softmax & \\ \midrule
          Total & & $797962$  & & \\ \bottomrule
\end{tabular}
%\vspace{5pt}
\end{table*}
\begin{table*}[th]
% \footnotesize	
\centering
\caption{Fashion MNIST model \label{table:fmnist_model}}
\begin{tabular}{lcccc} \toprule
          Layer &  Output Shape &  $\#$ of Trainable Parameters & Activation & Hyper-parameters  \\\midrule
           Input & $(1, 28, 28)$ & $0$ &  &  \\
           Conv2d & $(10, 24, 24)$ & $260$ & ReLU & kernel size =$5$; strides=$(1, 1)$ \\
           MaxPool2d & $(10, 12, 12)$ & $0$ &  & pool size=$(2, 2)$ \\
           Conv2d & $(20, 8, 8)$ & $5020$ & ReLU & kernel size =$5$; strides=$(1, 1)$ \\
           MaxPool2d & $(20, 4, 4)$ & $0$ &  & pool size=$(2, 2)$ \\
           Dropout2d & $(20, 4, 4)$ & $0$ &  & $p=0.5$ \\
           Flatten & $320$ & $0$ & & \\
            Dense &  $50$ & $16050$ & ReLU & \\
            Dropout & $50$ & $0$ &  & $p=0.5$ \\
            Dense &  $10$ & $510$ & softmax & \\ \midrule
          Total & & $21840$  & & \\ \bottomrule
\end{tabular}
%\vspace{5pt}
\end{table*}
\begin{table*}[th]
% \footnotesize	
\centering
\caption{Federated EMNIST model \citep{ReddiCZG20} \label{table:emnist_model}}
\begin{tabular}{lcccc} \toprule
          Layer &  Output Shape &  $\#$ of Trainable Parameters & Activation & Hyper-parameters  \\\midrule
           Input & $(1, 28, 28)$ & $0$ &  &  \\
           Conv2d & $(32, 26, 26)$ & $320$&  & kernel size =$3$; strides=$(1, 1)$ \\
           Conv2d & $(64, 24, 24)$ & $18496$ & ReLU & kernel size =$3$; strides=$(1, 1)$ \\
           MaxPool2d & $(64, 12, 12)$ & $0$&  & pool size=$(2, 2)$ \\ 
           Dropout & $(64, 12, 12)$ & $0$& & $p=0.25$ \\
           Flatten & $9216$ & $0$ & & \\
           Dense &  $128$ & $1179776$ & & \\
           Dropout & $128$ & $0$ & & $p=0.5$ \\ 
           Dense & $62$ & $7998$ & softmax & \\ \midrule
          Total & &  $1206590$ & & \\ \bottomrule
\end{tabular}
%\vspace{5pt}
\end{table*}

\begin{table*}[ht]
\footnotesize
\centering
\caption{Hyperparameters used in our experiments.  \label{table:hyper}}
\begin{tabular}{l|l} \toprule
    Name             &  Parameters   \\ \midrule
\afl{}               &   $\gamma_{\lambda} \in \{0.01,0.1,0.2,0.5\}, \gamma_{w} \in \{0.01,0.1\}$\\ \midrule
\qffl                & $q\in \{0.001, 0.01, 0.1, 0.5, 1, 2, 5, 10\}$, $L\in \{0.1, 1, 10\}$ \\\midrule
\fedMGDAn{}          & $\upeta \in \{0.5, 1, 1.5, 2\}$, and $\text{Decay} \in \{0, \frac{1}{40}, \frac{1}{30}, \frac{1}{20}, \frac{1}{10}, \frac{1}{3}, \frac{1}{2} \}$ \\\midrule
\fedavgn{}           & $\upeta \in \{0.5, 1, 1.5, 2\}$, and $ \text{Decay}\in \{0, \frac{1}{40}, \frac{1}{30}, \frac{1}{20}, \frac{1}{10}, \frac{1}{3}, \frac{1}{2}\}$ \\\midrule
\fedprox{}           & $\mu \ \in \{0.001, 0.01, 0.1, 0.5, 1, 10\}$ \\\midrule
\fedMGDAproxn{}      & $\mu=0.1$, $\upeta \in \{0.5, 1, 1.5, 2\}$, and  $ \text{Decay} \in \{0, \frac{1}{40}, \frac{1}{30}, \frac{1}{20}, \frac{1}{10}, \frac{1}{5}, \frac{1}{3}, \frac{1}{2}\}$ \\ \bottomrule
\end{tabular}
\vspace{5pt}
\end{table*}

In this subsection we provide experimental details including dataset descriptions, sampling schemes, model configurations and hyper-parameter settings. A quick summary of the datasets that we use can be found in \Cref{table:dataset_stat}. 
We have two parameters in \FMp to control the total number of local updates in each communication round: $k$, the number of local epochs, and $r = n / b$, the number of updates in each local epoch. Here $n$ is the number of samples at each user (assumed the same for simplicity) while $b$ is the minibatch size for each local update. As observed by, \eg, \citet{McMahanMRHA17} (Table 2), having a larger $k$ is similar as having a smaller $b$ (or equivalently a larger $r$), in terms of total number of local updates. Moreover, $k=1$ with a suitable $b$ usually leads to satisfying performance while very large $k$ can result in plateau or divergence. Thus, in our experiments we fix $k=1$ while vary $b$ to reduce the total number of hyperparameters. This corresponds to a single pass of the local data at each user in every communication round.

%\subsubsection{Experimental setup: CIFAR-10, and Fashion MNIST datasets}
\subsubsection{CIFAR-10 \citep{Krizhevsky09} and Fashion MNIST \citep{XiaoRV17} datasets}
\label{sec:image_exp_setup}
In order to create a non-i.i.d. dataset, we follow a similar sampling procedure as in \citet{McMahanMRHA17}: first we sort all data points according to their classes. Then, they are split into $500$ shards, and each user is randomly assigned $5$ shards of data. By considering $100$ users, this procedure guarantees that no user receives data from more than $5$ classes and the data distribution of each user is different from each other. The local datasets are balanced--all users have the same amount of training samples.  The local data is split into train, validation, and test sets with percentage of $80$\%, $10$\%, and $10$\%, respectively. In this way, each user has $400$ data points for training,  $50$ for test,  and $50$ for  validation. We use a CNN model which resembles the one in \citet{McMahanMRHA17}, with two convolutional layers followed by three fully connected layers. The details are included in \Cref{table:cifar-10_model} for CIFAR-10 and in \Cref{table:fmnist_model} for Fashin MNIST. To update the local models at each user using its local data, we apply stochastic gradient descent (SGD) with local batch size $b=10$, local epoch $k=1$, and local learning rate $\eta=0.01$, or $b=400$, $k=1$, and $\eta=0.1$. To model the fact that not all users may participate in each communication round,  we employ a parameter $p$ to control the fraction of participating users: $p=0.1$ is the default setting which means that only $10$\% of users participate in each communication round.

\subsubsection{Federated EMNIST dataset \citep{Caldas18}}
\label{sec:emnist_exp_setup}
For this experimental setup, we use the same dataset, model, and hyper-parameters as \citet{ReddiCZG20}. We use the federated EMNIST dataset of \citet{Caldas18}. The dataset consists of images of digits, and English characters---both lower and upper cases, with 62  classes in total. The images are partitioned by their authors in a way that naturally makes the dataset heterogeneous and unbalanced. We use the  model described in \Cref{table:emnist_model} and  the following  hyper-parameters: local learning rate $\eta=0.1$ and selecting $10$ clients per communication round as recommended. The only difference between our setup and the one in \citep{ReddiCZG20} is that we use local epoch $k=1$ for all algorithms.
%, which is the optimal choice according to \citet{McMahanMRHA17}.        

\subsubsection{Shakespeare dataset \citep{LiSZSTS20}}
\label{sec:shakespeare_exp_setup}
For experiments on the Shakespeare dataset, we use the same model, data pre-processing and sampling procedure as in \qffl{} paper \citep{LiSBS20}. The dataset is built from \emph{The Complete Works of William Shakespeare}, where each role in the play represents one user. Following \citet{LiSZSTS20}, we subsample $31$ users to train a neural language model for next character prediction. Each character is embedded in an $8$-dimensional space and the sequence length is $80$ characters. The model we use is a two-layer LSTM (with hidden size $256$) followed by one dense layer \citep{McMahanMRHA17,LiSZSTS20}. Joint hyper-parameters that are shared by all algorithms include: total communication rounds $T=200$, local batch size $b=10$, local epoch $k=1$, and local optimizer being SGD, unless otherwise stated.

\subsubsection{Adult dataset \citep{Dua2019}}
\label{sec:adult_exp_setup}
Following the setting in \afl{} \citep{MohriSS19}, we split the Adult dataset into two non-overlapping domains based on the \textit{education} attribute---\texttt{phd} domain and \texttt{non-phd} domain. The resulting \FL setting consists of two users each of which has data from one of the two domains. Further, data is pre-processed as in \citet{LiSBS20} to have $99$ binary features. We use a logistic regression model for all \FL algorithms mentioned in the main paper. Local data is split into train, validation, and test sets with percentage of $80$\%, $10$\%, and $10$\%, respectively. In each round, both users participate and the server aggregates their losses and gradients (or weights). Joint hyper-parameters that are shared by all algorithms include: total communication rounds $T=500$, local batch size $b=10$, local epoch $k=1$, local learning rate $\eta=0.01$, and local optimizer being SGD without momentum, unless otherwise stated. Algorithm-specific hyper-parameters will be mentioned in the appropriate places below. One important note is that the \texttt{phd} domain has only $413$ samples while the \texttt{non-phd} domain has $32,148$ samples, so the split is very unbalanced while training \emph{only} on the \texttt{phd} domain yields inferior performance on all domains due to the insufficient sample size.

\subsubsection{Hyper-parameters}
\label{sec:hyper}

We evaluate the performance of different algorithms with a wide range of hyper-parameters, summarized in \Cref{table:hyper}. 
In particular, following \citet{AnstreicherWolsey09} we tried sublinear  $O(1/t)$ and exponential decay $O(\beta^t)$ learning rates $\upeta$ on the server, and a fixed local learning rate $\eta$ for client updates. 
Eventually we settled on decaying $\upeta_t$ by a factor of $\beta$ every $100$ steps: $\upeta_t =\beta^{[\frac{t}{100}]}$, where $\beta = \texttt{decay}^{100/T}$ and $T$ is the total number of communication rounds (with e.g. \texttt{decay} = $1/10$). 
%\blue{For Kiarash: $\upeta_t = \frac{2}{t+2}$ ($t$ is the current round), do it for the normalization? ($\upeta_t = \frac{2}{c(t+2)}$)}
We note that \citet{ReddiCZG20} also found exponential decay to be most effective in their experiments. 
%We use grid search to choose the proper local learning rate. 
We use grid search to choose suitable local learning rates for all algorithms.% and tune global learning rates with step decays for algorithms with gradient normalization.

\begin{figure*}[t]
\centering
\includegraphics[width=1.0\textwidth, valign=t]{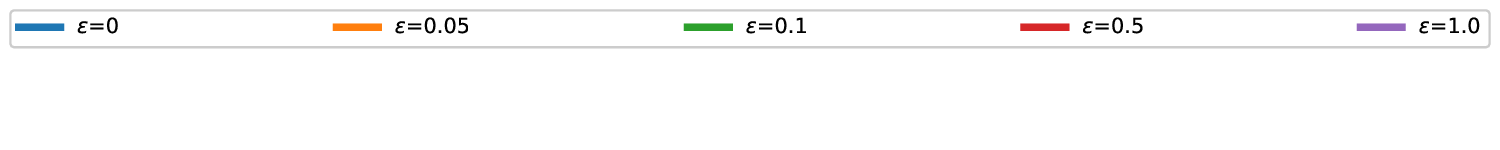}
\vspace{-25pt}

\includegraphics[width=0.24\textwidth, valign=t]{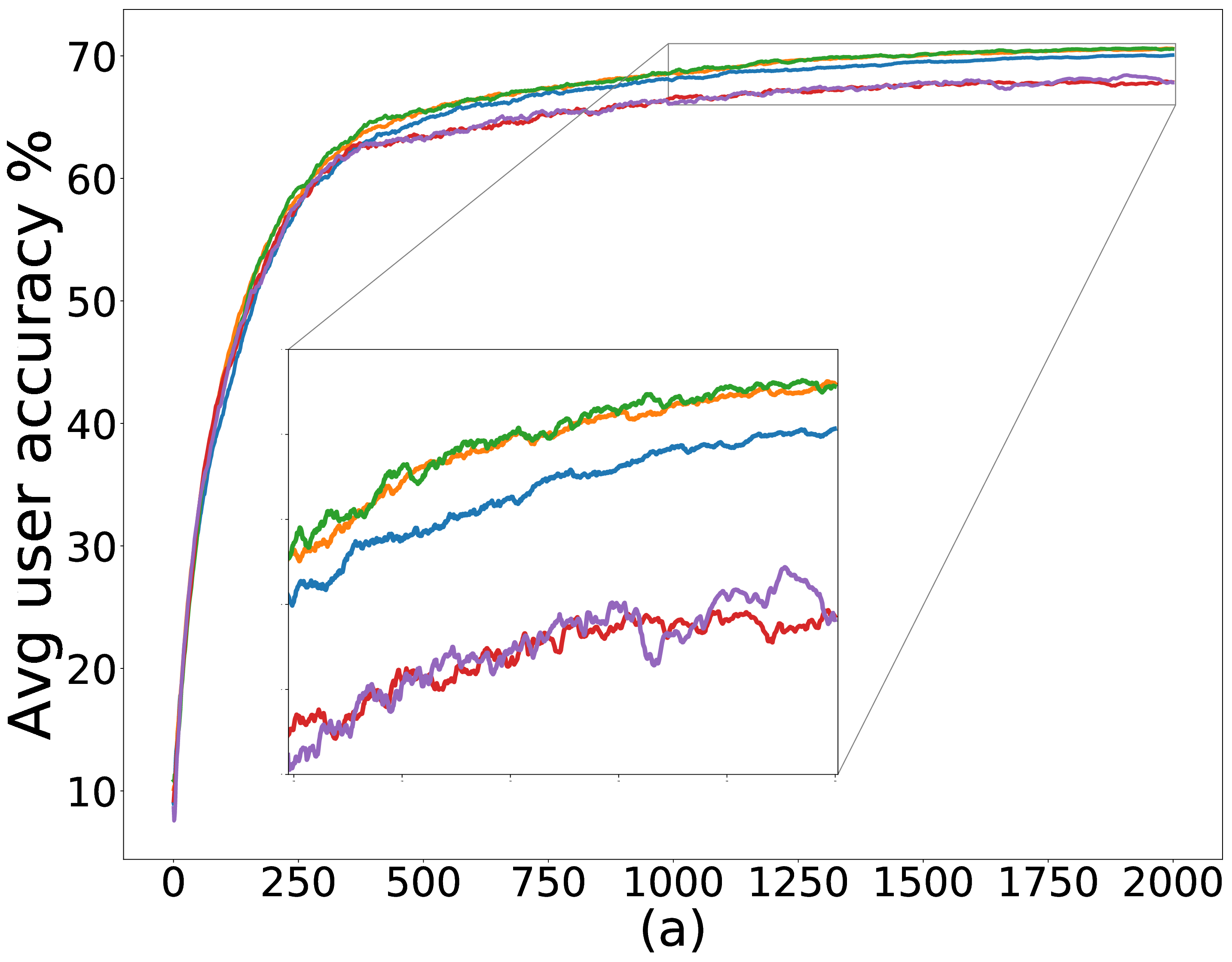}
\includegraphics[width=0.24\textwidth, valign=t]{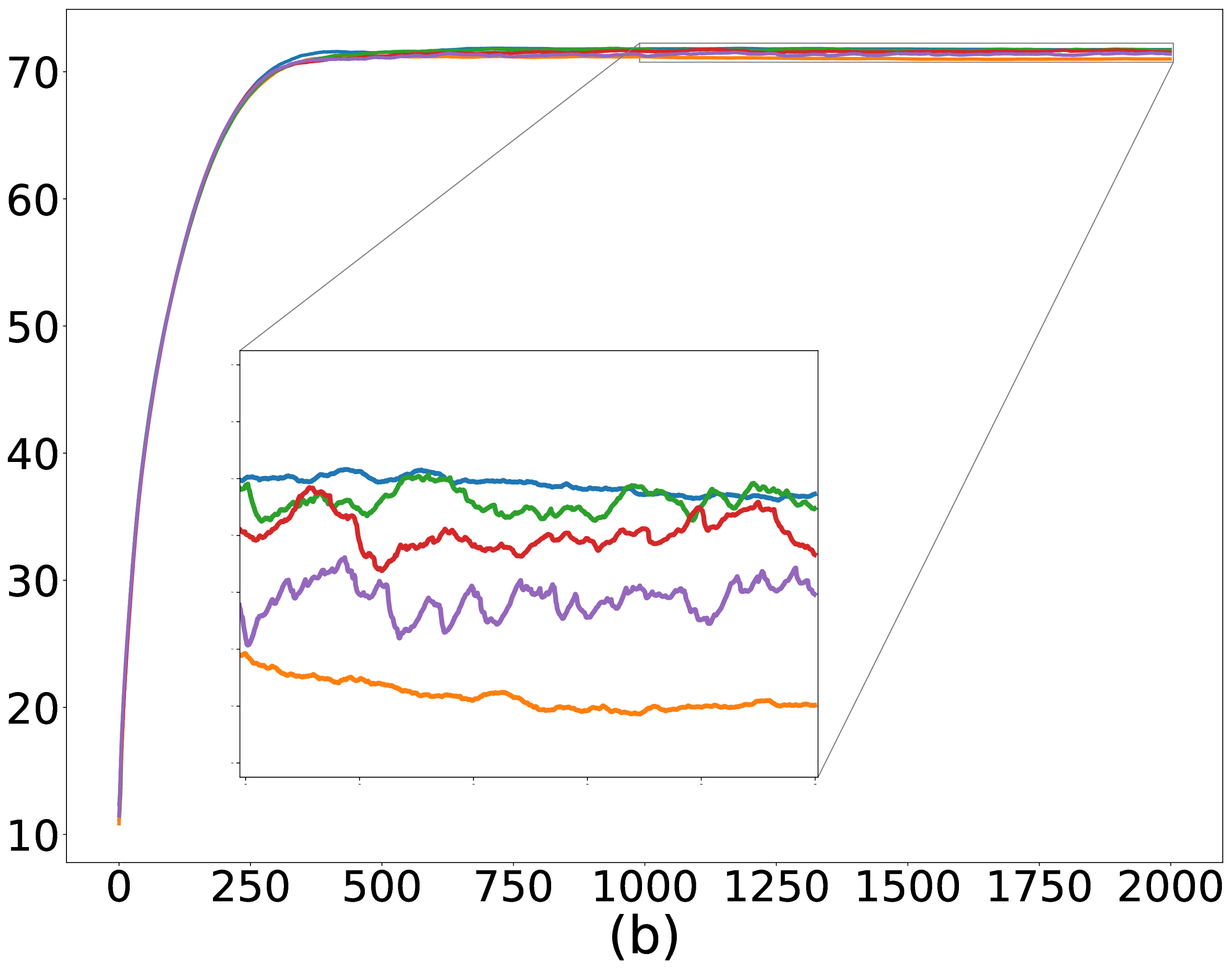}
\includegraphics[width=0.24\textwidth, valign=t]{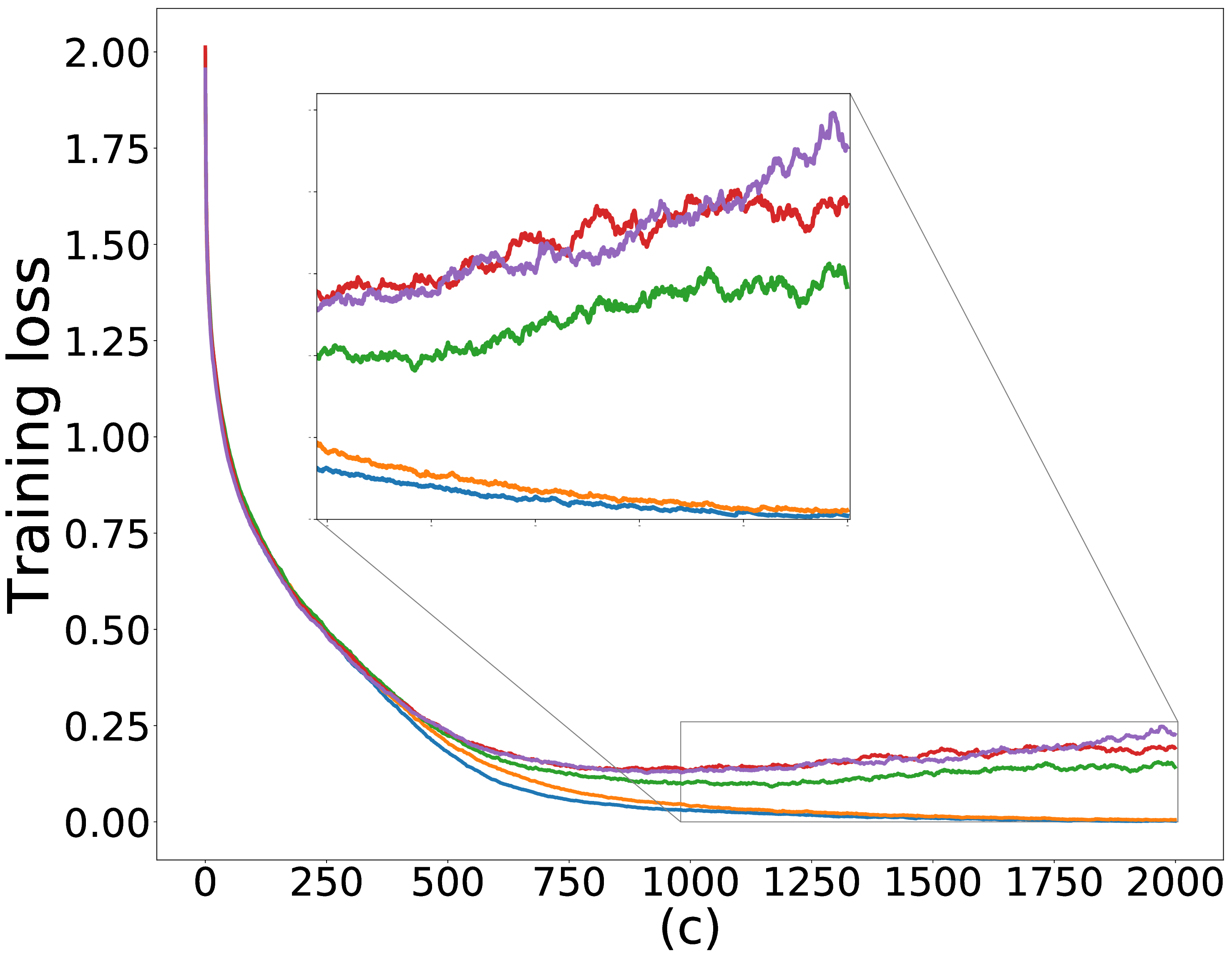}
\includegraphics[width=0.24\textwidth, valign=t]{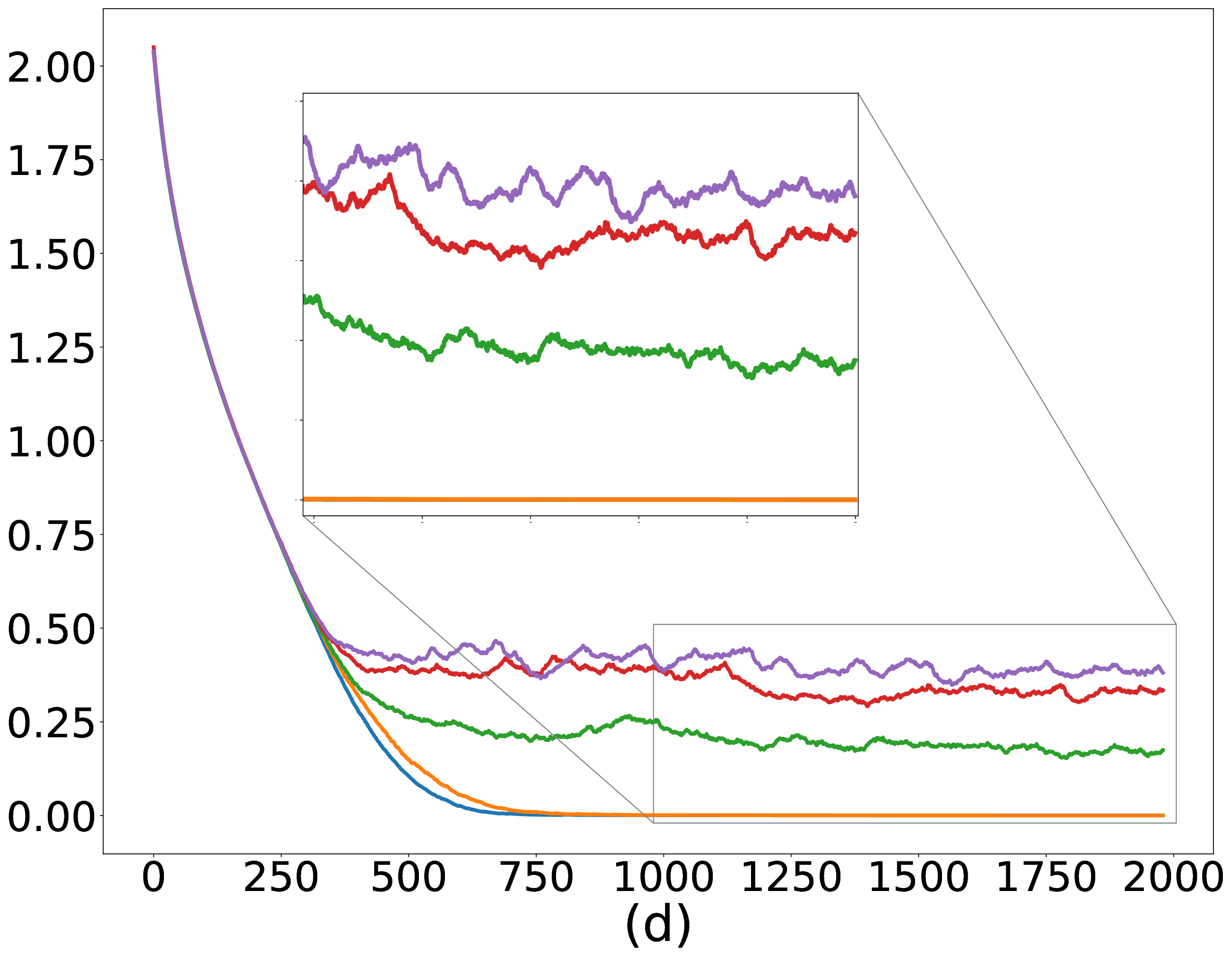}
\caption{Interpolation between \FA and \fedMGDA{} on CIFAR-10. $x$-axis is the number of communication rounds. From left to right: (\textbf{a}) and (\textbf{b}) Average user accuracy in non-iid/iid setting resp.  
(\textbf{c}) and (\textbf{d}) Uniformly averaged training loss in non-iid/iid setting resp. Results are averaged over $5$ runs with different random seeds.
}
\label{fig:interpolation_FA}
\end{figure*}
\begin{figure*}[t]
\centering
\includegraphics[width=0.425\textwidth]{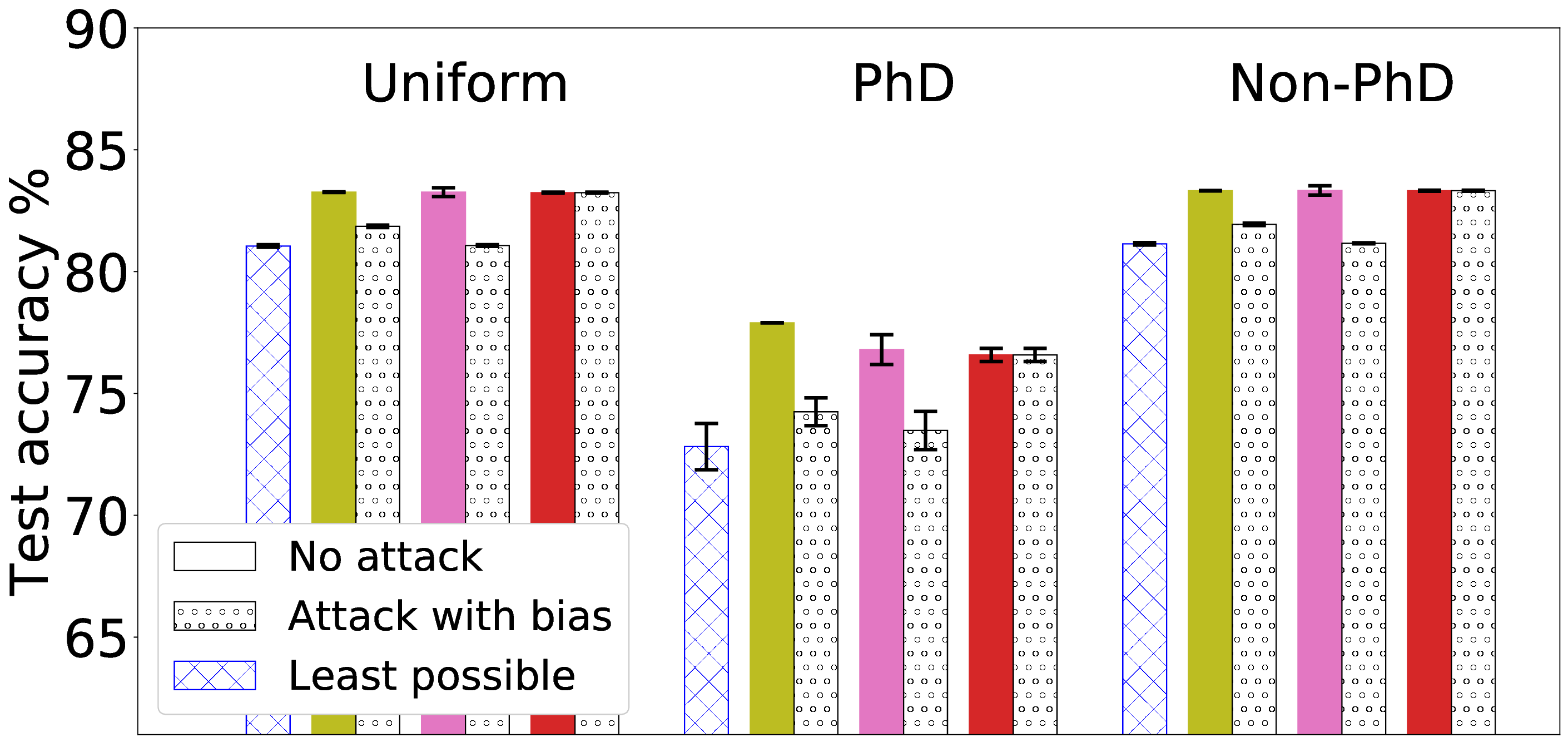}
\includegraphics[width=0.125\textwidth]{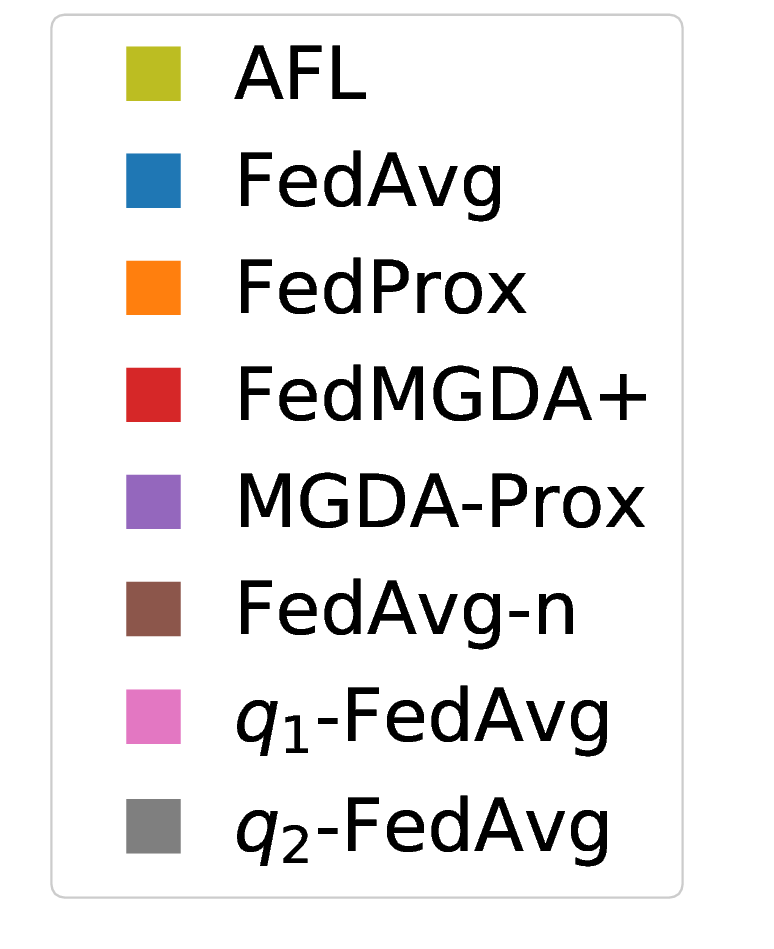}
\includegraphics[width=0.425\textwidth]{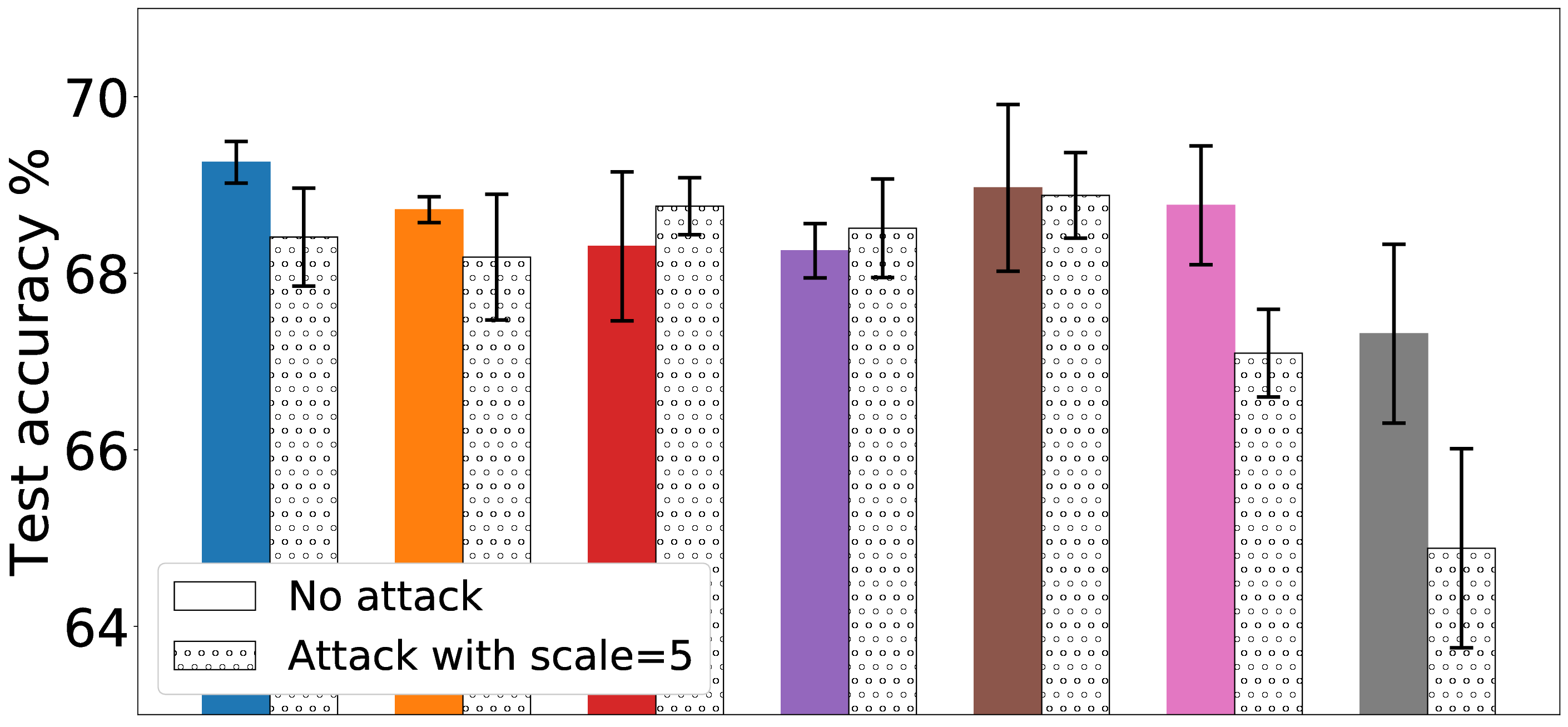}
\caption{\small 
 (\textbf{Left}) Test accuracy of SOTA algorithms on Adult dataset with adversarial biases added to the loss of PhD domain; and compared to the baseline of training only on PhD domain. The scales of biases for \afl{} and \qffl{} are different because \afl uses averaged loss while \qffl{} uses (non-averaged) total loss.
(\textbf{Right}) Test accuracy of different algorithms on CIFAR-10 in the presence of a malicious user who scales its loss function with a constant factor. 
All algorithms are run for $500$ rounds on Adult and $1500$ rounds on CIFAR-10. The reported results are averaged across $5$ runs with different random seeds. For detailed hyperparameter setting, see \Cref{sec:bias_attack}.
}
%\vspace{-1em}
\label{figure:attack}
\end{figure*}

We evaluate our algorithm \FMp on several public datasets:  CIFAR-10 \citep{Krizhevsky09}, F-MNIST \citep{XiaoRV17}, Federated EMNIST \citep{Caldas18}, Shakespeare \citep{LiSZSTS20} and Adult \citep{Dua2019}, and compare to existing \FL systems including \fedavg{} \citep{McMahanMRHA17}, \fedprox{} \citep{LiSZSTS20}, \qffl{} \citep{LiSBS20}, and \afl{}\footnote{Experiments of \afl in the original work \citep{MohriSS19} and later work that compare with it (e.g. \citep{LiSBS20}) was reported on datasets with very few clients ($2$ or $3$), possibly due to applicability reasons. We followed this convention in our work.} \citep{MohriSS19}. In addition, from the discussions in \S\ref{sec:tech}, one can envision several potential extensions of existing algorithms to improve their performance. So, we also compare to the following extensions: \fedavgn which is \fedavg{} with  gradient normalization, and \fedMGDAproxn{} which is \fedMGDAn{} with a proximal regularizer added to each user's loss function.\footnote{One can also apply the gradient normalization idea to \qffl{}; however, we observed from our experiments that the resulting algorithm is unstable particularly for large $q$ values.} We distinguish between \fedMGDAn{} and \fedMGDA{} which is a vanilla extension of MGDA to \FL. 

We point out that \FL algorithms are to be deployed on smart devices with moderate computational capabilities. Thus, the models we chose to experiment on are medium-sized (see \Cref{table:cifar-10_model,table:fmnist_model,table:emnist_model} for details), with similar complexity to the ones in \fedavg{}, \qffl{} and \afl. Due to space limits we only report some representative results in the main paper, and defer the full set of experiments to \Cref{sec:exp_appendix}.

%All experiments are run using a wide range of hyperparameters, for detailed setups, see \Cref{sec:setups}. 

\subsection{Experimental results}
In this subsection we report experimental results about our proposed algorithm \FMp and compare it with state-of-the-art (SOTA) alternatives under a variety of performance metrics, including accuracy, robustness and fairness. We remind that the accuracy metric is exactly what \FA aims to optimize during training, and hence it has some advantage in this metric over other alternative algorithms such as \FMp, \afl, and \qffl, which all aim to bring some fairness among users, perhaps at some occasional, and hopefully small, loss of accuracy.

\subsubsection{Recovering \FA}
As mentioned in \S\ref{sec:tech}, we can control the balance between the user average performance and fairness by tuning  the $\epsilon$-constraint in \cref{eq:interpolation}. 
Setting $\epsilon=0$ recovers \FA while setting $\epsilon=1$ recovers \fedMGDA. To verify this empirically,  we run \eqref{eq:interpolation} with different $\epsilon$, and report results on CIFAR-10 in  \Cref{fig:interpolation_FA} for both  iid and non-iid distributions of data (for results on F-MNIST, see \Cref{sec:fminist}).  These results confirm that changing $\epsilon$ from $0$ to $1$ yields an interpolation between \FA and \fedMGDA, as expected. Since \FA essentially optimizes the (uniformly) averaged training loss, it naturally performs the best under this metric (\Cref{fig:interpolation_FA} (c) and (d)). Nevertheless, it is interesting to note that some intermediate $\epsilon$ values actually lead to better user accuracy than \FA in the \emph{non-iid} setting (\Cref{fig:interpolation_FA} (a)).

\begin{figure*}[t]
% \centering\hspace{2pt}\includegraphics[width=1\textwidth, valign=t]{figures/violin_legend.eps}
\includegraphics[width=0.44\textwidth]{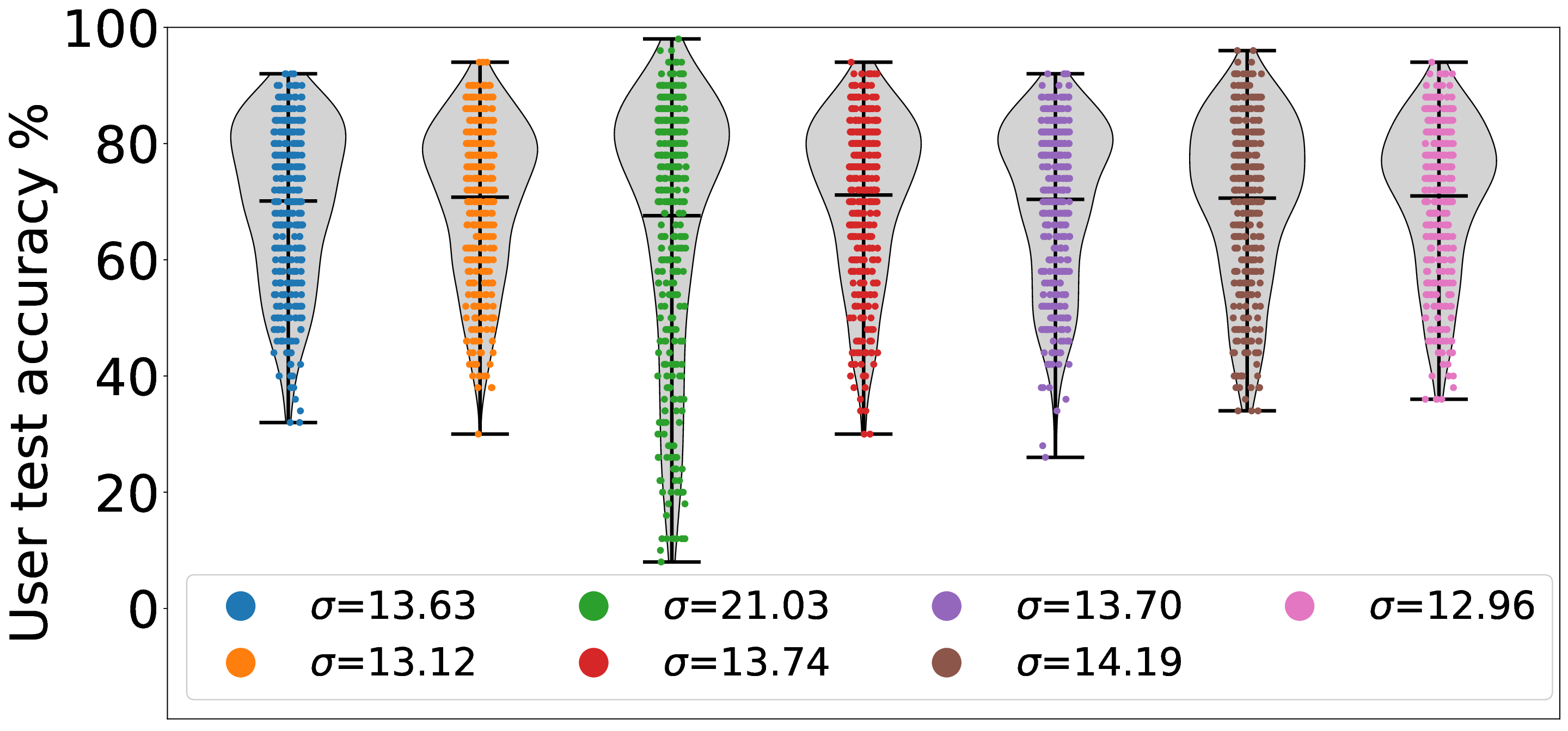} 
\includegraphics[width=0.15\textwidth]{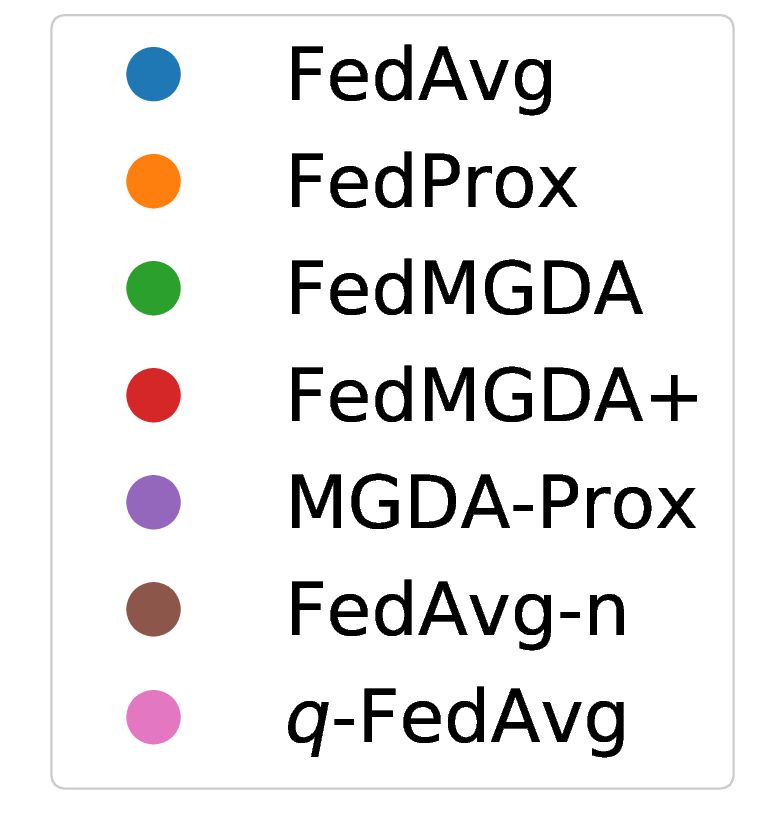} 
\includegraphics[width=0.44\textwidth]{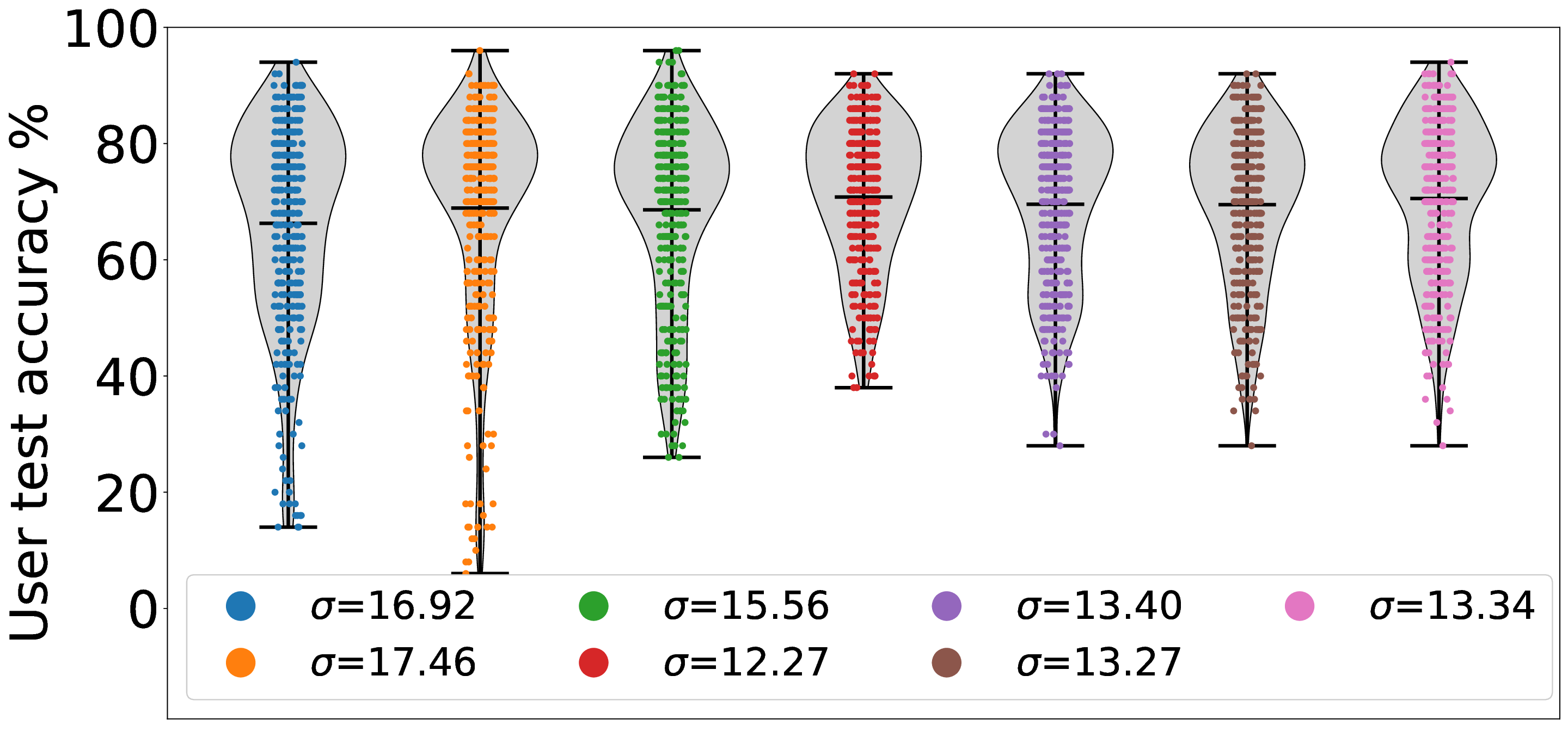}
\caption{\small Distribution of the user test accuracy on CIFAR-10: (\textbf{Left}) the algorithms are run for $2000$ communication rounds and $b=10$. The hyperparameters are: $\mu=0.01$ for \fedprox;  $\upeta=1.5$ and $\text{decay}=1/10$ for \fedMGDAn{} and \fedavg{}; $\upeta=1.0$ and $\text{decay}=1/10$  for \fedMGDAproxn{}; $q=0.5$ and $L=1.0$ for \qffl{}. (\textbf{Right}) the algorithms are run for  $3000$ communication rounds and $b=400$. The hyperparameters are:  $\mu=0.5$ for \fedprox; $\upeta=1.0$ and $\text{decay}=1/40$ for \fedMGDAn{}, \fedMGDAproxn{}, and \fedavg{};  $q=0.1$ and $L=0.1$ for \qffl{}. The reported statistics are averaged across $4$ runs with different random seeds. \label{figure:cifar-user}} 
\end{figure*}
\begin{figure*}[t]
\centering
\includegraphics[width=0.44\textwidth, valign=t]{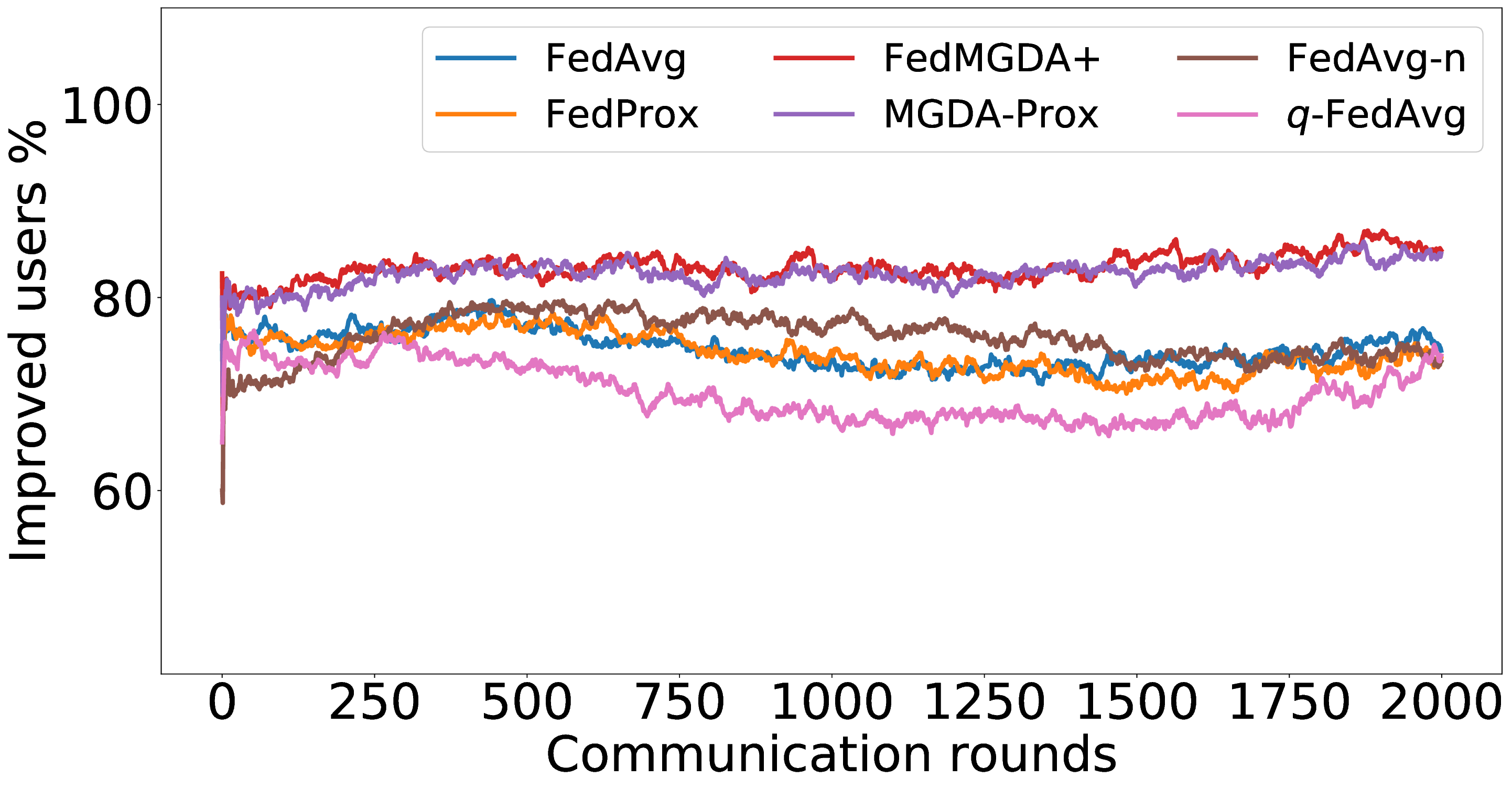} 
\hspace{2em}
\includegraphics[width=0.44\textwidth, valign=t]{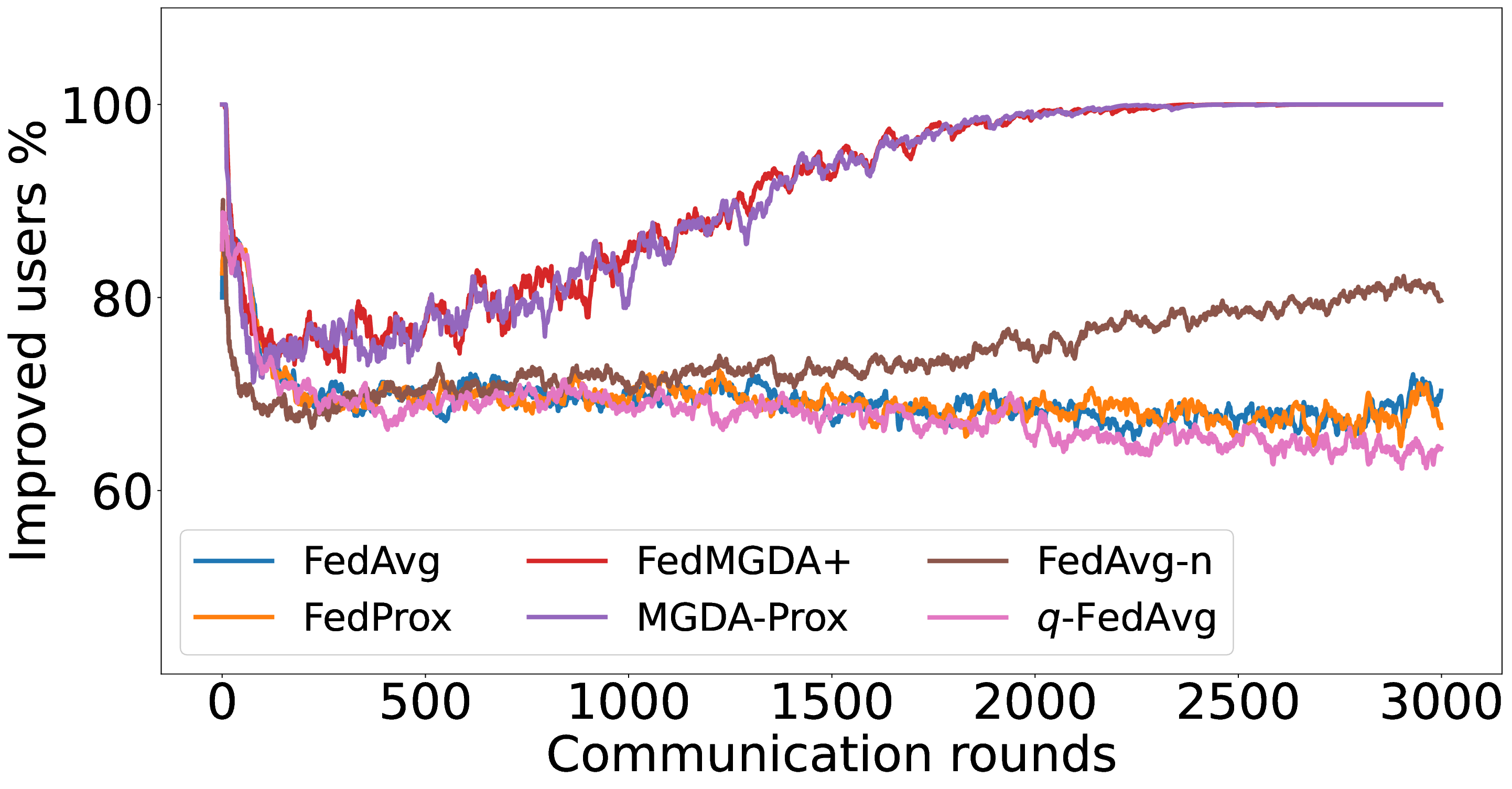}
\caption{\small The percentage of improved users in terms of training loss vs communication rounds on the CIFAR-10 dataset.  Two representative cases are shown: (\textbf{Left}) the local batch size $b=10$, and (\textbf{Right}) the local batch size $b=400$. The results are averaged across $4$ runs with different random seeds.} 
\label{fig:improved_user_cifar}
\end{figure*}

\subsubsection{Robustness}
We discussed earlier in  \S\ref{sec:tech} that the gradient normalization and MGDA’s built-in robustness allow \fedMGDAn{} to combat against certain adversarial attacks in practical \FL deployment. We now empirically evaluate  the robustness of \fedMGDAn{} against these attacks. We run various \FL algorithms in the presence of a single malicious user who aims to manipulate the system by inflating its loss. We consider an adversarial setting where the attacker participates  in each communication round and inflates its loss function by (i) adding a bias to it, or (ii) multiplying  it by a scaling factor, termed the \textit{bias} and \textit{scaling} attack, respectively. In the first experiment, we simulate a bias attack on the Adult dataset by adding a constant bias to the underrepresented user, i.e. the \texttt{PhD} domain, since it's more natural to expect an attacker to be consisted of a small number of users. In this setup, the worst performance we can get is bounded by training the model using \texttt{PhD} data \emph{only}. Results under the bias attack are presented in \Cref{figure:attack} (Left); also see \Cref{sec:bias_attack} for more results.  We observe that \afl and \qffl perform slightly better than \fedMGDAn{} without the attack; however,  their performances deteriorate  to a level close to the worst case scenario under the attack. In contrast, \fedMGDAn{} is not affected by the attack with any bias, which empirically supports our claim in \S\ref{sec:tech}. Note that we did not include \FA in this comparison since from its definition it is clear that \FA, like \FMp, is not affected by the bias attack. 
\Cref{figure:attack}~(Right) shows the results of different algorithms on CIFAR-10 with and without an adversarial scaling. As mentioned earlier, \qffl{} with gradient normalization is highly unstable particularly under the scaling attack, so we did not include its result here. From \Cref{figure:attack}~(Right) it is immediate to see that (i) the scaling attack affects all algorithms that do not employ gradient normalization; (ii) \qffl{} is the most affected under this attack; (iii) surprisingly, \fedMGDAn{} and, to a lesser extent,  \fedMGDAproxn{} actually converge to slightly better Pareto solutions, compared to their own results under no scaling attack. The above results empirically verify the robustness of \fedMGDAn{} under perhaps the most common \emph{bias} and \emph{scaling} attacks.

%\vspace{-0.5em}

\subsubsection{Fairness}
\label{sec:exp_fair}
Lastly,  we compare \fedMGDAn{}  with existing \FL algorithms using different notions of fairness on CIFAR-10. For the first experiment, we adopt the same  fairness metric as \citep{LiSBS20}, and measure fairness by calculating the variance of users' test error. We run each algorithm with different hyperparameters, and among the results, we pick the best ones  in terms of average accuracy to be shown in \Cref{figure:cifar-user}; full table of results can be found in \Cref{sec:fair_fisrt}. From this figure, we observe that (i) \fedMGDAn{} achieves the best average accuracy while its  standard deviation  is comparable with that of \qffl{}; (ii) \fedMGDAn{} significantly outperforms \fedMGDA{}, which clearly justifies our proposed modifications in \Cref{alg:FM} to the vanilla MGDA; and  
(iii)  \fedMGDAn{} outperforms \fedavgn{}, which uses the same normalization step as \fedMGDAn{},  in terms of average accuracy and standard deviation. These observations confirm the effectiveness of \FMp in inducing fairness. {We perform the same experiment on the Federated EMNIST dataset, and observed similar results, which can be found in \Cref{table:emnist-summary} and \Cref{sec:fair_first_emnist}.}

In the next experiment, we show that \fedMGDAn{} not only yields a fair final solution but also maintains fairness during the entire training process in the sense that, in each round, it refrains from sacrificing the performance of any \emph{participating} user for the sake of improving the overall performance. To the best of our knowledge,  ``fairness during training'' has not been investigated before, in spite of having great practical implications---it encourages user participation. To examine this fairness, we run several experiments on CIFAR-10 and measure the percentage of improved \emph{participants} in each communication round. Specifically, we measure the training loss before and after each round for all participating users, and report the percentage of those improved or stay unchanged.\footnote{The percentage of improved users at time $t$ is defined as ${\sum_{i\in I_t}\mathbb{I}\{f_i(\wv_{t+1})\leq f_i(\wv_t)\}}/{|I_t|},$ where $I_t$ is the selected users at time $t$, and $\mathbb{I}\{A\}$ is the indicator function of an event $A$.} \Cref{fig:improved_user_cifar} shows the percentage of improved participating users in each communication round in terms of training loss for two representative cases; see \Cref{sec:fair_second} for full results with different hyperparameters.

We can see that  \fedMGDAn{} consistently outperforms other algorithms in terms of percentage of improved users, which means that by using \fedMGDAn{}, fewer users' performances get worse after each participation. Furthermore, we notice from \Cref{fig:improved_user_cifar} (Left) that, with local batch size $b=10$, the percentage of improved users is less than  $100$\%, which can be explained as follows:   for small batch sizes (i.e., $b<|\Dc|$ where $\Dc$ represents a local dataset),  the received updates from users are not the  true gradients  of users' losses given the global model (i.e., $\gv_i\neq \nabla f_i(\wv)$); they are noisy estimates of the true gradients. Consequently, the common descent direction calculated by MGDA is noisy and may not always work for all participating users. To remove the effect of this noise, we set $b=|\Dc|$ which allows us to recover the true gradients from the users. The results are presented in  \Cref{fig:improved_user_cifar}~(Right), which confirms that, when step size decays (less overshooting), the percentage of improved users for \fedMGDAn{} reaches towards $100$\% during training, as is expected. 

\begin{table}[t]
\footnotesize
\centering
\caption{Test accuracy of users on federated EMNIST with full batch, $10$ users per rounds, local learning rate $\eta=0.1$, total communication rounds $1500$.  The reported statistics are averaged across $4$ runs with different random seeds.\label{table:emnist-summary}}
\begin{tabular}{l|ll} \toprule
        Algorithm           &  Average (\%) &  Std. (\%)  \\\midrule
        \fedMGDA{}          & $85.73 \pm 0.05$ & $14.79 \pm 0.12$ \\
        \fedMGDAn{}         & $87.60 \pm 0.20$ & $13.68 \pm 0.19$ \\
        \fedMGDAproxn{}     & $87.59 \pm 0.19$ & $13.75 \pm 0.18$ \\ 
        \fedavg             & $84.97 \pm 0.44$ & $15.25 \pm 0.36$ \\
        \fedavgn{}          & $87.57 \pm 0.09$ & $13.74 \pm 0.11$ \\
        \fedprox{}          & $84.97 \pm 0.45$ & $15.26 \pm 0.35$  \\ 
        \qffl{}             & $84.97 \pm 0.44$ & $15.25 \pm 0.37$\\ \bottomrule

\end{tabular}
\vspace{5pt}
\end{table}

\section{Conclusion}
\label{sec:con}
We have proposed a novel algorithm \FMp for federated learning. \FMp is based on multi-objective optimization and aims to converge to Pareto stationary solutions. \FMp is simple to implement, has fewer hyperparameters to tune, and complements existing \FL systems nicely. Most importantly, \FMp is robust against additive and multiplicative adversarial manipulations and ensures fairness among all participating users. We established preliminary convergence guarantees for \FMp, pointed out its connections to recent \FL algorithms, and conducted extensive experiments to verify its effectiveness. In the future we plan to formally quantify the tradeoff induced by multiple local updates and to establish some privacy guarantee for \FMp.

\section*{Acknowledgment}
Resources used in preparing this research were provided, in part, by the Province of Ontario, the Government of Canada through CIFAR, and companies sponsoring the Vector Institute. We gratefully acknowledge funding support from NSERC, the Canada CIFAR AI Chairs Program, and Waterloo-Huawei Joint Innovation Lab. We thank NVIDIA Corporation (the data science grant) for donating two Titan V GPUs that enabled in part the computation in this work.

\bibliographystyle{apalike}
\bibliography{refs}

% \documentclass[twoside]{article}

% \usepackage{aistats2021}

% If your paper is accepted, change the options for the package
% aistats2021 as follows:
%
%\usepackage[accepted]{aistats2021}
%
% This option will print headings for the title of your paper and
% headings for the authors names, plus a copyright note at the end of
% the first column of the first page.

% If you set papersize explicitly, activate the following three lines:
%\special{papersize = 8.5in, 11in}
%\setlength{\pdfpageheight}{11in}
%\setlength{\pdfpagewidth}{8.5in}

% If you use natbib package, activate the following three lines:
%\usepackage[round]{natbib}
%\renewcommand{\bibname}{References}
%\renewcommand{\bibsection}{\subsubsection*{\bibname}}

% If you use BibTeX in apalike style, activate the following line:
%\bibliographystyle{apalike}

% \begin{document}



% Supplementary material: To improve readability, you must use a single-column format for the supplementary material.
\onecolumn
% \aistatstitle{Instructions for Paper Submissions to AISTATS 2021: \\
% Supplementary Materials}

\newpage
\appendix
\section{Proofs}\label{sec:proof}
\begin{manualtheorem}{1a}[full version]
Suppose each user function $f_i$ is $L$-Lipschitz smooth (i.e., $\nabla^2 f_i \preceq L \Iv$) and $M$-Lipschitz continuous. Then, with step size $\upeta_t \in (0, \tfrac{1}{2L}]$ we have
\begin{align}
\label{eq:bound}
\min_{t=0, \ldots, T} \EE[\|J_{\fv}(\wv_t) \lambdav_t\|] \leq \frac{2[\fv(\wv_0) - \EE \fv(\wv_{T+1}) + \sum_{t=0}^T \upeta_t (M\sigma_t + L\upeta_t \sigma_t^2)]}{\sum_{t=0}^T \upeta_t},
\end{align}
where $\sigma_t^2 := \EE \|J_{\fv}(\wv_t) \lambdav_t - \hat J_{\fv}(\wv_t) \hat \lambdav_t\|^2$ is the variance of the stochastic common direction.
Moreover, if some user function $f_i$ is bounded from below, and it is possible to choose $\upeta_t$ so that $\sum_t \upeta_t = \infty, \sum_t \upeta_t \sigma_t < \infty$, then the left-hand side in \eqref{eq:bound} converges to 0.
\end{manualtheorem}
\begin{proof}
Let $\xiv_t := J_{\fv}(\wv_t) \lambdav_t - \hat J_{\fv}(\wv_t) \hat \lambdav_t$, where $\hat J_{\fv}(\wv_t) := [\hat\nabla f_1(\wv_t), \ldots, \hat\nabla f_m(\wv_t)]$ is the concatenation of stochastic gradients at each user, and 
\begin{align}
\lambdav_t = \argmin_{\lambdav\in\Delta} ~ \| J_{\fv}(\wv_t) \lambdav\|, \qquad \hat\lambdav_t = \argmin_{\hat\lambdav\in\Delta} ~ \| \hat J_{\fv}(\wv_t) \hat\lambdav\|,
\end{align}
where for the latter we also constrain $\hat\lambda_i = 0$ if the $i$-th user is not participating in round $t$.
Then, applying the quadratic bound and the update rule (we remind that comparison between vector and scalar should be understood as component-wise): 
\begin{align}
\fv(\wv_{t+1}) &\leq \fv(\wv_t) - \upeta_t J_{\fv}^\top(\wv_t) \hat J_{\fv}(\wv_t) \hat\lambdav_t + \frac{L\upeta_t^2}{2} \|\hat J_{\fv}(\wv_t) \hat \lambdav_t\|^2 
\\
&\leq 
\fv(\wv_t) - \upeta_t J_{\fv}^\top(\wv_t) J_{\fv}(\wv_t) \lambdav_t + L\upeta_t^2 \|J_{\fv}(\wv_t) \lambdav_t\|^2 + \upeta_t J_{\fv}^\top(\wv_t)\xiv_t + L\upeta_t^2 \| \xiv_t\|^2
\\
&\leq
\fv(\wv_t) - \upeta_t (1-L\upeta_t) \|J_{\fv}(\wv_t) \lambdav_t\|^2 + \upeta_t M \|\xiv_t\| + L\upeta_t^2 \| \xiv_t\|^2,
\end{align}
where we used the Lipschitz continuity $\|\nabla f_i(\wv)\|\leq M$ and the first-order optimality condition of $\lambdav_t$ so that 
\begin{align}
\forall \lambdav\in\Delta, ~ \inner{\lambdav}{J_{\fv}^\top(\wv_t) J_{\fv}(\wv_t)\lambdav_t } \geq \inner{\lambdav_t}{J_{\fv}^\top(\wv_t) J_{\fv}(\wv_t) \lambdav_t}.
\end{align}
Letting $\upeta_t \leq \tfrac{1}{2L}$, taking expectations and rearranging we obtain
\begin{align}
\min_{t=0, \ldots, T} \EE[\|J_{\fv}(\wv_t) \lambdav_t\|] \leq \frac{2[\fv(\wv_0) - \EE \fv(\wv_{T+1}) + \sum_{t=0}^T \upeta_t (M\sigma_t + L\upeta_t \sigma_t^2)]}{\sum_{t=0}^T \upeta_t},
\end{align}
where $\sigma_t^2 := \EE\|\xiv_t\|^2$. 
\end{proof}

\begin{manualtheorem}{1b}[full version]
Suppose each user function $f_i$ is $L$-Lipschitz smooth (i.e., $\nabla^2 f_i \preceq L \Iv$) and $M$-Lipschitz continuous. Then, for any number of local updates $k$, with global learning rate $\upeta_t \in (0, \tfrac{1}{2L}]$, deterministic gradient update and local learning rate $\eta^{l}_t$, we have
\begin{align}
\label{eq:bound2}
\min_{t=0, \ldots, T} \|J_{\fv}(\wv_t) \lambdav_t\| \leq \frac{2\Big[\fv(\wv_0) - \fv(\wv_{T+1}) + M^2 \upeta_t \sum_{t=0}^T \Big( \big(\varepsilon_t \sqrt{m} + \eta^l_t (k-1)\big) + L\upeta_t  \big(\varepsilon_t \sqrt{m} + \eta^l_t (k-1) \big)^2 \Big)\Big]}{\sum_{t=0}^T \upeta_t},
\end{align}
where 
%$\delta_t:= J_{\fv}(\wv_t) \lambdav_t - \tilde{J}_{\fv}(\wv_t) \tilde{ \lambdav}_t$ is the difference between the exact and approximate common direction; 
$\varepsilon_t := \|\lambdav_t-\tilde{ \lambdav}_t\|$ is the deviation between the exact and approximate (dual) weightings.
Moreover, if some user function $f_i$ is bounded from below, then the left-hand side in \eqref{eq:bound2} converges to 0 as long as $\varepsilon_t \to 0$ , $\eta^l_t \to 0$ and $\upeta_t \to 0$ with $\sum_t \upeta_t = \infty$.
\end{manualtheorem}
\begin{proof}
Let
\begin{align}
\lambdav_t = \argmin_{\lambdav\in\Delta} ~ \| J_{\fv}(\wv_t) \lambdav\|, \qquad \tilde{\lambdav}_t = \argmin_{\lambdav\in\Delta} ~ \| \tilde{ J}_{\fv}(\wv_t) \lambdav\|,
\end{align}

and $\delta_t := J_{\fv}(\wv_t) \lambdav_t - \tilde{J}_{\fv}(\wv_t) \tilde{ \lambdav}_t$, where $\tilde{ J}_{\fv}(\wv_t) := [\tilde{\nabla} f_1(\wv_t), \ldots, \tilde{\nabla} f_m(\wv_t)]$ is the concatenation of accumulated updates $\tilde{\nabla} f_i(\wv_t)$ at each user. Formally, $\tilde{\nabla} f_i(\wv_t) :=  \wv_{t} - \wv_{t}^{k}$, which denotes the difference between the initial $\wv_t$ and the final $\wv_{t}^{k}$ after $k$ local updates, for user $i$.
(Note that we have abused the notation $\wv_t$ and $\wv_t^k$ a bit for simplicity here, as they do not distinguish user $i$. This is not a big problem since the context is clear.) 

Let $\wv_t^1 := \wv_t - \nabla f_i(\wv_t)$ and $\wv_t^{j+1} := \wv_t^{j} - \eta_t^{l} \nabla f_i(\wv_t^j), ~ j ={1,\ldots,k-1}$ be the local optimization steps.

Then,
\begin{align}
    \tilde{\nabla} f_i(\wv_t) &=  \wv_{t} - \wv_{t}^{k}\\
    &= (\wv_{t}-\wv_t^1) + (\wv_t^1-\wv_t^2) + \ldots + (\wv_t^{k-1} - \wv_t^{k})\\
    &= \nabla f_i(\wv_t) + \eta^l_t \nabla f_i(\wv_t^1) + \ldots + \eta^l_t \nabla f_i(\wv_t^{k-1}),\\
\end{align}

Thus, the difference between $\tilde{\nabla} f_i(\wv_t)$ and gradient $\nabla f_i(\wv_t)$ is bounded by:

\begin{align}
    \|\tilde{\nabla} f_i(\wv_t)-\nabla f_i(\wv_t)\| &=  \|\eta^l_t \sum_{j=1}^{k-1} \nabla f_i(\wv_t^{j})\|  \\
    &\leq \eta_t^l \sum_{j=1}^{k-1}\|\nabla f_i(\wv_t^j)\| \\
    &\leq \eta^l_t (k-1) M,
\end{align}

% The difference between local update $\tilde{\nabla} f_i(\wv_t)$ and exact gradient at $\wv_t$, denoted by $\nabla f_i(\wv_t)$ is bounded by:

Thus, 
\begin{align}
\|\delta_t\| &= \|J_{\fv}(\wv_t) \lambdav_t - \tilde{ J}_{\fv}(\wv_t) \tilde{ \lambdav}_t \|
% \\
% &= \|J_{\fv}(\wv_t) \lambdav_t -  J_{\fv}(\wv_t) \tilde{ \lambdav}_t + J_{\fv}(\wv_t) \tilde{ \lambdav}_t -\tilde{J}_{\fv}(\wv_t) \tilde{ \lambdav}_t \|
\\
&\leq \|J_{\fv}(\wv_t) \lambdav_t -  J_{\fv}(\wv_t) \tilde{ \lambdav}_t\| + \|J_{\fv}(\wv_t) \tilde{ \lambdav}_t -\tilde{J}_{\fv}(\wv_t) \tilde{ \lambdav}_t \|
\\
&\leq \varepsilon_t \sqrt{m} M + \eta^{l}_{t} (k-1) M ,
\label{eq:delta_bound}
\end{align}
the last step comes from matrix norm inequality on the first term, and triangular inequality on the second term. Note that $\|\delta_t\|$ vanishes when $\varepsilon_t \to 0$ and $\eta^l_t \to 0$.

Then, applying the quadratic upper bound, we have
\begin{align}
\fv(\wv_{t+1}) &\leq \fv(\wv_t) - \upeta_t J_{\fv}^\top(\wv_t) \tilde{J}_{\fv}(\wv_t) \tilde{\lambdav}_t + \frac{L\upeta_t^2}{2} \|\tilde{J}_{\fv}(\wv_t) \tilde{\lambdav}_t\|^2 
\\
&= 
\fv(\wv_t) - \upeta_t J_{\fv}^\top(\wv_t) J_{\fv}(\wv_t) \lambdav_t + L\upeta_t^2 \|J_{\fv}(\wv_t) \lambdav_t\|^2 + \upeta_t J_{\fv}^\top(\wv_t)\delta_t + L\upeta_t^2 \| \delta_t\|^2
\\
&\leq
\fv(\wv_t) - \upeta_t (1-L\upeta_t) \|J_{\fv}(\wv_t) \lambdav_t\|^2 + \upeta_t M \|\delta_t\| + L\upeta_t^2 \| \delta_t\|^2,
\end{align}

Letting $\upeta_t \leq \tfrac{1}{2L}$, telescoping and rearranging we obtain
\begin{align}
\min_{t=0, \ldots, T} \|J_{\fv}(\wv_t) \lambdav_t\| \leq \frac{2[\fv(\wv_0) - \fv(\wv_{T+1}) + \sum_{t=0}^T \upeta_t (M\|\delta_t\| + L\upeta_t \|\delta_t\|^2)]}{\sum_{t=0}^T \upeta_t},
\end{align}
substitute $\|\delta_t\|$ with \eqref{eq:delta_bound}, and we get \eqref{eq:bound2}.
\vspace{2ex}

Finally, if $\varepsilon_t \to 0$ and $\eta_t^{l} \to 0$, then $\delta_t \to 0$ and hence 
%we take $\varepsilon_t \propto \frac{1}{t}$, $\eta_t^{l} \propto \frac{1}{t}$ and $\upeta_t \propto \frac{1}{t}$, then 
the right-hand side in \eqref{eq:bound2} $\to 0 $ when $T \to \infty$, in which case the left-hand side $\min_{t=0, \ldots, T} \|J_{\fv}(\wv_t) \lambdav_t\|$ converges to $0$ as well.
\end{proof}

\begin{manualtheorem}{2}[full version]
\label{thm:conv}
Suppose each user function $f_i$ is $\sigma$-strongly convex (i.e. $\nabla^2 f_i \succeq\sigma \Iv$) and $M$-Lipschitz  continuous. Suppose at each round $t$ \FM includes some function $f_{v_t}$ such that 
\begin{align}
f_{v_t}(\wv_t) - f_{v_t}(\wv_t^*) \geq \tfrac{\ell_t}{2} \|\wv_t - \wv_t^*\|^2,
\end{align}
where $\wv_t^*$ is the projection of $\wv_t$ to the Pareto stationary set $W^*$ of \eqref{eq:MOM}. Assume $\EE[\lambda_{v_t} \ell_t + \sigma_t |\wv_t] \geq c > 0$, then 
\begin{align}
\EE[\|\wv_{t+1} - \wv_{t+1}^*\|^2] \leq \pi_t(1-c\upeta_0) \EE[\|\wv_0-\wv_0^*\|^2] + \sum_{s=0}^t \frac{\pi_t}{\pi_s} \upeta_s^2 M^2,
\end{align}
where $\pi_t = \prod_{s=1}^t \upeta_s$ and $\pi_0 = 1$. In particular,
\begin{itemize}[topsep=0pt,leftmargin=*]
\item if $\sum_t \upeta_t = \infty, \sum_t \upeta_t^2 < \infty$, then $\EE[\|\wv_{t} - \wv_{t}^*\|^2] \to 0$ and $\wv_{t}$ converges to the Pareto stationarity set $W^*$ almost surely;
\item with the choice $\upeta_t = \tfrac{2}{c(t+2)}$ we have 
\begin{align}
\EE[\|\wv_{t} - \wv_{t}^*\|^2] \leq \frac{4M^2}{c^2(t+3)}.
\end{align}
\end{itemize}
\end{manualtheorem}
\begin{proof}
For each user $i$, let us define the function
\begin{align}
\hat f_i(\wv, I) := I_i f_i(\wv),
\end{align}
where the random variable $I \in \{0, 1\}^\numuser$ indicates which user participates at a particular round. Clearly, we have $\EE \hat f_i(\wv, I) = f_i(\wv) \EE I_i $. Therefore, our multi-objective minimization problem is equivalent as:
\begin{align}
\label{eq:stoc}
\min_{\wv} ~ \left\{ \EE \hat f_1(\wv, I), \ldots, \EE\hat f_m(\wv, I) \right\},
\end{align} 
since positive scaling does not change Pareto stationarity. (If one prefers, we can also normalize the stochastic functions $\hat f_i(\wv, I)$ so that the unbiasedness property $\EE \hat f_i(\wv, I) = f_i(\wv)$ holds.)

We now proceed as in \citet{MercierPD18} and provide a slightly sharper analysis.
Let us denote $\wv_t^*$ the projection of $\wv_t$ to the Pareto-stationary set $W^*$ of \eqref{eq:stoc}, \ie, 
\begin{align}
\wv_t^* = \argmin_{\wv \in W^*} ~ \|\wv_t - \wv\|.
\end{align}
Then, 
\begin{align}
\|\wv_{t+1} - \wv_{t+1}^*\|^2 & \leq \|\wv_{t+1} - \wv_{t}^*\|^2 \\
&=  \|\wv_t - \upeta_t \dv_t - \wv_t^*\|^2 \\
\label{eq:sub} &= \|\wv_t - \wv_t^*\|^2 - 2\upeta_t \inner{\wv_t - \wv_t^*}{\dv_t} + \upeta_t^2 \|\dv_t\|^2.
\end{align}
To bound the middle term, we have from our assumption:
\begin{align}
\exists v_t, ~ \hat f_{v_t}(\wv_t, I_t) - \hat f_{v_t}(\wv_t^*, I_t) &\geq \tfrac{\ell_t}{2} \|\wv_t - \wv_t^*\|^2, \\
\forall i, ~ \hat f_i(\wv_t, I_t) - \hat f_i(\wv_t^*, I_t) &\geq 0,
\end{align}
where the second inequality follows from the definition of $\wv_t^*$. Therefore, 
\begin{align}
\inner{\wv_t - \wv_t^*}{\dv_t} &= \inner{\wv_t - \wv_t^*}{\sum_{i: I_i = 1} \lambda_i \nabla f_i(\wv_t)} \\
&\geq \sum_{i : I_i = 1} \lambda_i \left( f_i(\wv_t) - f_i(\wv_t^*) \right) + \tfrac{\sigma_t}{2} \|\wv_t - \wv_t^*\|^2\\
& = \sum_i \lambda_i \left( \hat f_i(\wv_t, I_t) - \hat f_i(\wv_t^*, I_t) \right) + \tfrac{\sigma_t}{2} \|\wv_t - \wv_t^*\|^2\\
& \geq \tfrac{\lambda_{v_t}\ell_t + \sigma_t}{2} \|\wv_t - \wv_t^*\|^2.
\end{align}
Continuing from \eqref{eq:sub} and taking conditional expectation:
\begin{align}
\EE[\|\wv_{t+1} - \wv_{t+1}^*\|^2 | \wv_t] &\leq  (1 - c_t \upeta_t )\|\wv_t - \wv_t^*\|^2 + \upeta_t^2 M^2,
\end{align}
where $c_t := \EE[\lambda_{v_t}\ell_{t} + \sigma_t | \wv_t] \geq c > 0$. Taking expectation  we obtain the familiar recursion:
\begin{align}
\EE[\|\wv_{t+1} - \wv_{t+1}^*\|^2] &\leq  (1-c\upeta_t) \EE[\|\wv_t - \wv_t^*\|^2] + \upeta_t^2 M^2,
\end{align}
from which we derive
\begin{align}
\EE[\|\wv_{t+1} - \wv_{t+1}^*\|^2] &\leq  \pi_t (1 - c\upeta_0) \EE[\|\wv_0- \wv_0^*\|^2] + \sum_{s=0}^t \tfrac{\pi_t}{\pi_s} \upeta_s^2 M^2,
\end{align}
where $\pi_t = \prod_{s=1}^t (1 - c\upeta_s)$ and $\pi_0 = 1$. Since $\pi_t \to 0 \iff \sum_t \upeta_t = \infty$, we know 
\begin{align}
\EE[\|\wv_{t+1} - \wv_{t+1}^*\|^2] \to 0 
\end{align}
if $\sum_t \upeta_t = \infty$ and $\sum_t \upeta_t^2 < \infty$.

Setting $\upeta_t = \tfrac{2}{c(t+2)}$ we obtain $\pi_t = \tfrac{2}{(t+2)(t+1)}$ and by induction 
\begin{align}
\sum_{s=0}^t \frac{\pi_t}{\pi_s} \upeta_s^2 = \frac{4}{c^2(t+2)(t+1)} \sum_{s=0}^t \frac{s+1}{s+2} \leq \frac{4}{c^2(t+4)},
\end{align}
whence
\begin{align}
\EE[\|\wv_{t+1} - \wv_{t+1}^*\|^2] \leq \frac{4M^2}{c^2(t+4)}.
\end{align}

Using a standard supermartingale argument we can also prove that 
\begin{align}
\wv_t - \wv_t^* \to 0 \text{ almost surely}.
\end{align}
The proof is well-known in stochastic optimization hence omitted (or see \citet[Theorem 5]{MercierPD18} for details). 
\end{proof}

\newpage

\section{Full experimental results}
\label{sec:exp_appendix}

In this section we provide additional results that are deferred from the main paper.
%\newpage

\subsection{Recovering \fedavg{} full results: results on Fashion-MNIST and CIFAR-10}
\label{sec:fminist}
Complementary to the results shown in \Cref{fig:interpolation_FA}, \Cref{fig:fmnist_iid} and \Cref{fig:fmnist_noniid} summarize similar results on the F-MNIST dataset, while \Cref{fig:cifar_log_interpolate} depicts the training losses on CIFAR-10 dataset in log-scale.

\begin{figure}[ht]
\centering
\includegraphics[width=0.48\columnwidth, valign=t]{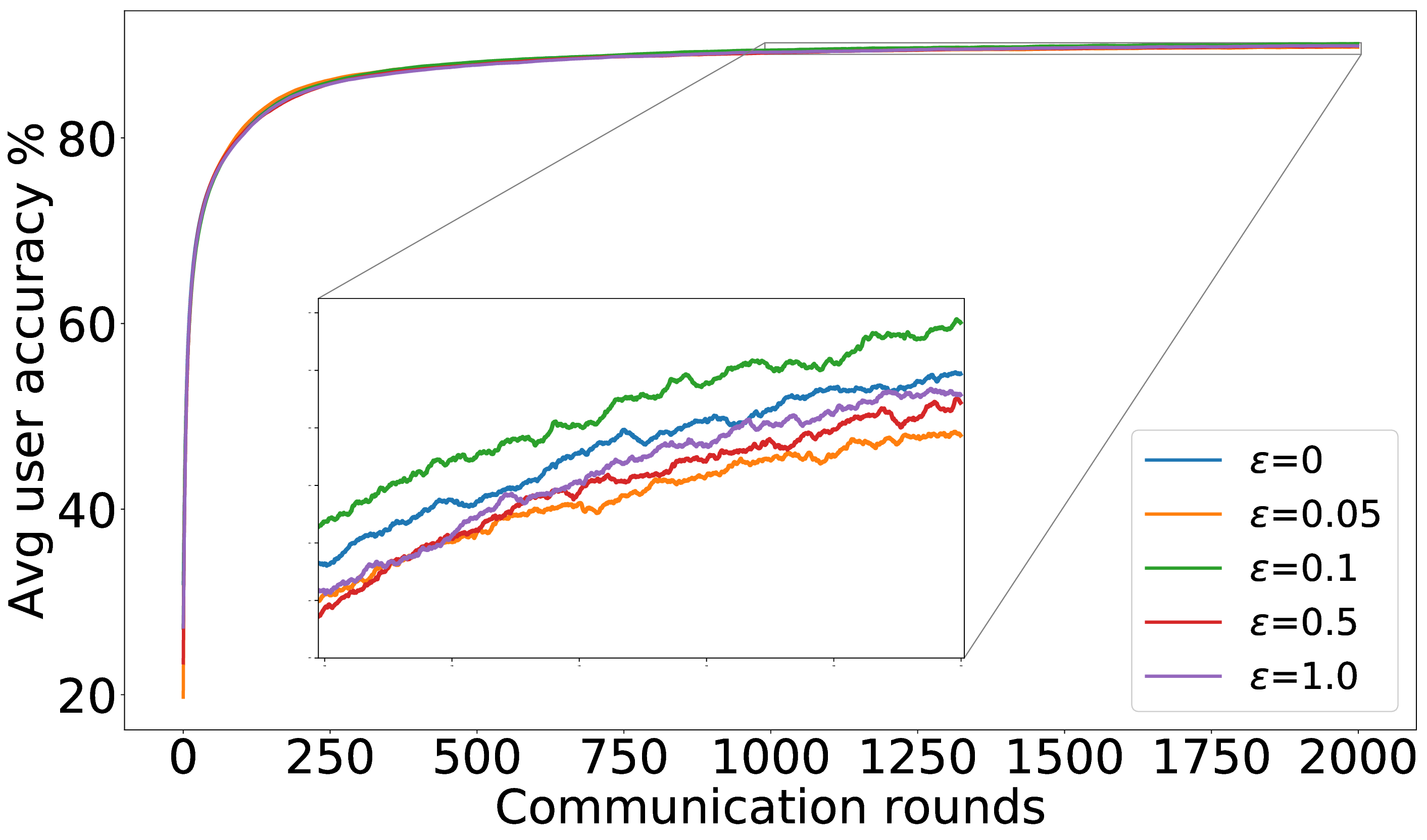}
\includegraphics[width=0.48\columnwidth, valign=t]{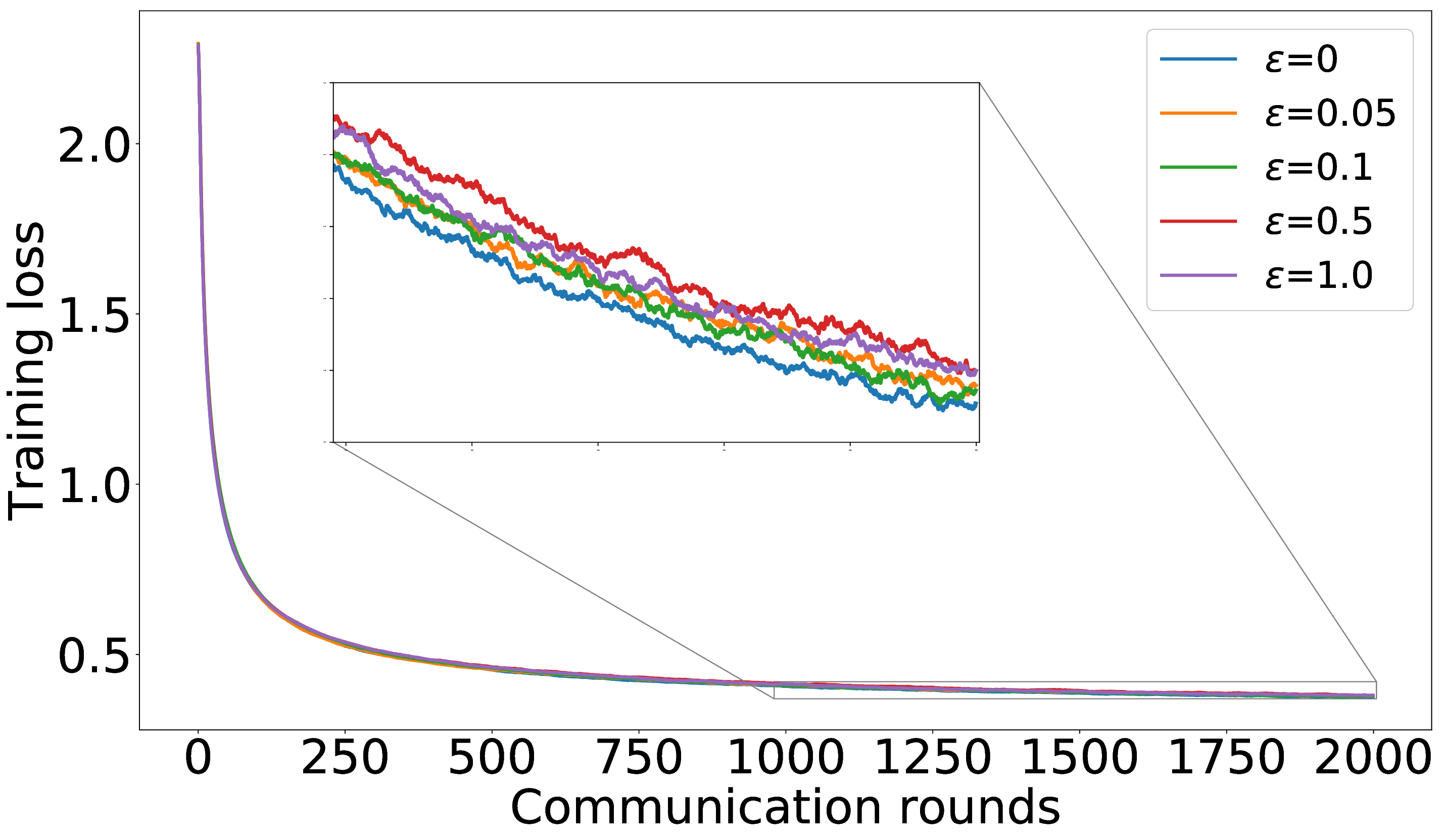}
\vspace{-1em}
\caption{Interpolation between \FA and \fedMGDA{} (F-MNIST, iid setting). (\textbf{Left}) Average user accuracy. (\textbf{Right}) Uniformly averaged training loss. Results are averaged over $5$ runs with different random seeds.}
\label{fig:fmnist_iid}
\end{figure}
\begin{figure}[ht]
\centering
\includegraphics[width=0.48\columnwidth, valign=t]{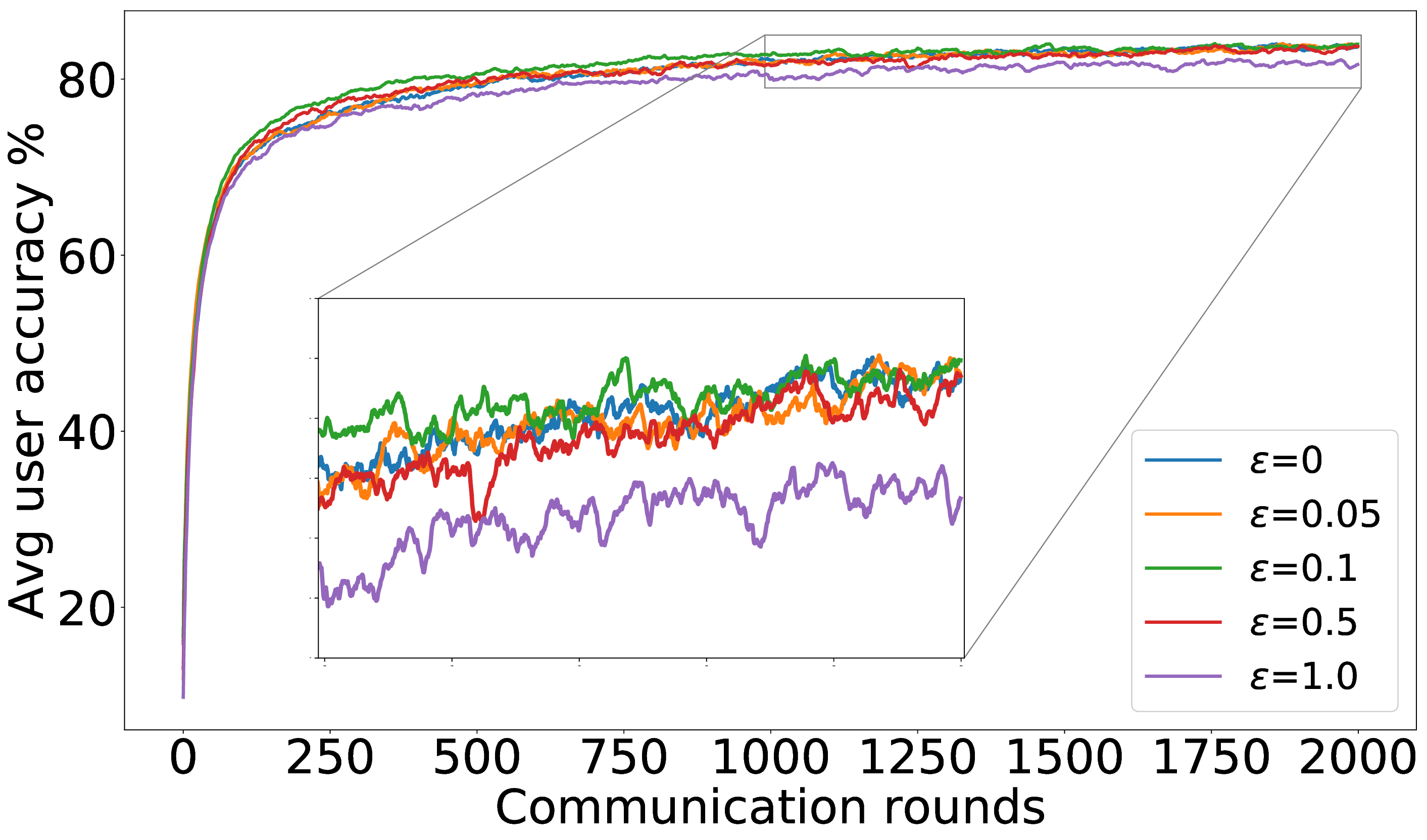}
\includegraphics[width=0.48\columnwidth, valign=t]{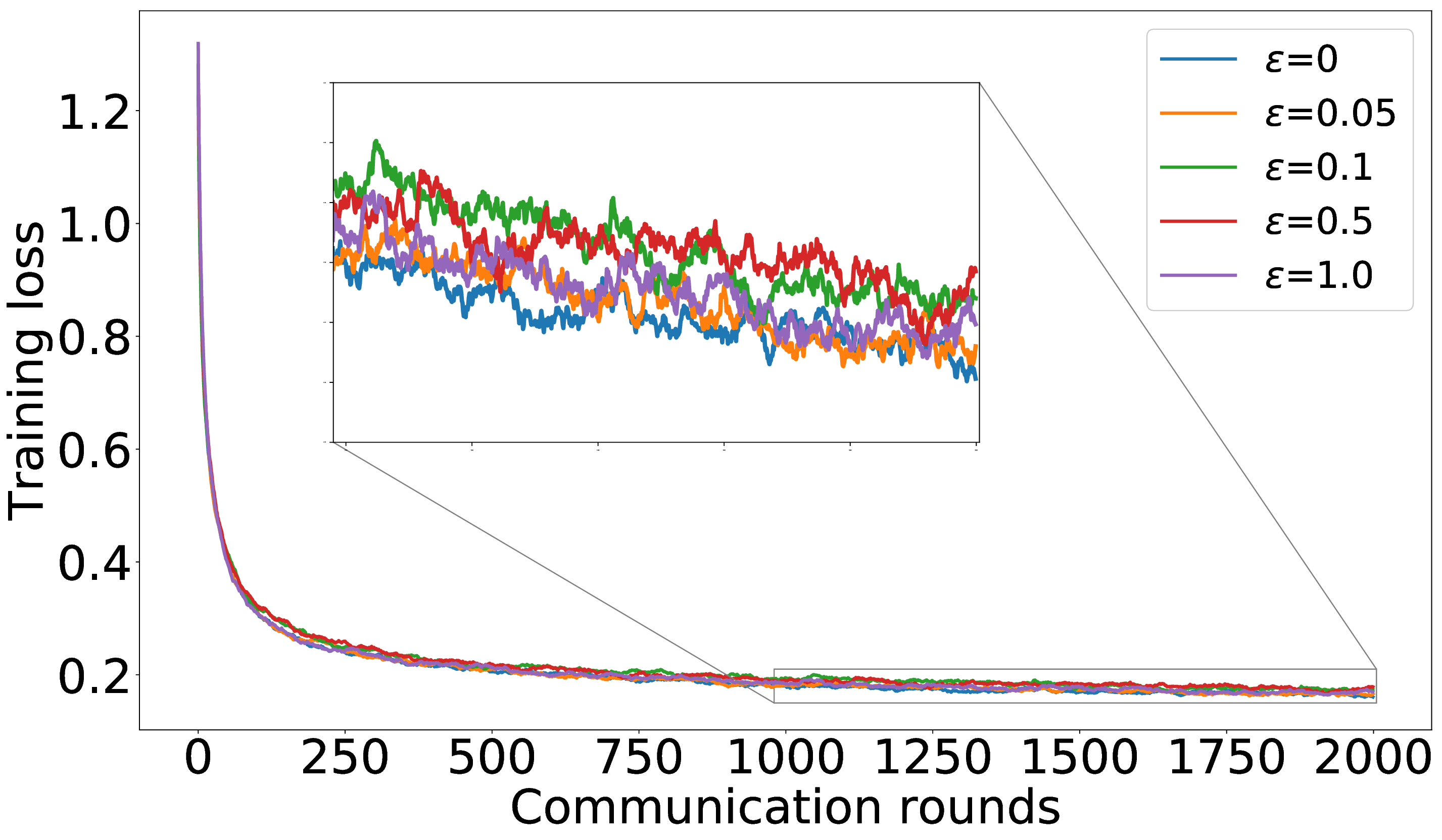}
%\vspace{-1em}
\caption{Interpolation between \FA and \fedMGDA{} (F-MNIST, non-iid setting). (\textbf{Left}) Average user accuracy. (\textbf{Right}) Uniformly averaged training loss. Results are averaged over $5$ runs with different random seeds.}
\label{fig:fmnist_noniid}
\end{figure}
\begin{figure}[ht]
\centering
\includegraphics[width=0.48\columnwidth,valign=t]{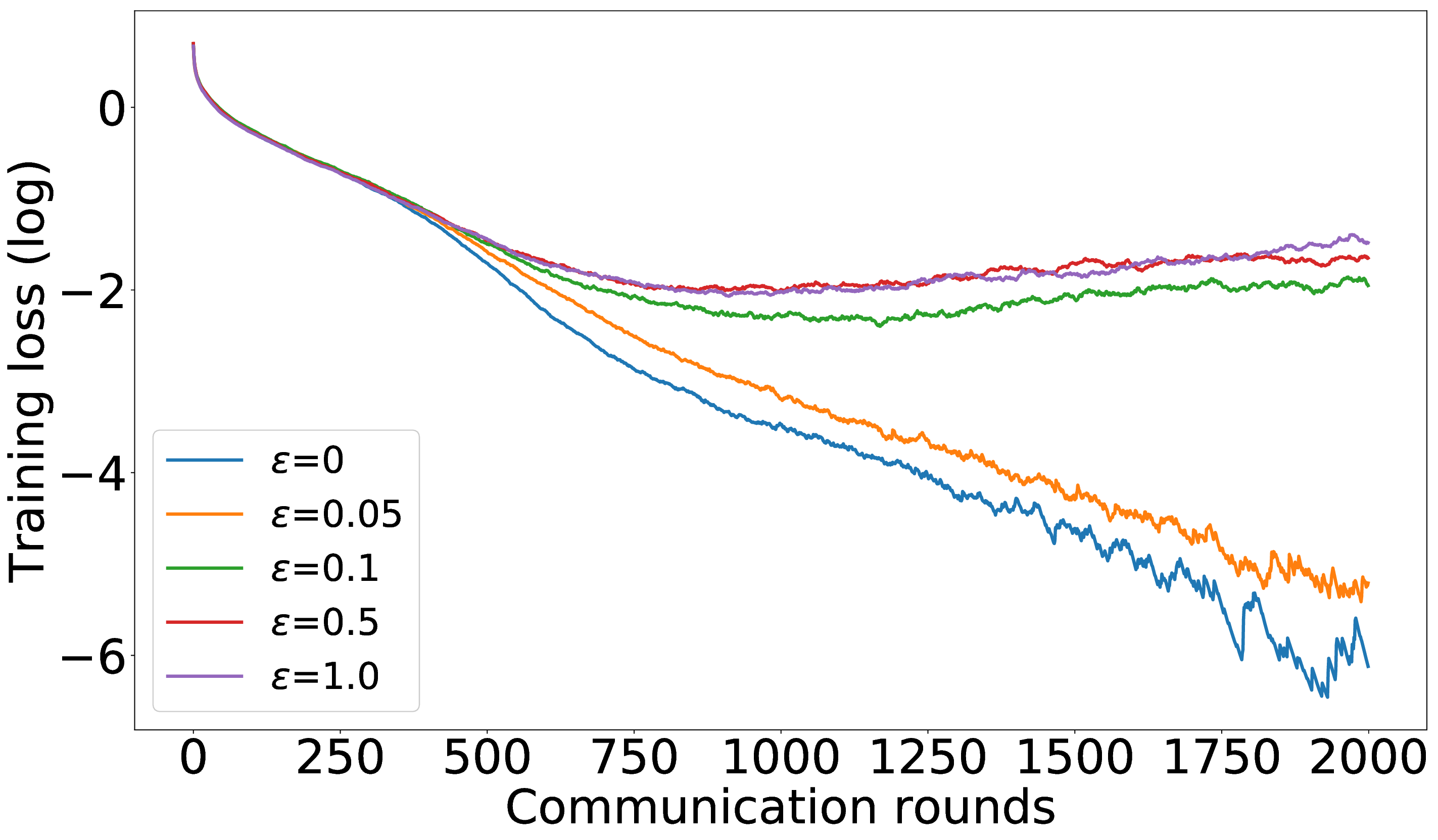}
\includegraphics[width=0.48\columnwidth,valign=t]{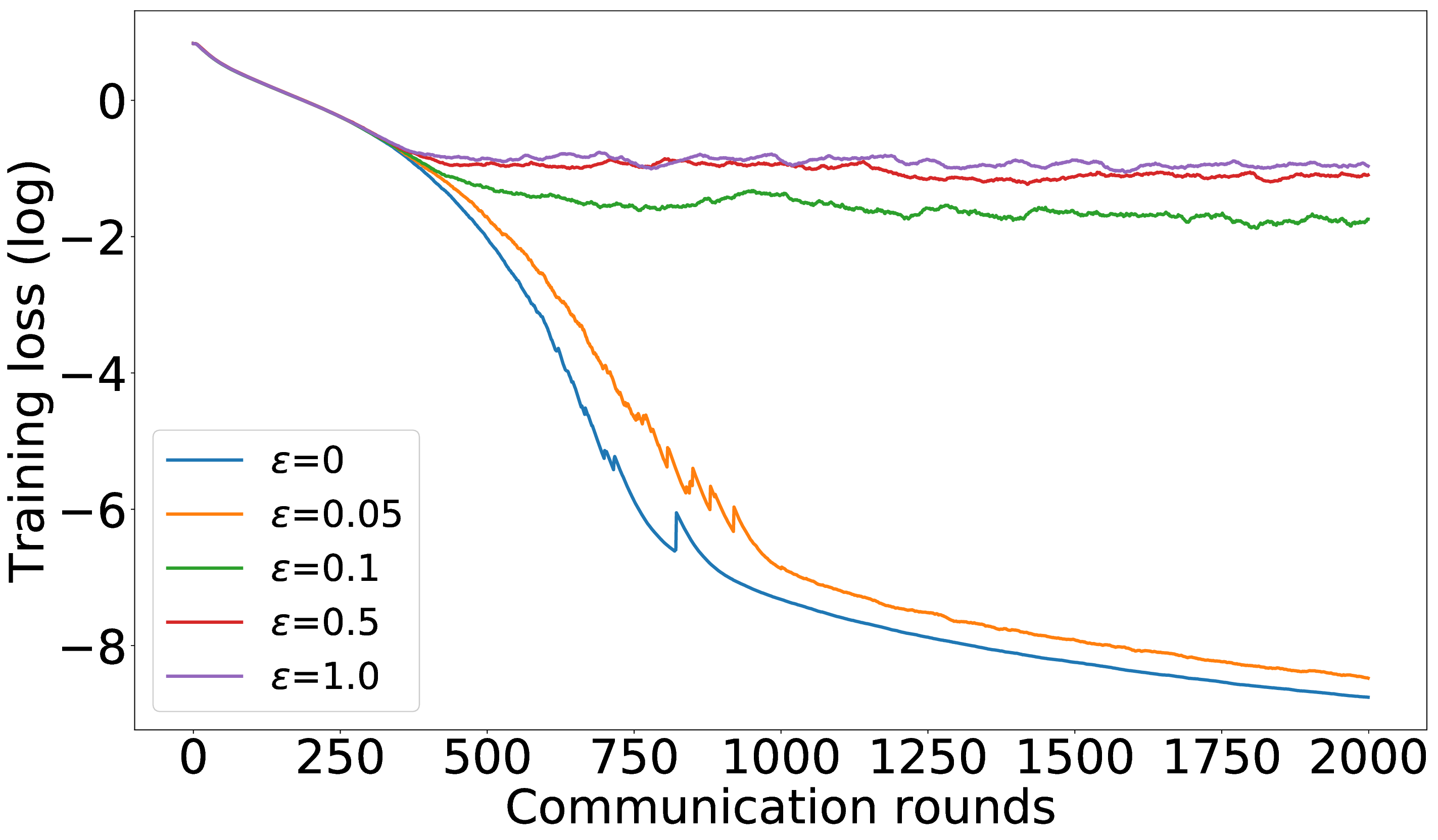}
%\vspace{-1em}
\caption{Interpolation between \FA and \fedMGDA{} (CIFAR-10). Both figures plot the uniformly averaged training loss in log-scale. (\textbf{Left}) non-iid setting. (\textbf{Right}) iid setting. Results are averaged over $5$ runs with different random seeds.}
%\vspace{-2em}
\label{fig:cifar_log_interpolate}
\end{figure}

\newpage

\subsection{Robustness full results: bias attack on Adult dataset}
\label{sec:bias_attack}
\Cref{table:adult_bias} shows the full results of the experiment presented in \Cref{figure:attack}~(Left). 
\begin{table}[ht]
\footnotesize
\centering
\caption{ Test accuracy of SOTA algorithms on Adult dataset with various scales of adversarial bias added to the domain loss of \texttt{PhD}; and compared to the baseline of training only on the \texttt{PhD} domain. 
The scale of bias for \afl{} is different from \qffl{} since \afl uses averaged loss while \qffl{} uses (non-averaged) total loss. The algorithms are run for $500$ rounds, and the reported results are averaged across $5$ runs with different random seeds.\label{table:adult_bias}}
\begin{tabular}{ll|lll} \toprule
    Name             &  Bias     & Uniform    & PhD        & Non-PhD    \\ \midrule
\afl{}               &  $0$      & $83.26 \pm 0.01$ & $77.90 \pm 0.00$ & $83.32 \pm 0.01$ \\ 
\afl{}               &  $0.01$   & $83.28 \pm 0.03$ & $76.58 \pm 0.27$ & $83.36 \pm 0.03$ \\ 
\afl{}               &  $0.1$    & $82.30 \pm 0.04$ & $74.59 \pm 0.00$ & $82.39 \pm 0.04$ \\ 
\afl{}               &  $1$      & $81.86 \pm 0.05$ & $74.25 \pm 0.57$ & $81.94 \pm 0.05$ \\ \midrule
\qffl, $q=5$         &  $0$      & $83.26 \pm 0.18$ & $76.80 \pm 0.61$ & $83.33 \pm 0.19$\\ 
\qffl, $q=5$         &  $1000$   & $83.34 \pm 0.04$ & $76.57 \pm 0.44$ & $83.41 \pm 0.04$\\ 
\qffl, $q=5$         &  $5000$   & $81.19 \pm 0.03$ & $74.14 \pm 0.41$ & $81.27 \pm 0.03$\\ 
\qffl, $q=5$         &  $10000$  & $81.07 \pm 0.03$ & $73.48 \pm 0.78$ & $81.16 \pm 0.02$\\ \midrule
\qffl, $q=2$         &  $0$      & $83.30 \pm 0.09$ & $76.46 \pm 0.56$ & $83.38 \pm 0.09$\\ 
\qffl, $q=2$         &  $1000$   & $83.33 \pm 0.04$ & $76.24 \pm 0.00$ & $83.41 \pm 0.04$\\ 
\qffl, $q=2$         &  $5000$   & $83.11 \pm 0.03$ & $75.69 \pm 0.00$ & $83.20 \pm 0.03$\\ 
\qffl, $q=2$         &  $10000$  & $82.50 \pm 0.07$ & $75.69 \pm 0.00$ & $82.58 \pm 0.07$\\ \midrule
\qffl, $q=0.1$       &  $0$      & $83.44 \pm 0.06$ & $76.46 \pm 0.56$ & $83.52 \pm 0.07$\\ 
\qffl, $q=0.1$       &  $1000$   & $83.34 \pm 0.03$ & $76.35 \pm 0.41$ & $83.42 \pm 0.02$\\ 
\qffl, $q=0.1$       &  $5000$   & $83.35 \pm 0.03$ & $76.57 \pm 0.66$ & $83.42 \pm 0.03$\\ 
\qffl, $q=0.1$       &  $10000$  & $83.36 \pm 0.05$ & $76.80 \pm 0.49$ & $83.43 \pm 0.05$\\ \midrule
\fedMGDAn{}          & Arbitrary & $83.24 \pm 0.02$ &  $76.58 \pm 0.27$ &  $83.32 \pm 0.02$ \\ \midrule
Baseline\_PhD        &          &    $81.05 \pm 0.05$ & $72.82 \pm 0.95$ & $81.14 \pm 0.05$       \\ \bottomrule
\end{tabular}
\vspace{5pt}
\end{table}

\textbf{The hyper-parameter setting for \Cref{figure:attack}.}
\textbf{Left}: step sizes $\gamma_{\lambda}=0.5$, $\gamma_{w}=0.01$, and bias$=1$ for \afl{}; $q = 5$, and bias$=10000$ for \qffl{}; $\upeta=1$, decay=$1/3$, and arbitrary bias for \fedMGDAn{}. 
\textbf{Right}: $\mu=0.01$ for \fedprox{}; $q_1=0.5$ and $q_2=2.0$ for \qffl{}; $\upeta=1$ and decay=$1/10$ for \fedMGDAn{} and \fedavgn{}. The simulations are run  with $20$\% user participation ($p=0.2$) in each round which reduces the effectiveness of the adversary since it needs to participate in a bigger pool of users in comparison to our default setting $p=0.1$.

\newpage

\subsection{Fairness full results: first experiment on CIFAR-10}
\label{sec:fair_fisrt}
\Cref{table:cifar-user-1,,table:cifar-user-2,,table:cifar-user-3,,table:cifar-user-4} report the full results of the experiment presented in \Cref{figure:cifar-user} for different batch sizes and fractions of user participation. 
\begin{table}[th]
\footnotesize	
\centering
\caption{Test accuracy of users on CIFAR-10 with local batch size $b=10$, fraction of users $p=0.1$, local learning rate $\eta=0.01$, total communication rounds $2000$.  The reported statistics are averaged across $4$ runs with different random seeds. \label{table:cifar-user-1}}
\begin{tabular}{lll|llllll} \toprule
        \multicolumn{3}{c}{Algorithm}  &  Average (\%) &  Std. (\%)  & Worst $5$\% (\%) & Best $5$\% (\%) \\\midrule
                        Name & $\upeta$ & decay     &   \multicolumn{4}{c}{}\\ \midrule
                        \fedMGDA{}  &       &  & $67.59 \pm 0.65$ & $21.03 \pm 2.40$ & $22.95 \pm 7.27$ & \B $90.50 \pm 0.87$\\
                            \fedMGDAn{}& $1.0$ & $0$ & $69.06\pm 1.08$ & $14.10\pm 1.61$ & $44.38\pm 5.90$ & $87.55\pm 0.84$  \\
                            \fedMGDAn{}& $1.0$ & $1/10$& $69.87\pm 0.87$ & $14.33\pm 0.61$ & $42.42\pm 3.61$ & $87.05\pm 0.95$  \\
                            \fedMGDAn{}& $1.5$ & $1/10$& \B $71.15\pm 0.62$ & $13.74\pm 0.49$ & $44.48\pm 1.64$ & $88.53\pm 0.85$  \\
                            \fedMGDAn{}& $1.0$ & $1/40$& $68.68 \pm 1.25$ & $17.23 \pm 1.60$ & $34.40 \pm 6.23$ & $88.07 \pm 0.04$  \\
                            \fedMGDAn{}& $1.5$ & $1/40$& $71.05 \pm 0.82$ & $13.53 \pm 0.77$ & \B $46.50 \pm 2.96$ & $88.53 \pm 0.85$ \\\midrule
                            Name   & $\upeta$ & decay     &   \multicolumn{4}{c}{}\\ \midrule
                            \fedMGDAproxn{} & $1.0$ & $0$  & $66.98 \pm 1.52$ & $15.46 \pm 3.15$ & $39.42 \pm 10.35$ & $87.60 \pm 2.18$ \\
                            \fedMGDAproxn{} & $1.0$ & $1/10$ & $70.39 \pm 0.96$ & $13.70 \pm 1.08$ & $46.43 \pm 2.17$ & $87.50 \pm 0.87$ \\
                            \fedMGDAproxn{} & $1.5$ & $1/10$ & $69.45 \pm 0.77$ & $14.98 \pm 1.61$ & $40.42 \pm 5.88$ & $87.05 \pm 1.00$ \\ 
                            \fedMGDAproxn{} & $1.0$ & $1/40$ & $69.01 \pm 0.51$ & $16.24 \pm 0.74$ & $36.92 \pm 4.12$ & $88.53 \pm 0.85$ \\
                            \fedMGDAproxn{} & $1.5$ & $1/40$ & $69.53 \pm 0.70$ & $15.90 \pm 1.79$ & $36.43 \pm 7.42$ & $87.53 \pm 2.14$ \\  \midrule
                            Name   & $\upeta$ & decay     &   \multicolumn{4}{c}{}\\ \midrule
                            \fedavg &  & & $70.11 \pm 1.27$ & $13.63 \pm 0.81$ & $45.45 \pm 2.21$ & $88.00 \pm 0.00$\\
                            \fedavgn{}& $1.0$& $0$    &  $67.69 \pm 1.15$ & $16.97 \pm 2.33$ & $37.98 \pm 6.61$ & $89.55 \pm 2.61$\\
                            \fedavgn{}& $1.0$& $1/10$ &$69.66 \pm 1.22$ & $15.11 \pm 1.14$ & $40.42 \pm 1.71$ & $88.55 \pm 0.84$\\
                            \fedavgn{}& $1.5$& $1/10$ & $70.62 \pm 0.82$ & $14.19 \pm 0.49$ & $43.48 \pm 2.17$ & $89.03 \pm 1.03$\\ 
                            \fedavgn{}& $1.0$& $1/40$ & $70.31 \pm 0.29$ & $14.97 \pm 0.96$ & $42.48 \pm 2.56$ & $88.55 \pm 2.15$\\
                            \fedavgn{}& $1.5$& $1/40$ & $70.47 \pm 0.70$ & $13.88 \pm 0.96$ & $44.95 \pm 4.07$ & $88.03 \pm 0.04$\\ \midrule
                            Name   & $\mu$ &      &   \multicolumn{4}{c}{}\\ \midrule
                            \fedprox{} & $0.01$ &    & $70.77 \pm 0.70$ & $13.12 \pm 0.47$ & $46.43 \pm 2.95$ & $88.50 \pm 0.87$ \\
                            \fedprox{} & $0.1$ &    & $70.69 \pm 0.58$ & $13.42 \pm 0.43$ & $45.42 \pm 2.14$ & $87.55 \pm 1.64$\\
                            \fedprox{} & $0.5$ &    & $68.89 \pm 0.83$ & $14.10 \pm 1.08$ & $43.95 \pm 4.52$ & $88.00 \pm 0.00$ \\ \midrule
                            Name   & $q$ & $L$     &   \multicolumn{4}{c}{}\\ \midrule
                            \qffl{} & $0.1$ & $0.1$ & $70.40 \pm 0.41$ & \B $12.43 \pm 0.24$ & $46.48 \pm 2.14$ & $87.50 \pm 0.87$\\
                            \qffl{} & $0.5$ & $0.1$ & $70.58 \pm 0.73$ & $13.60 \pm 0.47$ & \B $46.50 \pm 2.96$ & $88.05 \pm 1.38$ \\
                            \qffl{} & $1.0$ & $0.1$ & $70.27 \pm 0.61$ & $13.31 \pm 0.46$ & $45.95 \pm 1.38$ & $87.55 \pm 0.90$\\
                            \qffl{} & $0.1$ & $1.0$ & $70.95 \pm 0.83$ & $12.70 \pm 0.74$ & $46.45 \pm 4.07$ & $87.00 \pm 1.00$ \\
                            \qffl{} & $0.5$ & $1.0$ & $70.98 \pm 0.52$ & $12.96 \pm 0.63$ & $45.95 \pm 1.45$ & $88.00 \pm 0.00$ \\
                            \qffl{} & $1.0$ & $1.0$ & $69.98 \pm 0.67$ & $13.15 \pm 1.12$ & $45.95 \pm 2.49$ & $87.53 \pm 0.82$\\ \bottomrule
\end{tabular}
\vspace{5pt}
\end{table}
\newpage
\begin{table}[th]
\footnotesize
\centering
\caption{Test accuracy of users on CIFAR-10 with local batch size $b=10$, fraction of users $p=0.2$, local learning rate $\eta=0.01$, total communication rounds $2000$.  The reported statistics are averaged across $4$ runs with different random seeds.\label{table:cifar-user-2}}
\begin{tabular}{lll|llllll} \toprule
        \multicolumn{3}{c}{Algorithm}  &  Average (\%) &  Std. (\%)  & Worst $5$\% (\%) & Best $5$\% (\%) \\\midrule
                        Name & $\upeta$ & decay     &   \multicolumn{4}{c}{}\\ \midrule
                        \fedMGDA{} &  &  & $66.50 \pm 1.77$ & $23.22 \pm 1.20$ & $19.48 \pm 4.54$ & \B $91.53 \pm 2.14$ \\
                        \fedMGDAn{}& $1.0$& $0$ & $66.91 \pm 1.15$ & $16.28 \pm 1.12$ & $36.00 \pm 6.16$ & $88.00 \pm 2.00$  \\
                        \fedMGDAn{}& $1.0$& $1/10$& \B $70.64 \pm 0.35$ & $16.23 \pm 0.84$ & $35.95 \pm 4.66$ & $87.55 \pm 0.90$ \\
                        \fedMGDAn{}& $1.5$& $1/10$&  $69.29 \pm 0.88$ & $13.52 \pm 0.69$ & $44.45 \pm 1.64$ & $86.55 \pm 0.84$ \\
                        \fedMGDAn{}& $1.0$ & $1/40$& $68.47 \pm 0.65$ & $18.07 \pm 2.84$ & $32.95 \pm 9.48$ & $88.55 \pm 2.15$  \\
                        \fedMGDAn{}& $1.5$ & $1/40$& $68.76 \pm 0.54$ & $17.21 \pm 1.28$ & $34.43 \pm 6.26$ & $87.53 \pm 0.88$  \\\midrule
                        Name   & $\upeta$ & decay     &   \multicolumn{4}{c}{}\\\midrule
                        \fedMGDAprox    &       &        & $70.06 \pm 0.67$ & $13.69 \pm 0.46$ & $42.43 \pm 2.99$ & $87.03 \pm 0.98$\\
                        \fedMGDAproxn{} & $1.0$ & $0$  & $68.41 \pm 0.88$ & $16.30 \pm 1.84$ & $37.98 \pm 3.15$ & $89.50 \pm 1.66$\\
                        \fedMGDAproxn{} & $1.0$ & $1/10$ & $68.19 \pm 0.96$ & $19.25 \pm 2.57$ & $28.90 \pm 7.02$ & $88.55 \pm 0.84$ \\
                        \fedMGDAproxn{} & $1.5$ & $1/10$ & $68.92 \pm 0.78$ & $14.64 \pm 0.59$ & $41.42 \pm 2.98$ & $87.55 \pm 1.61$ \\ 
                        \fedMGDAproxn{} & $1.0$ & $1/40$ & $68.87 \pm 0.60$ & $17.47 \pm 1.96$ & $31.48 \pm 7.49$ & $89.00 \pm 2.24$\\
                        \fedMGDAproxn{} & $1.5$ & $1/40$ & $69.29 \pm 0.66$ & $16.67 \pm 1.49$ & $35.48 \pm 10.31$ & $88.03 \pm 1.41$ \\ \midrule
                        Name   & $\upeta$ & decay     &   \multicolumn{4}{c}{}\\ \midrule
                        \fedavg   &      &     & $69.83 \pm 0.69$ & $13.17 \pm 0.60$ & $46.95 \pm 1.70$ & $86.57 \pm 1.64$ \\
                        \fedavgn{}& $1.0$& $0$ &  $69.05 \pm 0.94$ & $14.14 \pm 1.53$ & $39.50 \pm 6.22$ & $87.50 \pm 0.87$ \\
                        \fedavgn{}& $1.0$& $1/10$ & $70.52 \pm 1.18$ & $15.22 \pm 1.74$ & $39.48 \pm 7.90$ & $88.03 \pm 0.04$ \\
                        \fedavgn{}& $1.5$& $1/10$ & $69.27 \pm 0.97$ & $14.42 \pm 0.85$ & $43.95 \pm 3.76$ & $89.00 \pm 1.00$ \\ 
                        \fedavgn{}& $1.0$& $1/40$ & $69.34 \pm 1.75$ & $15.64 \pm 2.90$ & $38.45 \pm 8.50$ & $86.53 \pm 0.85$ \\
                        \fedavgn{}& $1.5$& $1/40$ & $69.87 \pm 0.59$ & $14.13 \pm 0.14$ & $43.95 \pm 1.42$ & $86.57 \pm 1.64$ \\ \midrule
                        Name   & $\mu$ &      &   \multicolumn{4}{c}{}\\ \midrule
                        \fedprox{} & $0.01$ &    & $69.74 \pm 0.84$ & $13.26 \pm 0.44$ & $47.90 \pm 1.45$ & $87.50 \pm 0.87$ \\
                        \fedprox{} & $0.1$ &    & $70.06 \pm 0.67$ & $13.69 \pm 0.46$ & $42.43 \pm 2.99$ & $87.03 \pm 0.98$ \\
                        \fedprox{} & $0.5$ &    & $69.64 \pm 0.74$ & $13.55 \pm 0.50$ & $44.90 \pm 1.67$ & $88.00 \pm 0.00$ \\ \midrule
                        Name   & $q$ & $L$     &   \multicolumn{4}{c}{}\\ \midrule
                        \qffl{} & $0.1$ & $0.1$ & $70.21 \pm 0.71$ & $13.23 \pm 0.42$ & $46.98 \pm 0.98$ & $87.03 \pm 0.98$\\
                        \qffl{} & $0.5$ & $0.1$ & $70.34 \pm 0.71$ & $13.05 \pm 0.27$ & $47.43 \pm 1.64$ & $88.00 \pm 0.00$  \\
                        \qffl{} & $1.0$ & $0.1$ & $70.19 \pm 0.79$ & \B $12.79 \pm 0.23$ & \B $48.42 \pm 0.91$ & $88.00 \pm 0.00$ \\
                        \qffl{} & $0.1$ & $1.0$ & $70.18 \pm 0.53$ & $13.04 \pm 0.38$ & $47.45 \pm 0.84$ & $88.05 \pm 1.38$ \\
                        \qffl{} & $0.5$ & $1.0$ & $70.30 \pm 0.70$ & $13.28 \pm 1.03$ & $45.45 \pm 1.71$ & $87.55 \pm 0.84$ \\
                        \qffl{} & $1.0$ & $1.0$ & $69.39 \pm 0.35$ & $13.75 \pm 0.34$ & $43.98 \pm 1.41$ & $87.08 \pm 0.98$\\ \bottomrule
\end{tabular}
\vspace{5pt}
\end{table}
\newpage
\begin{table}[th]
\footnotesize	
\centering
\caption{Test accuracy of users on CIFAR-10 with local batch size $b=400$, fraction of users $p=0.1$, local learning rate $\eta=0.1$, total communication rounds $3000$.  The reported statistics are averaged across $4$ runs with different random seeds.\label{table:cifar-user-3}}
\begin{tabular}{lll|llllll} \toprule
        \multicolumn{3}{c|}{Algorithm}  &  Average (\%) &  Std. (\%)  & Worst $5$\% (\%) & Best $5$\% (\%) \\\midrule
                        Name & $\upeta$ & decay     &   \multicolumn{4}{c}{}\\ \midrule
                        \fedMGDA{} &  &          & $68.58 \pm 1.99$ & $15.56 \pm 1.63$ & $38.40 \pm 6.34$ & $87.55 \pm 1.61$ \\
                        \fedMGDAn{}& $1.0$& $0$  & $37.62 \pm 30.63$ & $19.77 \pm 5.80$ & $21.48 \pm 21.58$ &\B $93.03 \pm 6.98$ \\
                        \fedMGDAn{}& $1.0$& $1/10$ & $68.82 \pm 1.56$ & $14.66 \pm 1.03$ & $41.92 \pm 3.07$ & $87.50 \pm 0.87$ \\
                        \fedMGDAn{}& $1.5$& $1/10$ & $67.21 \pm 0.89$ & $13.76 \pm 0.89$ & $43.88 \pm 2.49$ & $85.50 \pm 1.66$ \\
                        \fedMGDAn{}& $1.0$ & $1/40$& \B $70.78 \pm 0.63$ & \B $12.27 \pm 0.35$ & \B $47.00 \pm 2.24$ & $87.00 \pm 1.00$  \\
                        \fedMGDAn{}& $1.5$ & $1/40$&  $67.02 \pm 1.04$ & $13.44 \pm 0.88$ & $43.48 \pm 1.65$ & $85.50 \pm 0.87$ \\\midrule
                        Name   & $\upeta$ & decay     &   \multicolumn{4}{c}{}\\ \midrule
                        \fedMGDAproxn{} & $1.0$ &    $0$    & $52.98 \pm 27.14$ & $17.18 \pm 3.90$ & $31.85 \pm 18.46$ & $91.53 \pm 4.96$ \\
                        \fedMGDAproxn{} & $1.0$ & $1/10$ & $69.10 \pm 1.58$ & $14.14 \pm 0.73$ & $43.42 \pm 0.83$ & $87.50 \pm 0.87$ \\
                        \fedMGDAproxn{} & $1.5$ & $1/10$ & $66.66 \pm 0.17$ & $14.51 \pm 0.37$ & $38.00 \pm 1.41$ & $84.53 \pm 0.85$ \\
                        \fedMGDAproxn{} & $1.0$ & $1/40$ & $69.55 \pm 0.52$ & $13.40 \pm 1.04$ & $45.50 \pm 3.84$ & $86.55 \pm 1.61$ \\
                        \fedMGDAproxn{} & $1.5$ & $1/40$ & $67.77 \pm 0.83$ & $14.00 \pm 0.99$ & $41.93 \pm 3.78$ & $85.53 \pm 0.82$ \\ \midrule
                        Name   & $\upeta$ & decay     &   \multicolumn{4}{c}{}\\\midrule
                        \fedavg &  &                & $66.28 \pm 2.04$ & $16.92 \pm 3.78$ & $34.48 \pm 11.77$ & $87.53 \pm 0.82$\\
                        \fedavgn{}& $1.0$&    $0$    & $66.11 \pm 0.91$ & $14.95 \pm 0.67$ & $36.95 \pm 4.32$ & $86.03 \pm 1.38$\\
                        \fedavgn{}& $1.0$& $1/10$   & $67.76 \pm 0.74$ & $14.34 \pm 0.96$ & $40.50 \pm 2.18$ & $86.50 \pm 0.87$\\
                        \fedavgn{}& $1.5$& $1/10$   & $64.04 \pm 0.99$ & $14.87 \pm 1.97$ & $35.95 \pm 6.12$ & $81.55 \pm 1.61$ \\ 
                        \fedavgn{}& $1.0$& $1/40$   & $69.50 \pm 0.45$ & $13.27 \pm 0.77$ & $44.45 \pm 2.92$ & $87.03 \pm 1.69$\\
                        \fedavgn{}& $1.5$& $1/40$   & $66.54 \pm 0.97$ & $13.20 \pm 0.92$ & $42.95 \pm 2.21$ & $84.07 \pm 1.41$ \\ \midrule
                        Name   & $\mu$ &      &   \multicolumn{4}{c}{}\\ \midrule
                        \fedprox{} & $0.01$ &    & $67.43 \pm 3.57$ & $14.75 \pm 1.36$ & $40.93 \pm 7.87$ & $87.05 \pm 1.00$ \\
                        \fedprox{} & $0.1$ &    & $68.35 \pm 1.65$ & $16.49 \pm 2.83$ & $36.93 \pm 9.97$ & $87.05 \pm 1.00$\\
                        \fedprox{} & $0.5$ &    & $68.89 \pm 1.17$ & $17.46 \pm 3.63$ & $29.90 \pm 12.64$ & $87.05 \pm 1.00$ \\ \midrule
                        Name   & $q$ & $L$      &   \multicolumn{4}{c}{}\\ \midrule
                        \qffl{} & $0.1$ & $0.1$  & $70.53 \pm 0.73$ & $13.34 \pm 0.39$ & $44.95 \pm 4.15$ & $89.00 \pm 1.73$ \\
                        \qffl{} & $0.5$ & $0.1$  & $67.78 \pm 1.81$ & $17.56 \pm 1.99$ & $31.48 \pm 10.12$ & $90.57 \pm 2.21$ \\
                        \qffl{} & $1.0$ & $0.1$  & $66.86 \pm 3.02$ & $18.56 \pm 3.80$ & $28.93 \pm 11.67$ & $87.07 \pm 1.69$ \\
                        \qffl{} & $0.1$ & $1.0$  & $64.73 \pm 7.39$ & $16.01 \pm 2.83$ & $33.45 \pm 10.46$ & $84.05 \pm 4.71$ \\
                        \qffl{} & $0.5$ & $1.0$  & $68.47 \pm 1.74$ & $15.33 \pm 1.24$ & $37.92 \pm 4.51$ & $87.05 \pm 1.00$ \\
                        \qffl{} & $1.0$ & $1.0$  & $69.60 \pm 0.98$ & $14.19 \pm 0.15$ & $41.95 \pm 5.14$ & $88.55 \pm 2.95$ \\ \bottomrule
\end{tabular}
\vspace{5pt}
\end{table}
\newpage
\begin{table}[ht]
\footnotesize
\centering
\caption{Test accuracy of users on CIFAR-10 with local batch size $b=400$, fraction of users $p=0.2$, local learning rate $\eta=0.1$, total communication rounds $3000$.  The reported statistics are averaged across $4$ runs with different random seeds.\label{table:cifar-user-4}}
\begin{tabular}{lll|llllll} \toprule
        \multicolumn{3}{c|}{Algorithm}  &  Average (\%) &  Std. (\%)  & Worst $5$\% (\%) & Best $5$\% (\%) \\\midrule
         Name            & $\upeta$  & decay     &   \multicolumn{4}{c}{}\\ \midrule
         \fedMGDA{}      &          &           & $65.18 \pm 5.41$ & $16.52 \pm 4.70$ & $33.95 \pm 17.40$ & $84.55 \pm 1.67$ \\
        \fedMGDAn{}      & $1.0$    &  $0$      & $6.50 \pm 1.66$ & $24.40 \pm 3.10$ & $0.00 \pm 0.00$ & $75.00 \pm 43.30$ \\
        \fedMGDAn{}      & $1.0$    & $1/10$    & $70.14 \pm 0.72$ & $13.53 \pm 1.28$ & $44.93 \pm 2.23$ & $87.53 \pm 0.88$ \\
        \fedMGDAn{}      & $1.5$    & $1/10$    & $69.57 \pm 0.74$ & $13.39 \pm 1.09$ & $46.42 \pm 2.95$ & $87.03 \pm 1.03$ \\
        \fedMGDAn{}     & $1.0$     & $1/40$    & $68.09 \pm 0.43$ & $14.59 \pm 0.64$ & $41.98 \pm 2.47$ & $87.00 \pm 1.00$ \\
        \fedMGDAn{}     & $1.5$     & $1/40$    & $69.29 \pm 1.00$ & $12.95 \pm 0.44$ & $46.48 \pm 2.14$ & $86.53 \pm 0.85$ \\ \midrule
        Name            & $\upeta$   & decay     &   \multicolumn{4}{c}{}\\ \midrule
        \fedMGDAproxn{} & $1.0$     & $0$       & $6.75 \pm 1.09$ & $24.99 \pm 1.96$ & $0.00 \pm 0.00$ & $76.25 \pm 41.14$ \\
        \fedMGDAproxn{} & $1.0$     & $1/10$    & \B $70.73 \pm 0.51$ &  $12.62 \pm 0.73$ &  $48.00 \pm 2.45$ & $88.53 \pm 0.85$\\
        \fedMGDAproxn{} & $1.5$     & $1/10$    & $68.79 \pm 0.41$ & $13.99 \pm 1.03$ & $42.95 \pm 3.03$ & $87.00 \pm 1.00$ \\
        \fedMGDAproxn{} & $1.0$     & $1/40$    & $68.38 \pm 0.86$ & $13.53 \pm 0.69$ & $43.98 \pm 0.04$ & $86.05 \pm 1.42$ \\
        \fedMGDAproxn{} & $1.5$     & $1/40$    & $70.14 \pm 1.30$ &\B $11.94 \pm 0.79$ &\B $48.98 \pm 2.25$ & $87.53 \pm 0.88$ \\ \midrule
        Name            & $\upeta$   & decay     &   \multicolumn{4}{c}{}\\\midrule
        \fedavg         &           &           & $67.84 \pm 3.24$ & $13.83 \pm 0.79$ & $42.98 \pm 4.62$ & $86.00 \pm 2.45$\\
        \fedavgn{}      & $1.0$     &    $0$    & $66.11 \pm 0.72$ & $14.04 \pm 0.70$ & $40.93 \pm 2.27$ & $84.53 \pm 1.62$\\
        \fedavgn{}      & $1.0$     & $1/10$    & $69.03 \pm 1.02$ & $14.17 \pm 1.13$ & $42.45 \pm 5.14$ & $87.03 \pm 1.03$\\
        \fedavgn{}      & $1.5$     & $1/10$    & $64.08 \pm 1.24$ & $13.74 \pm 0.43$ & $40.38 \pm 2.09$ & $82.05 \pm 1.34$ \\ 
        \fedavgn{}      & $1.0$     & $1/40$    & $68.89 \pm 0.81$ & $13.43 \pm 0.41$ & $43.98 \pm 1.38$ & $85.57 \pm 0.85$\\
        \fedavgn{}      & $1.5$     & $1/40$    & $65.66 \pm 0.98$ & $14.18 \pm 0.62$ & $40.90 \pm 2.21$ & $84.00 \pm 0.00$\\ \midrule
        Name            & $\mu$     &           &   \multicolumn{4}{c}{}\\ \midrule
        \fedprox{}      & $0.01$    &           & $69.14 \pm 1.25$ & $14.25 \pm 1.13$ & $41.95 \pm 4.47$ & $87.00 \pm 1.00$ \\
        \fedprox{}      & $0.1$     &           & $70.15 \pm 0.73$ & $12.51 \pm 1.44$ & $47.90 \pm 4.69$ & $87.03 \pm 1.75$ \\
        \fedprox{}      & $0.5$     &           & $69.44 \pm 0.82$ & $13.83 \pm 1.84$ & $42.38 \pm 5.33$ & $86.53 \pm 0.85$ \\ \midrule
        Name            & $q$       & $L$       &   \multicolumn{4}{c}{}\\ \midrule
        \qffl{}         & $0.1$     & $0.1$     & $69.42 \pm 1.10$ & $14.34 \pm 1.36$ & $44.98 \pm 4.35$ & $87.55 \pm 0.90$ \\
        \qffl{}         & $0.5$     & $0.1$     & $69.73 \pm 1.69$ & $15.53 \pm 4.18$ & $36.45 \pm 16.43$ & $87.53 \pm 0.88$ \\
        \qffl{}         & $1.0$     & $0.1$     & $65.97 \pm 3.01$ & $17.29 \pm 5.55$ & $31.50 \pm 16.02$ &\B $90.00 \pm 2.45$ \\
        \qffl{}         & $0.1$     & $1.0$     & $68.20 \pm 1.64$ & $14.55 \pm 3.59$ & $40.45 \pm 12.56$ & $86.03 \pm 1.38$ \\
        \qffl{}         & $0.5$     & $1.0$     & $68.50 \pm 2.62$ & $16.98 \pm 4.20$ & $30.43 \pm 16.90$ & $88.05 \pm 1.38$ \\
        \qffl{}         & $1.0$     & $1.0$     & $65.89 \pm 3.50$ & $20.33 \pm 4.07$ & $24.92 \pm 15.00$ & $89.00 \pm 1.73$ \\ \bottomrule
\end{tabular}
\vspace{5pt}
\end{table}

\newpage

\subsection{Fairness full results: first experiment on Federated  EMNIST }
\label{sec:fair_first_emnist}
\Cref{table:emnist-user-1,,table:emnist-user-2} report the full results of the experiment presented in \Cref{table:emnist-summary} for different batch sizes.
\begin{table}[ht]
\footnotesize
\centering
\caption{Test accuracy of users on federated EMNIST with local batch size $b=20$, $10$ users per rounds, local learning rate $\eta=0.1$, total communication rounds $1500$.  The reported statistics are averaged across $4$ runs with different random seeds.\label{table:emnist-user-1}}
\begin{tabular}{lll|llllll} \toprule
        \multicolumn{3}{c|}{Algorithm}  &  Average (\%) &  Std. (\%)  & Worst $5$\% (\%) & Best $5$\% (\%) \\\midrule
         Name            & $\upeta$  & decay     &   \multicolumn{4}{c}{}\\ \midrule
         \fedMGDA{}      &          &          & $85.86 \pm 1.12$ & $14.06 \pm 0.40$ & $59.07 \pm 0.49$ & $100.00 \pm 0.00$\\
        \fedMGDAn{}      & $1.0$    & $0$      & $86.66 \pm 0.53$ & $13.59 \pm 0.18$ & $60.18 \pm 0.56$ & $100.00 \pm 0.00$\\
        \fedMGDAn{}      & $1.0$    & $1/2$    & $85.87 \pm 0.50$ & $14.04 \pm 0.25$ & $58.80 \pm 0.16$ & $100.00 \pm 0.00$ \\
        \fedMGDAn{}      & $2.0$    & $1/2$    & $87.22 \pm 0.67$ & $13.31 \pm 0.17$ & $61.12 \pm 0.12$ & $100.00 \pm 0.00$ \\
        \fedMGDAn{}      & $1.0$    & $1/2$    & $84.95 \pm 0.25$ & $14.64 \pm 0.15$ & $57.41 \pm 0.84$ & $100.00 \pm 0.00$ \\
        \fedMGDAn{}      & $2.0$    & $1/5$    & $86.39 \pm 0.52$ & $13.70 \pm 0.23$ & $60.21 \pm 0.88$ & $100.00 \pm 0.00$  \\\midrule
        Name            & $\upeta$   & decay     &   \multicolumn{4}{c}{}\\ \midrule
        \fedMGDAproxn{} & $1.0$     & $0$      & $86.67 \pm 0.73$ & $13.59 \pm 0.22$ & $60.34 \pm 0.75$ & $100.00 \pm 0.00$  \\
        \fedMGDAproxn{} & $1.0$     & $1/2$    & $85.72 \pm 0.63$ & $14.07 \pm 0.29$ & $58.66 \pm 0.65$ & $100.00 \pm 0.00$\\
        \fedMGDAproxn{} & $2.0$     & $1/2$    & $87.13 \pm 0.79$ & $13.29 \pm 0.21$ & $60.84 \pm 0.74$ & $100.00 \pm 0.00$ \\
        \fedMGDAproxn{} & $1.0$     & $1/5$    & $84.83 \pm 0.40$ & $14.63 \pm 0.22$ & $57.52 \pm 1.21$ & $100.00 \pm 0.00$\\
        \fedMGDAproxn{} & $2.0$     & $1/5$    & $86.31 \pm 0.59$ & $13.68 \pm 0.26$ & $59.99 \pm 0.90$ & $100.00 \pm 0.00$ \\ \midrule
        Name            & $\upeta$   & decay     &   \multicolumn{4}{c}{}\\\midrule
        \fedavg         &           &          & $87.50 \pm 0.30$ & $13.47 \pm 0.37$ & $60.46 \pm 1.23$ & $100.00 \pm 0.00$\\
        \fedavgn{}      & $1.0$     & $0$      & $86.70 \pm 0.63$ & $13.65 \pm 0.18$ & $60.51 \pm 0.98$ & $100.00 \pm 0.00$\\
        \fedavgn{}      & $1.0$     & $1/2$    & $85.95 \pm 0.59$ & $14.02 \pm 0.25$ & $58.79 \pm 0.44$ & $100.00 \pm 0.00$\\
        \fedavgn{}      & $2.0$     & $1/2$    & $87.32 \pm 0.59$ &\B $13.29 \pm 0.12$ & \B $61.13 \pm 0.31$ & $100.00 \pm 0.00$ \\ 
        \fedavgn{}      & $1.0$     & $1/5$    & $84.95 \pm 0.28$ & $14.64 \pm 0.13$ & $57.19 \pm 0.47$ & $100.00 \pm 0.00$\\
        \fedavgn{}      & $2.0$     & $1/5$    & $86.53 \pm 0.46$ & $13.63 \pm 0.19$ & $60.18 \pm 0.56$ & $100.00 \pm 0.00$ \\ \midrule
        Name            & $\mu$     &           &   \multicolumn{4}{c}{}\\ \midrule
        \fedprox{}      & $0.01$    &           &  $87.55 \pm 0.28$ & $13.46 \pm 0.41$ & $60.44 \pm 1.48$ & $100.00 \pm 0.00$ \\
        \fedprox{}      & $0.1$     &           &  $87.37 \pm 0.48$ & $13.44 \pm 0.27$ & $60.55 \pm 1.00$ & $100.00 \pm 0.00$ \\
        \fedprox{}      & $0.5$     &           &  $86.85 \pm 0.51$ & $13.64 \pm 0.19$ & $60.03 \pm 0.99$ & $100.00 \pm 0.00$ \\ \midrule
        Name            & $q$       & $L$       &   \multicolumn{4}{c}{}\\ \midrule
        \qffl{}         & $0.1$     & $0.1$     & $87.65 \pm 0.43$ & $13.45 \pm 0.41$ & $60.37 \pm 1.43$ & $100.00 \pm 0.00$ \\
        \qffl{}         & $0.5$     & $0.1$     & $87.65 \pm 0.15$ & $13.54 \pm 0.38$ & $60.25 \pm 1.50$ & $100.00 \pm 0.00$ \\
        \qffl{}         & $1.0$     & $0.1$     & $87.71 \pm 0.22$ & $13.65 \pm 0.31$ & $59.69 \pm 1.63$ & $100.00 \pm 0.00$ \\
        \qffl{}         & $0.1$     & $1.0$     & $87.67 \pm 0.26$ & $13.41 \pm 0.40$ & $60.49 \pm 1.51$ & $100.00 \pm 0.00$ \\
        \qffl{}         & $0.5$     & $1.0$     & \B $87.72 \pm 0.38$ & $13.50 \pm 0.27$ & $60.58 \pm 1.02$ & $100.00 \pm 0.00$ \\
        \qffl{}         & $1.0$     & $1.0$     & $87.70 \pm 0.42$ & $13.59 \pm 0.07$ & $60.01 \pm 0.61$ & $100.00 \pm 0.00$ \\ \bottomrule
\end{tabular}
\vspace{5pt}
\end{table}
\newpage
\begin{table}[ht]
\footnotesize
\centering
\caption{Test accuracy of users on federated EMNIST with full batch, $10$ users per rounds, local learning rate $\eta=0.1$, total communication rounds $1500$.  The reported statistics are averaged across $4$ runs with different random seeds.\label{table:emnist-user-2}}
\begin{tabular}{lll|llllll} \toprule
        \multicolumn{3}{c|}{Algorithm}  &  Average (\%) &  Std. (\%)  & Worst $5$\% (\%) & Best $5$\% (\%) \\\midrule
         Name            & $\upeta$  & decay     &   \multicolumn{4}{c}{}\\ \midrule
        \fedMGDA{}       &          &            & $85.73 \pm 0.05$ & $14.79 \pm 0.12$ & $55.64 \pm 0.15$ & $100.00 \pm 0.00$ \\
        \fedMGDAn{}      & $1.0$    & $0$        & $86.61 \pm 0.49$ & $14.05 \pm 0.20$ & $58.30 \pm 1.66$ & $100.00 \pm 0.00$ \\
        \fedMGDAn{}      & $1.0$    & $1/2$      & $86.93 \pm 0.07$ & $14.07 \pm 0.20$ & $58.21 \pm 0.84$ & $100.00 \pm 0.00$ \\ 
        \fedMGDAn{}      & $2.0$    & $1/2$      & $87.42 \pm 0.10$ & \B $13.67 \pm 0.18$ & \B $59.94 \pm 0.73$ & $100.00 \pm 0.00$ \\ 
        \fedMGDAn{}      & $1.0$    & $1/5$      & $86.59 \pm 0.24$ & $14.32 \pm 0.22$ & $57.70 \pm 0.61$ & $100.00 \pm 0.00$ \\ 
        \fedMGDAn{}      & $2.0$    & $1/5$      &  \B $87.60 \pm 0.20$ & $13.68 \pm 0.19$ & $59.88 \pm 0.83$ & $100.00 \pm 0.00$ \\ \midrule
        Name            & $\upeta$   & decay     &   \multicolumn{4}{c}{}\\ \midrule
        \fedMGDAproxn{} & $1.0$     & $0$      & $86.63 \pm 0.45$ & $14.04 \pm 0.20$ & $58.26 \pm 1.65$ & $100.00 \pm 0.00$ \\
        \fedMGDAproxn{} & $1.0$     & $1/2$    & $86.94 \pm 0.10$ & $14.03 \pm 0.13$ & $58.40 \pm 0.73$ & $100.00 \pm 0.00$\\
        \fedMGDAproxn{} & $2.0$     & $1/2$    & $87.47 \pm 0.13$ & $13.69 \pm 0.22$ & $59.12 \pm 0.51$ & $100.00 \pm 0.00$ \\
        \fedMGDAproxn{} & $1.0$     & $1/5$    & $86.63 \pm 0.21$ & $14.31 \pm 0.16$ & $57.78 \pm 0.12$ & $100.00 \pm 0.00$\\
        \fedMGDAproxn{} & $2.0$     & $1/5$    & $87.59 \pm 0.19$ & $13.75 \pm 0.18$ & $59.65 \pm 0.82$ & $100.00 \pm 0.00$ \\ \midrule
        Name            & $\upeta$   & decay     &   \multicolumn{4}{c}{}\\\midrule
        \fedavg         &           &            & $84.97 \pm 0.44$ & $15.25 \pm 0.36$ & $54.74 \pm 1.05$ & $100.00 \pm 0.00$ \\
        \fedavgn{}      & $1.0$     & $0$        & $86.81 \pm 0.42$ & $14.05 \pm 0.34$ & $58.21 \pm 1.56$ & $100.00 \pm 0.00$\\
        \fedavgn{}      & $1.0$     & $1/2$      & $86.86 \pm 0.09$ & $14.01 \pm 0.20$ & $58.64 \pm 0.57$ & $100.00 \pm 0.00$ \\
        \fedavgn{}      & $2.0$     & $1/2$      & $87.54 \pm 0.20$ & $13.69 \pm 0.18$ & $59.41 \pm 1.05$ & $100.00 \pm 0.00$ \\
        \fedavgn{}      & $1.0$     & $1/5$      & $86.59 \pm 0.18$ & $14.31 \pm 0.15$ & $57.89 \pm 0.00$ & $100.00 \pm 0.00$ \\
        \fedavgn{}      & $2.0$     & $1/5$      &  $87.57 \pm 0.09$ & $13.74 \pm 0.11$ & $59.18 \pm 0.48$ & $100.00 \pm 0.00$ \\\midrule
        Name            & $\mu$     &           &   \multicolumn{4}{c}{}\\ \midrule
        \fedprox{}      & $0.01$    &           & $84.95 \pm 0.44$ & $15.27 \pm 0.36$ & $54.90 \pm 1.13$ & $100.00 \pm 0.00$ \\
        \fedprox{}      & $0.1$     &           & $84.95 \pm 0.42$ & $15.30 \pm 0.34$ & $54.34 \pm 0.92$ & $100.00 \pm 0.00$ \\
        \fedprox{}      & $0.5$     &           & $84.97 \pm 0.45$ & $15.26 \pm 0.35$ & $54.90 \pm 1.13$ & $100.00 \pm 0.00$ \\ \midrule
        Name            & $q$       & $L$       &   \multicolumn{4}{c}{}\\ \midrule
        \qffl{}         & $0.1$     & $0.1$     & $84.97 \pm 0.44$ & $15.25 \pm 0.37$ & $54.90 \pm 1.13$ & $100.00 \pm 0.00$\\
        \qffl{}         & $0.5$     & $0.1$     &  $84.92 \pm 0.46$ & $15.35 \pm 0.38$ & $54.15 \pm 0.91$ & $100.00 \pm 0.00$\\
        \qffl{}         & $1.0$     & $0.1$     & $84.81 \pm 0.55$ & $15.48 \pm 0.45$ & $53.72 \pm 0.74$ & $100.00 \pm 0.00$ \\
        \qffl{}         & $0.1$     & $1.0$     & $84.92 \pm 0.45$ & $15.27 \pm 0.37$ & $54.77 \pm 1.08$ & $100.00 \pm 0.00$ \\
        \qffl{}         & $0.5$     & $1.0$     & $84.76 \pm 0.55$ & $15.44 \pm 0.41$ & $54.47 \pm 0.92$ & $100.00 \pm 0.00$ \\
        \qffl{}         & $1.0$     & $1.0$     & $84.51 \pm 0.55$ & $15.65 \pm 0.45$ & $53.32 \pm 0.43$ & $100.00 \pm 0.00$ \\ \bottomrule

\end{tabular}
\vspace{5pt}
\end{table}
\newpage

\subsection{Fairness full results: second experiment }
\label{sec:fair_second}
\Cref{fig:cifar-percentage-1,,fig:cifar-percentage-2} show the full results of the experiment presented in \Cref{fig:improved_user_cifar} for different batch sizes.
\begin{figure}[hbt!]
\centering
\includegraphics[width=0.45\columnwidth]{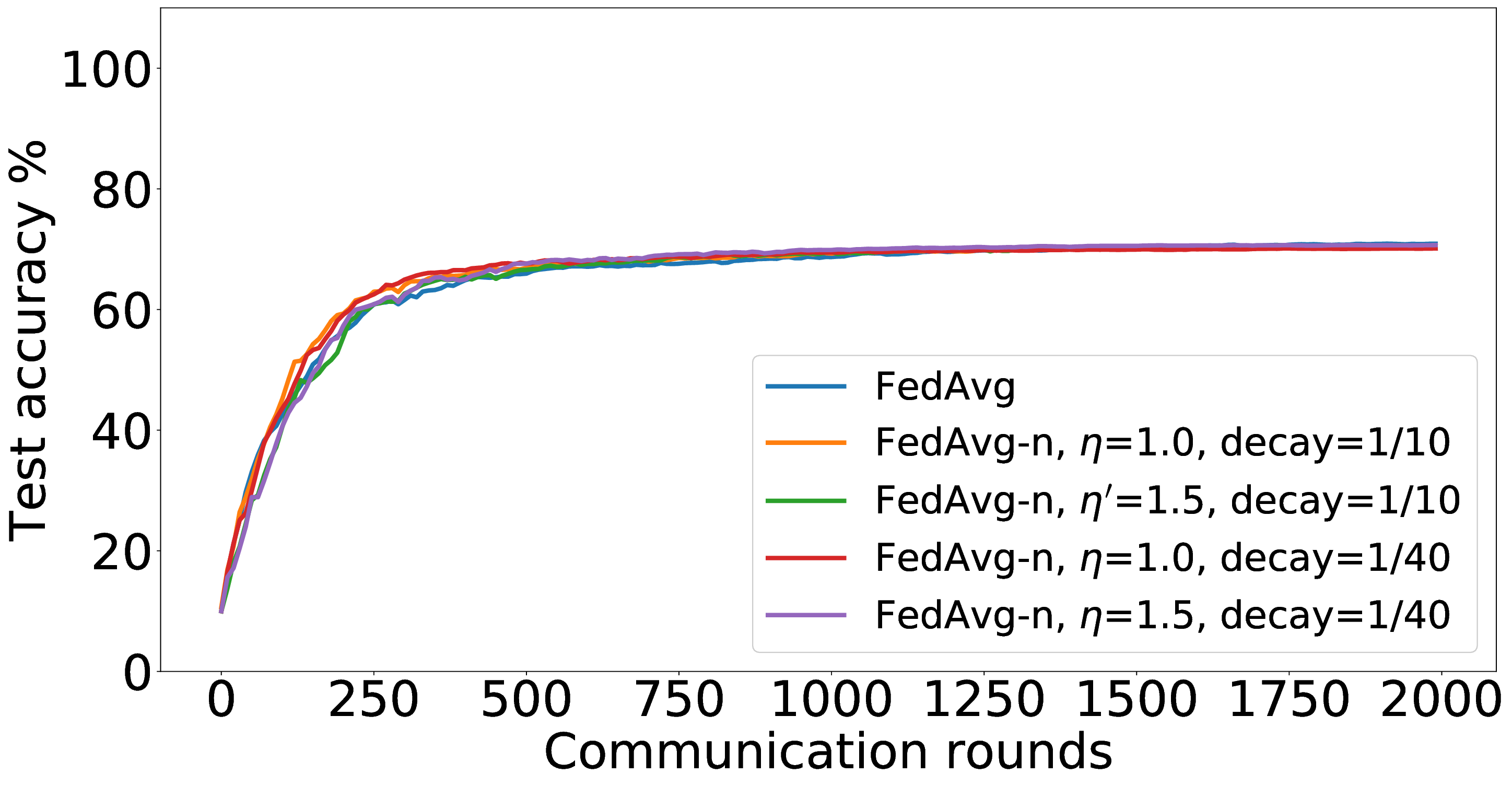}
\includegraphics[width=0.45\columnwidth]{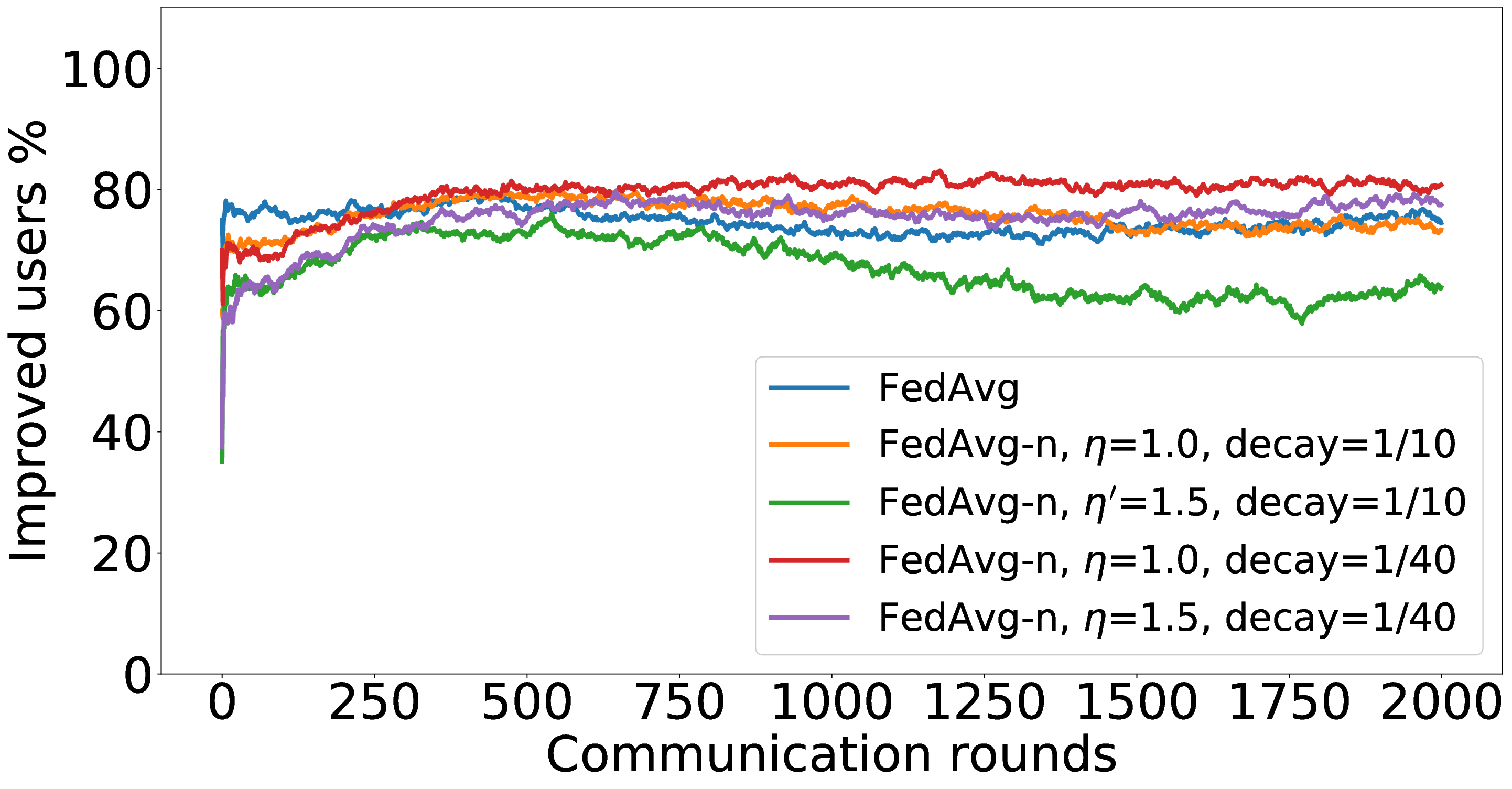}
\includegraphics[width=0.45\columnwidth]{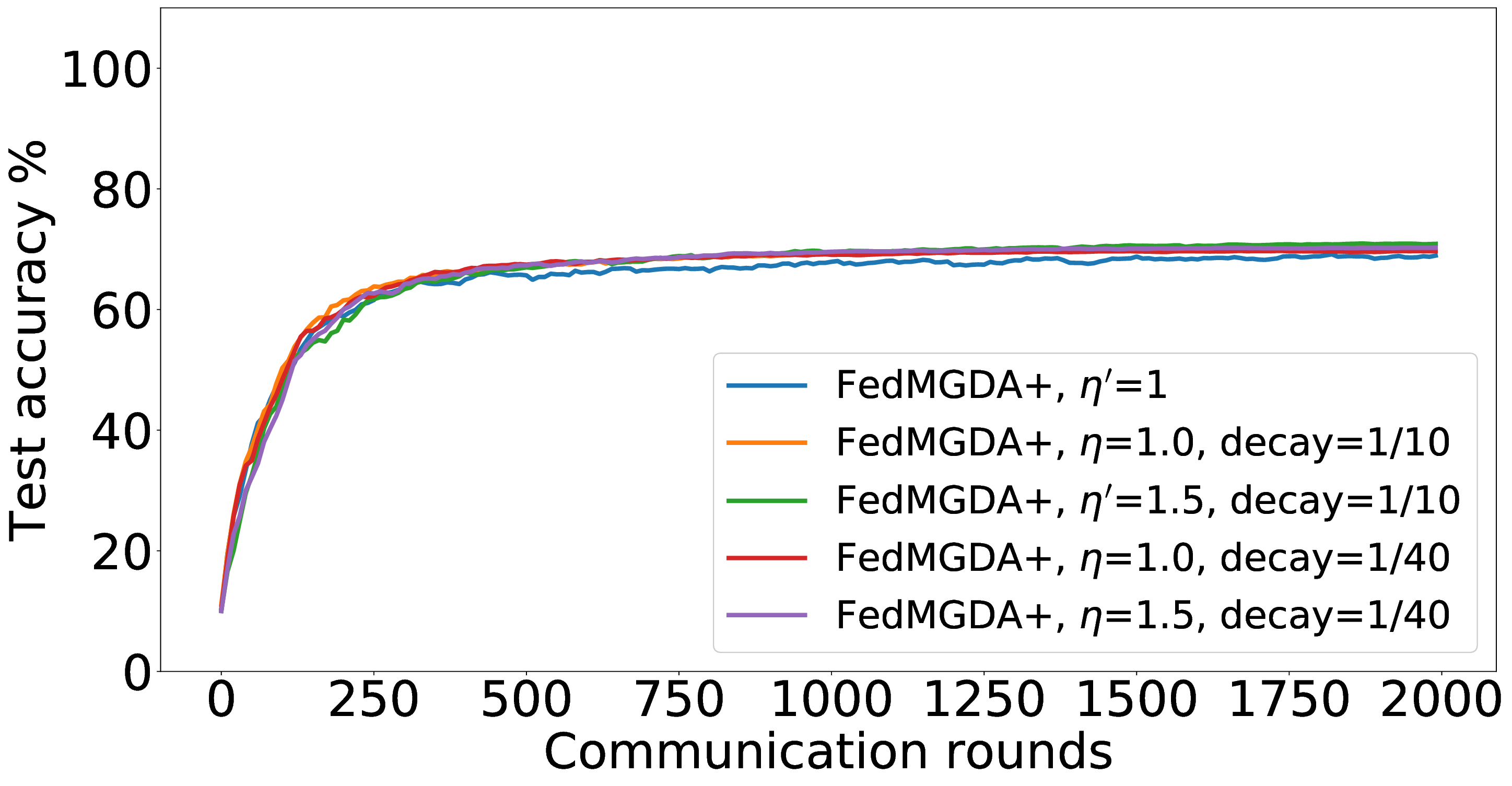}
\includegraphics[width=0.45\columnwidth]{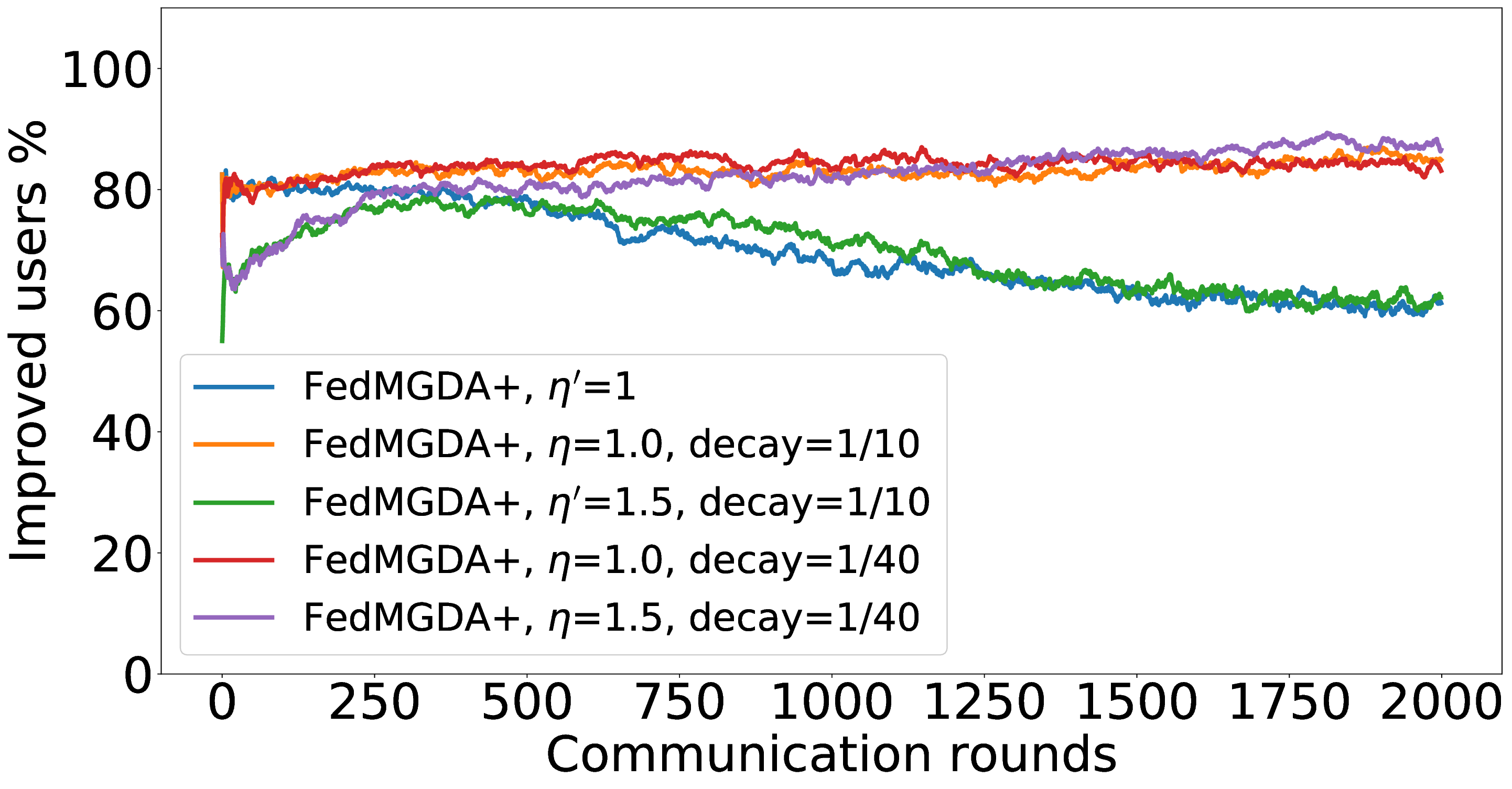}
\includegraphics[width=0.45\columnwidth]{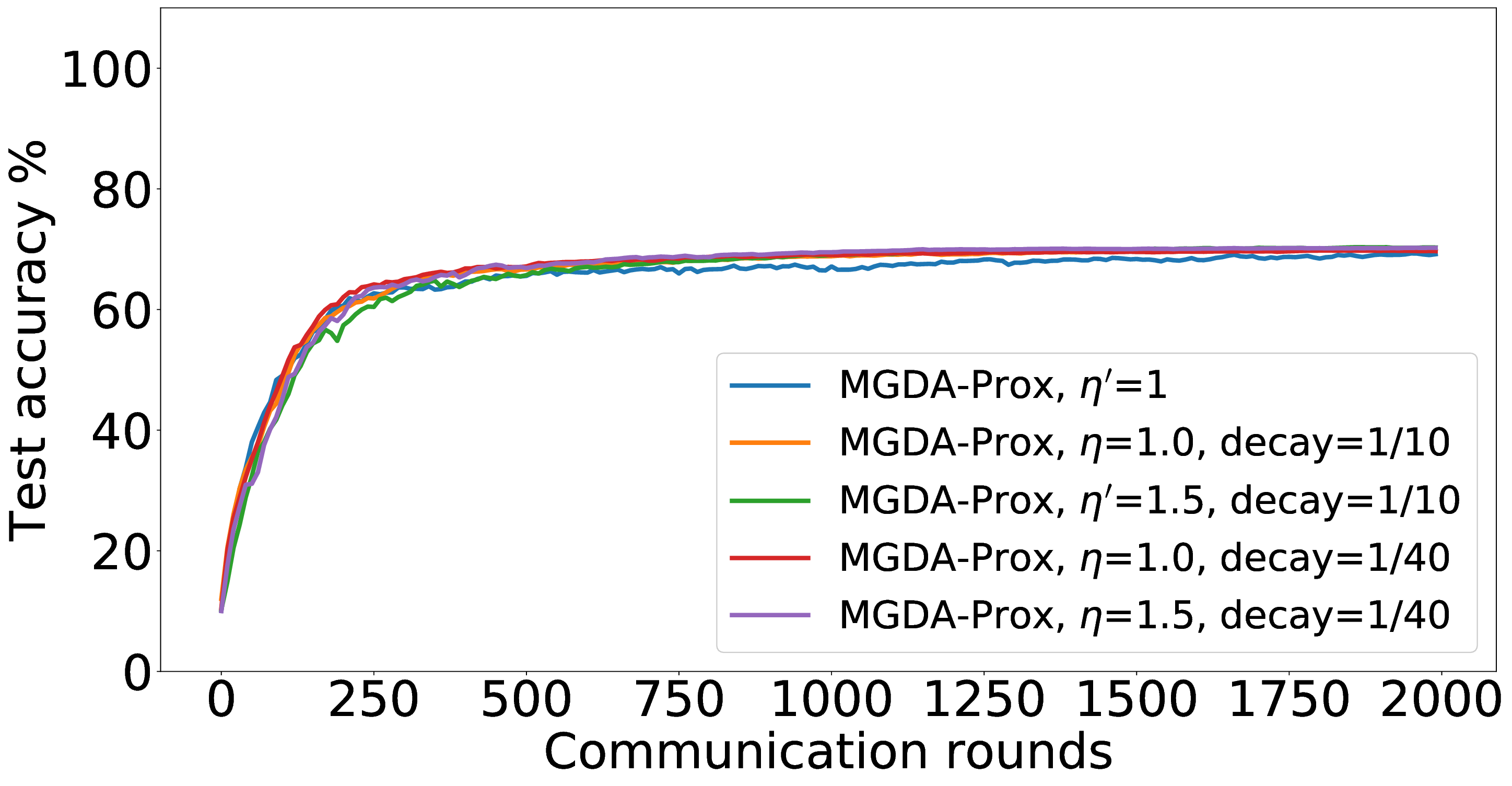}
\includegraphics[width=0.45\columnwidth]{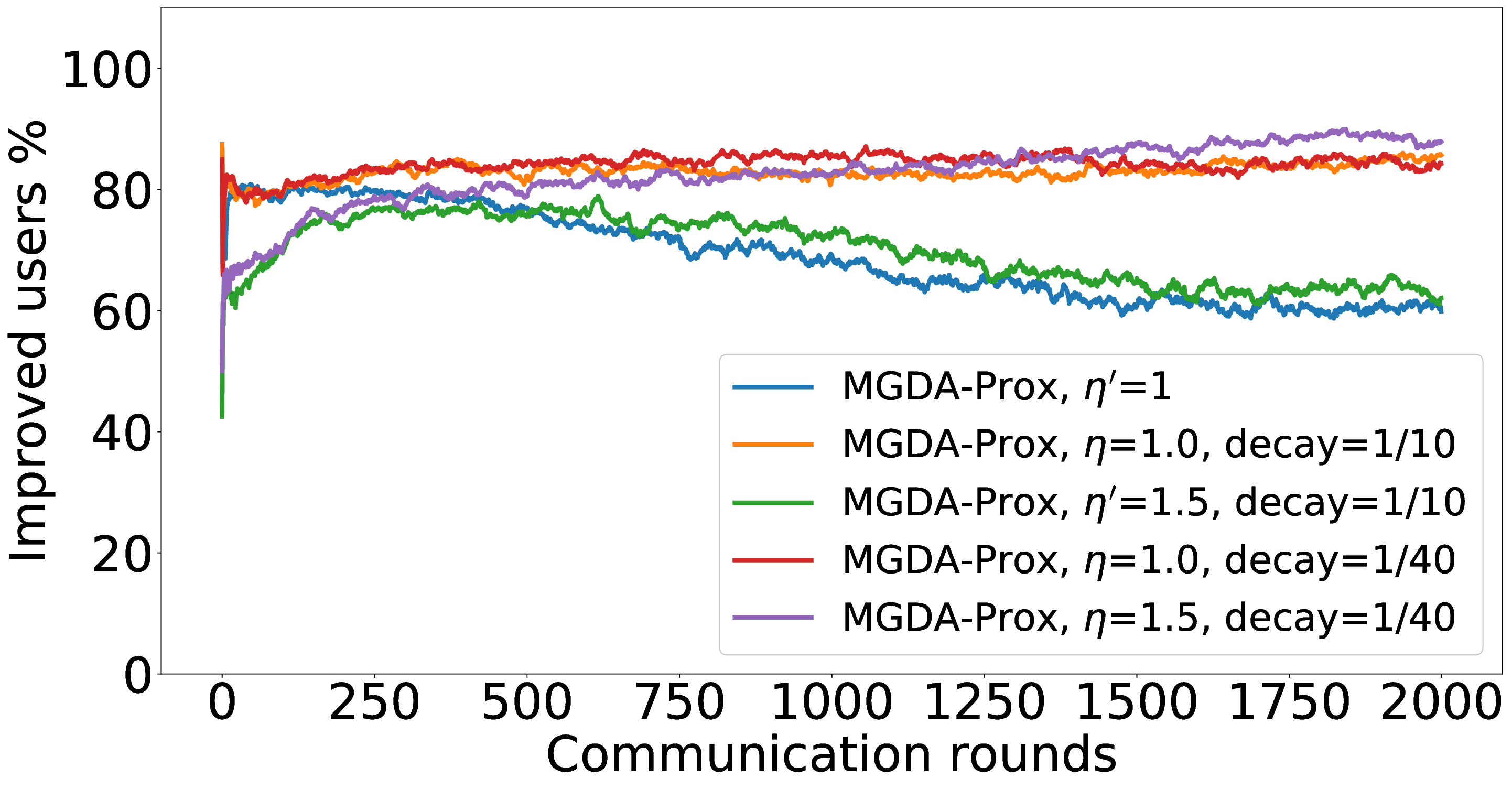}
\includegraphics[width=0.45\columnwidth]{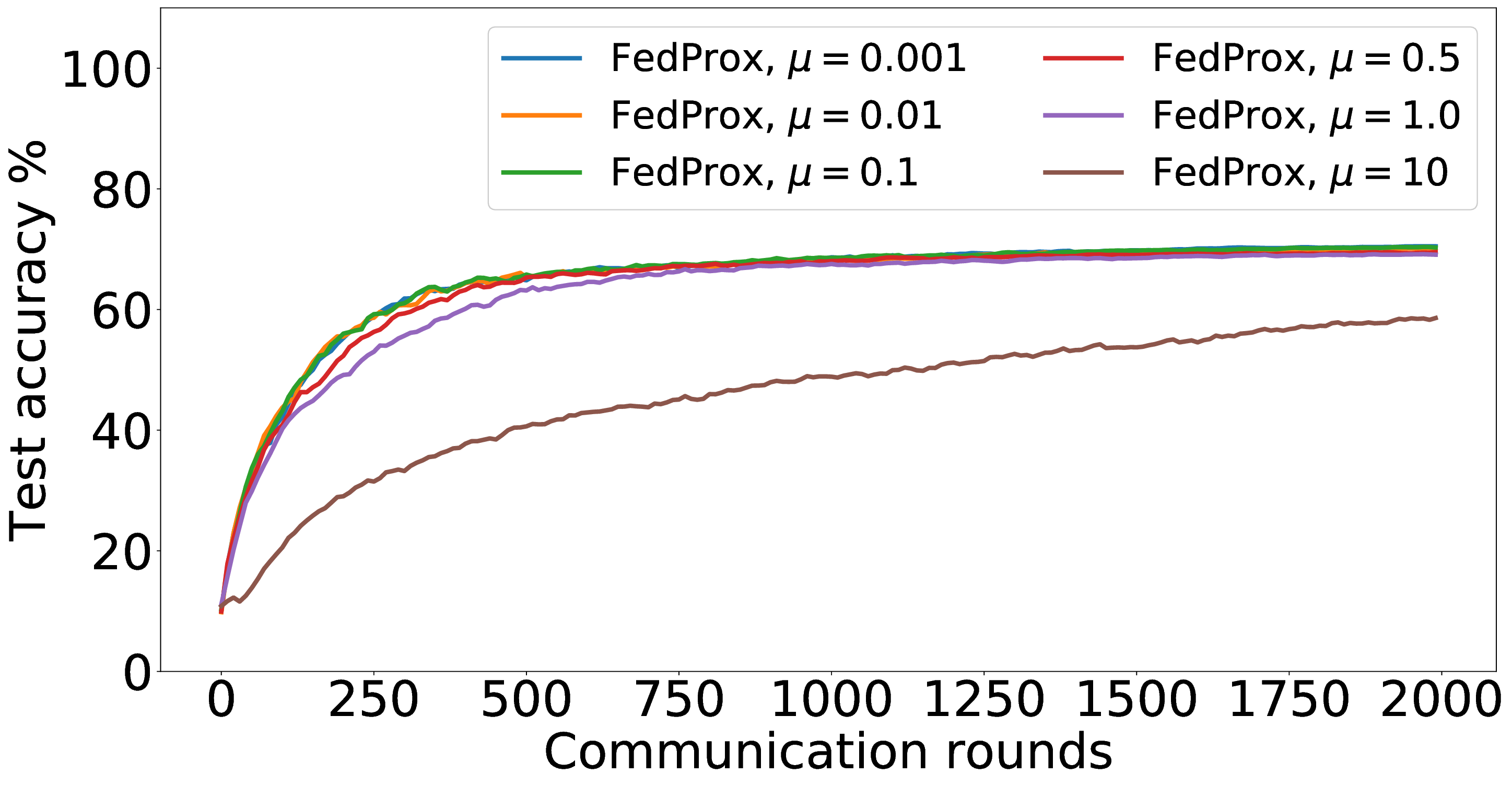}
\includegraphics[width=0.45\columnwidth]{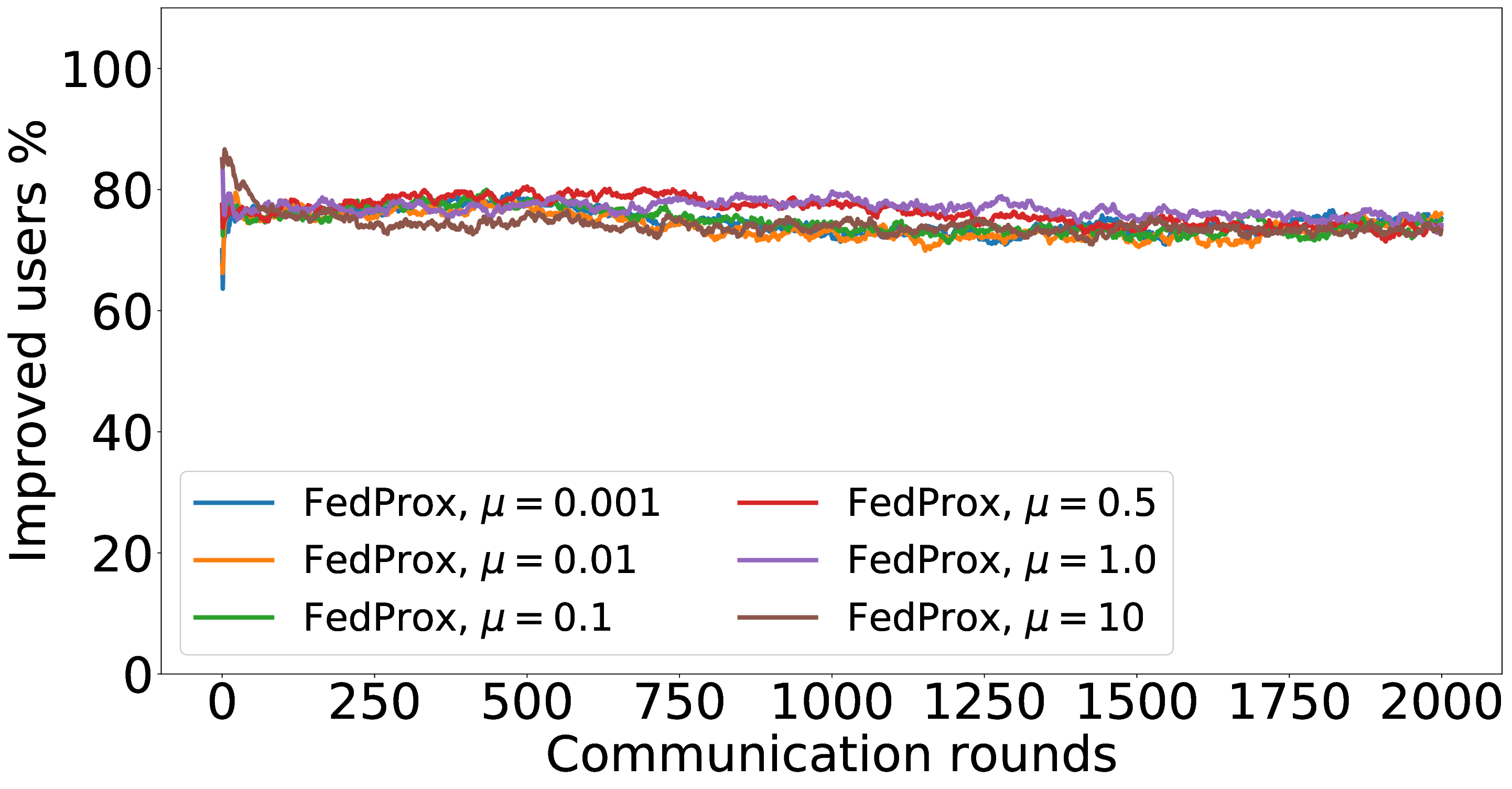}
\includegraphics[width=0.45\columnwidth]{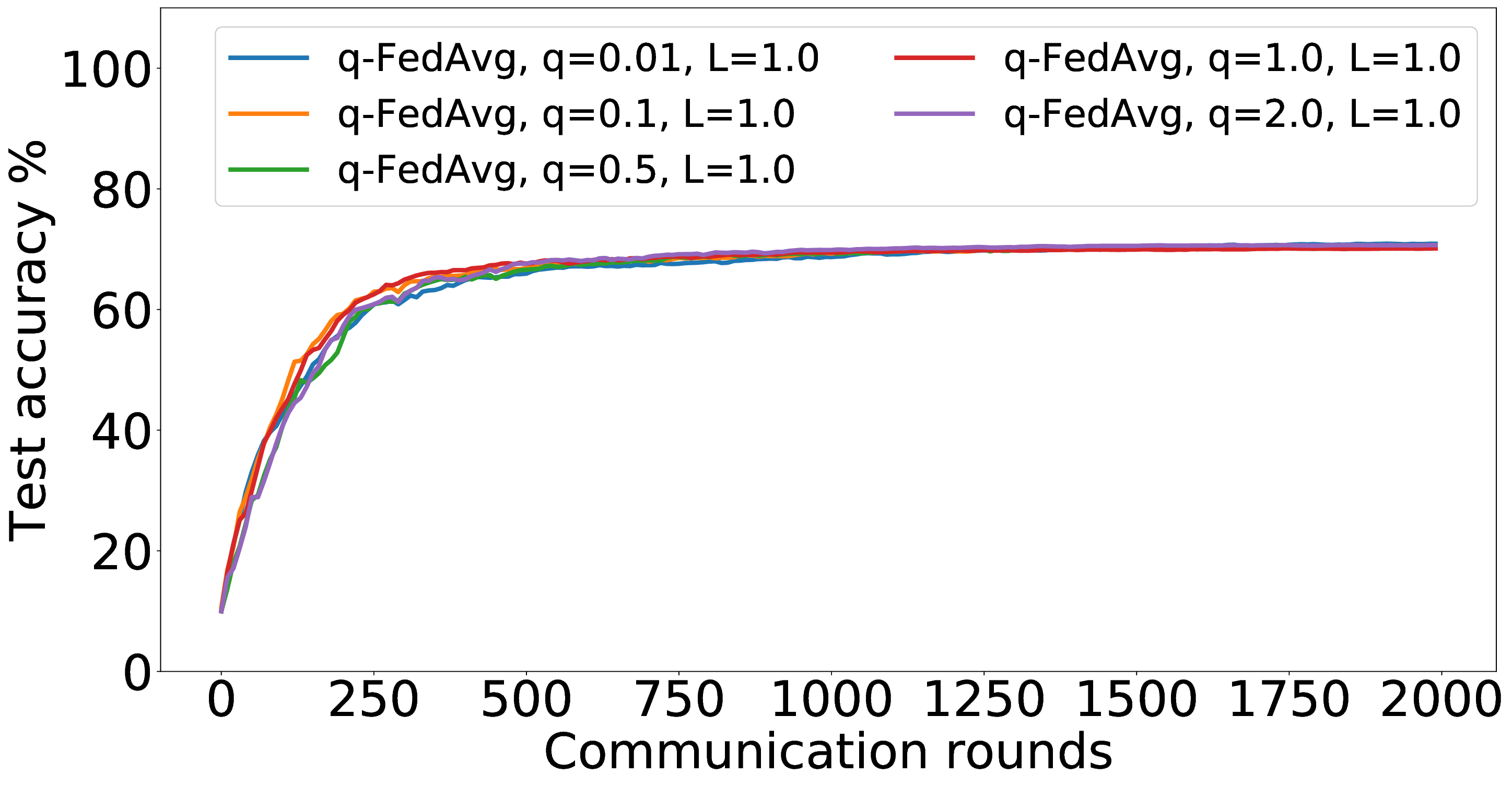}
\includegraphics[width=0.45\columnwidth]{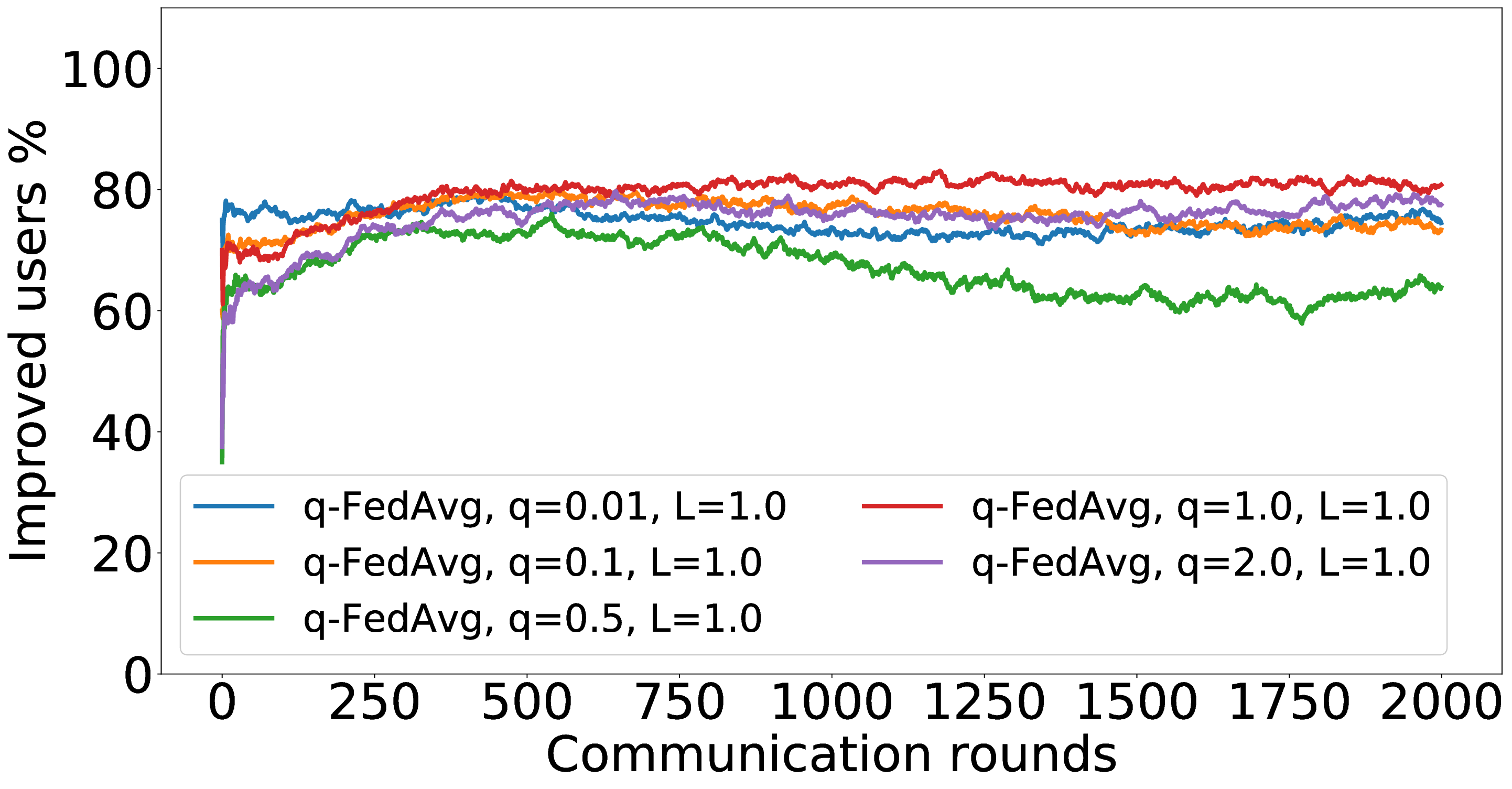}
\caption{\small The percentage of improved users in terms of training loss and the global test accuracy vs communication rounds on CIFAR-10 dataset with $p=0.1$ and $b=10$. Results are averaged across $4$ runs with different random seeds.} 
\label{fig:cifar-percentage-1}
\end{figure}
\newpage
\begin{figure}[hbt!]
\centering
\includegraphics[width=0.46\columnwidth]{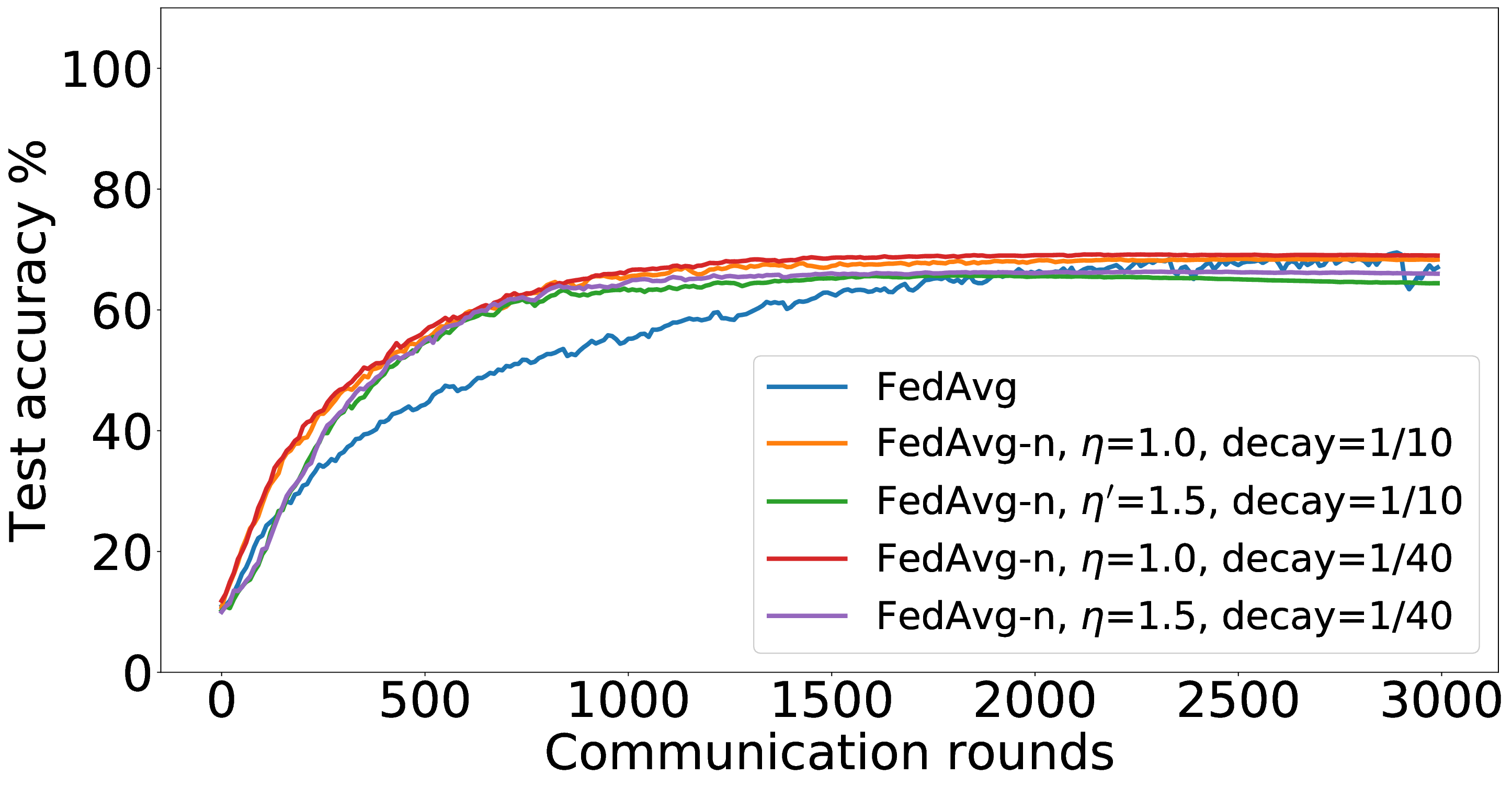}
\includegraphics[width=0.46\columnwidth]{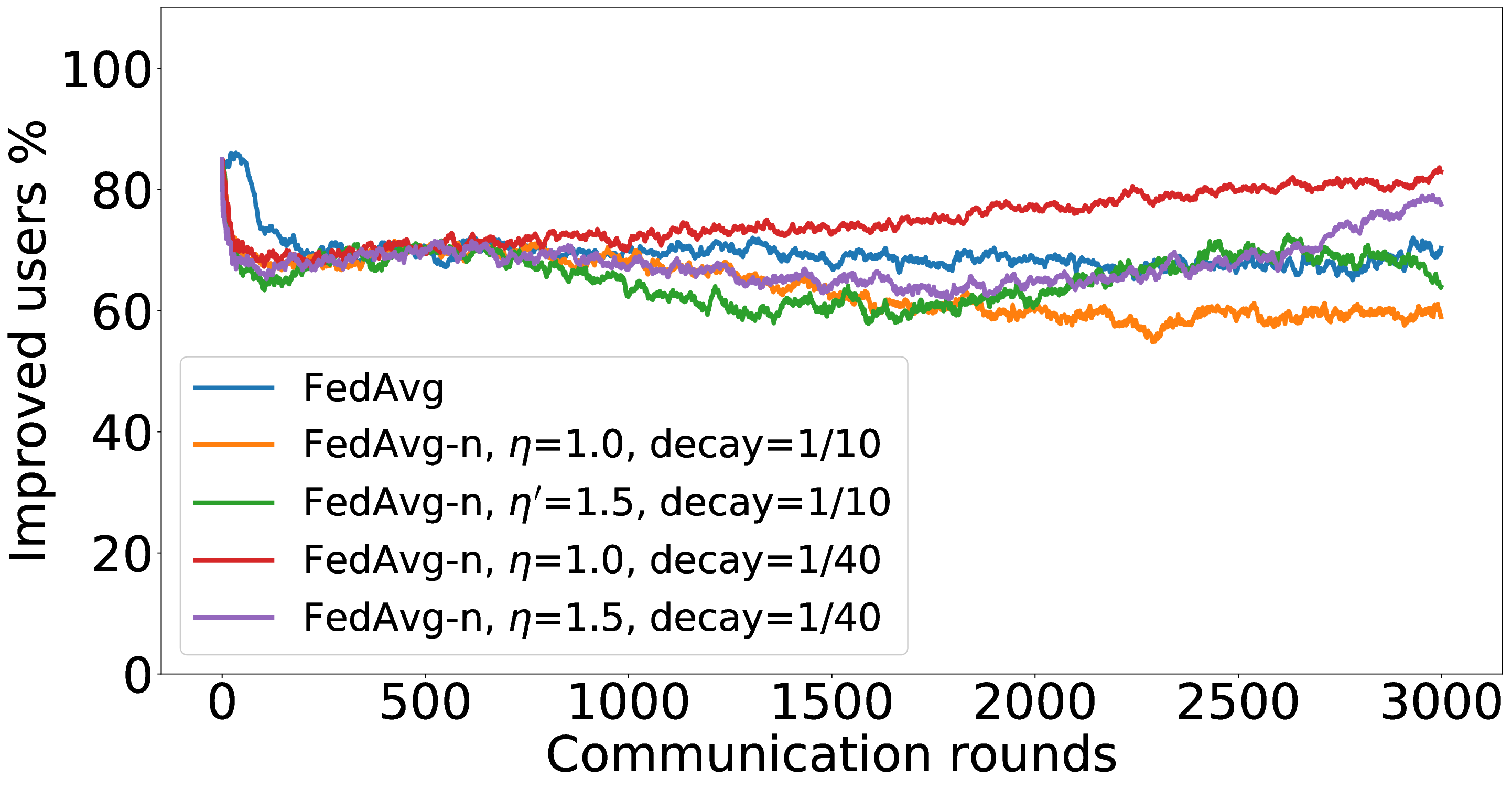}
\includegraphics[width=0.46\columnwidth]{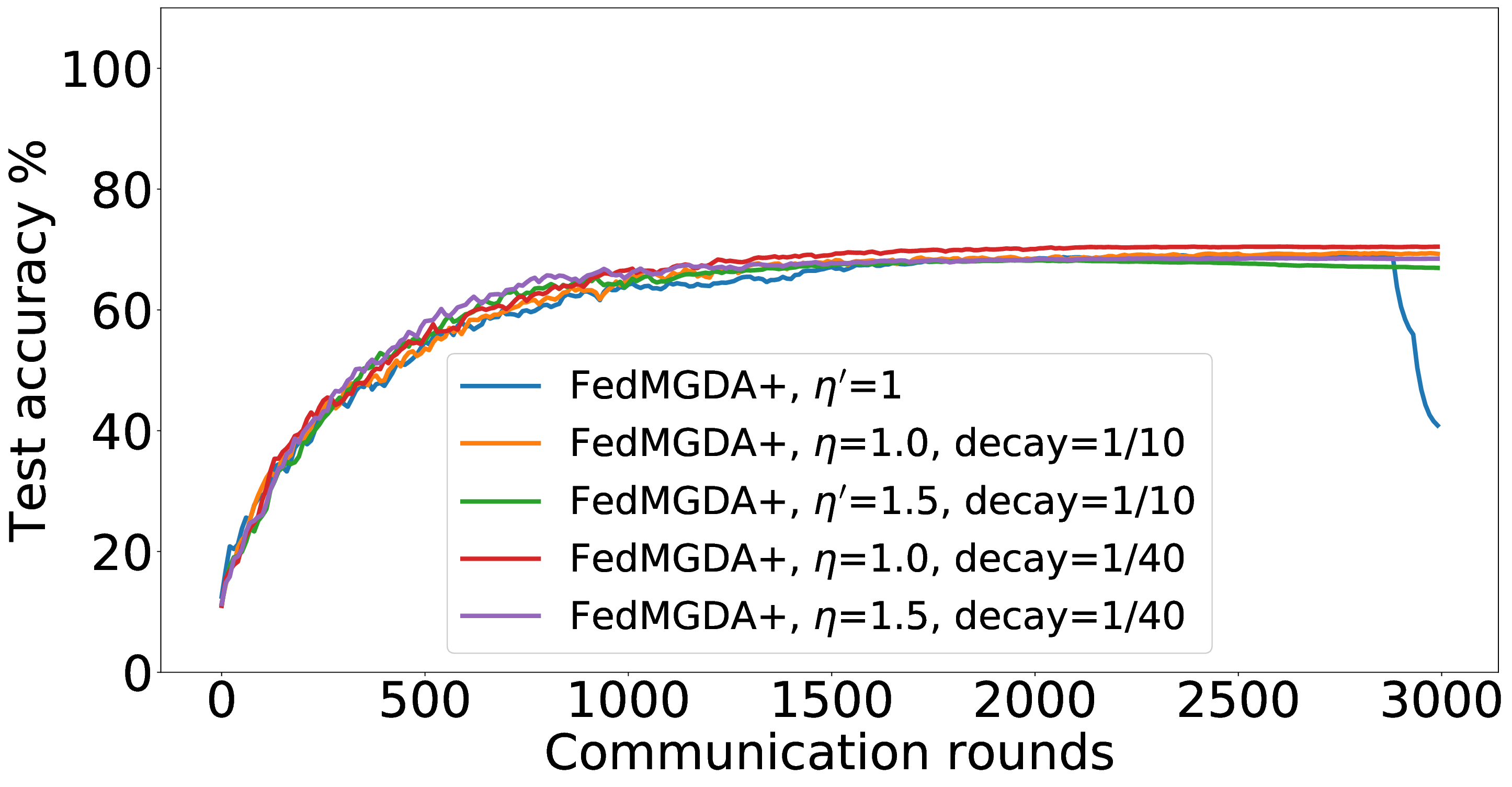}
\includegraphics[width=0.46\columnwidth]{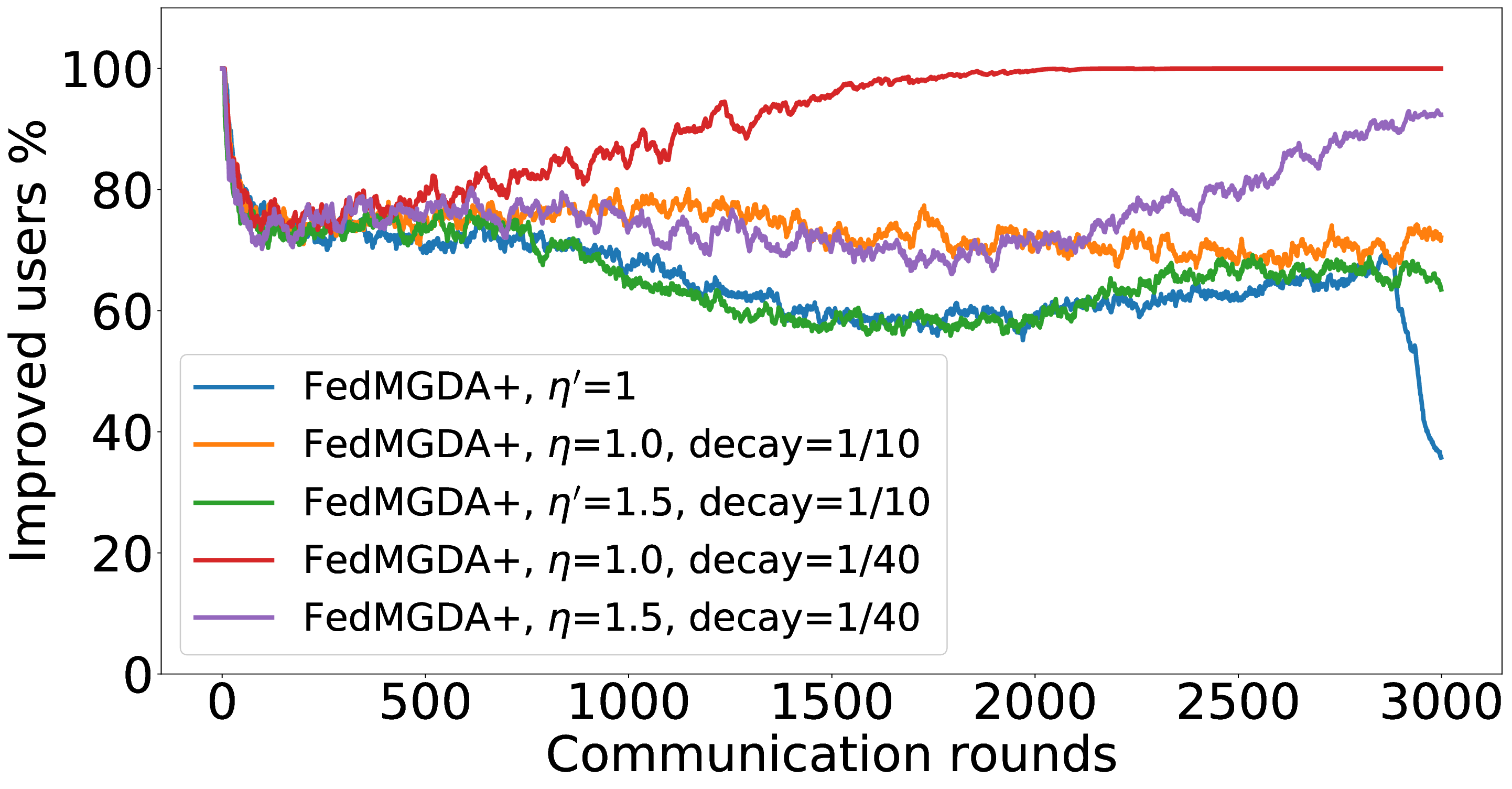}
\includegraphics[width=0.46\columnwidth]{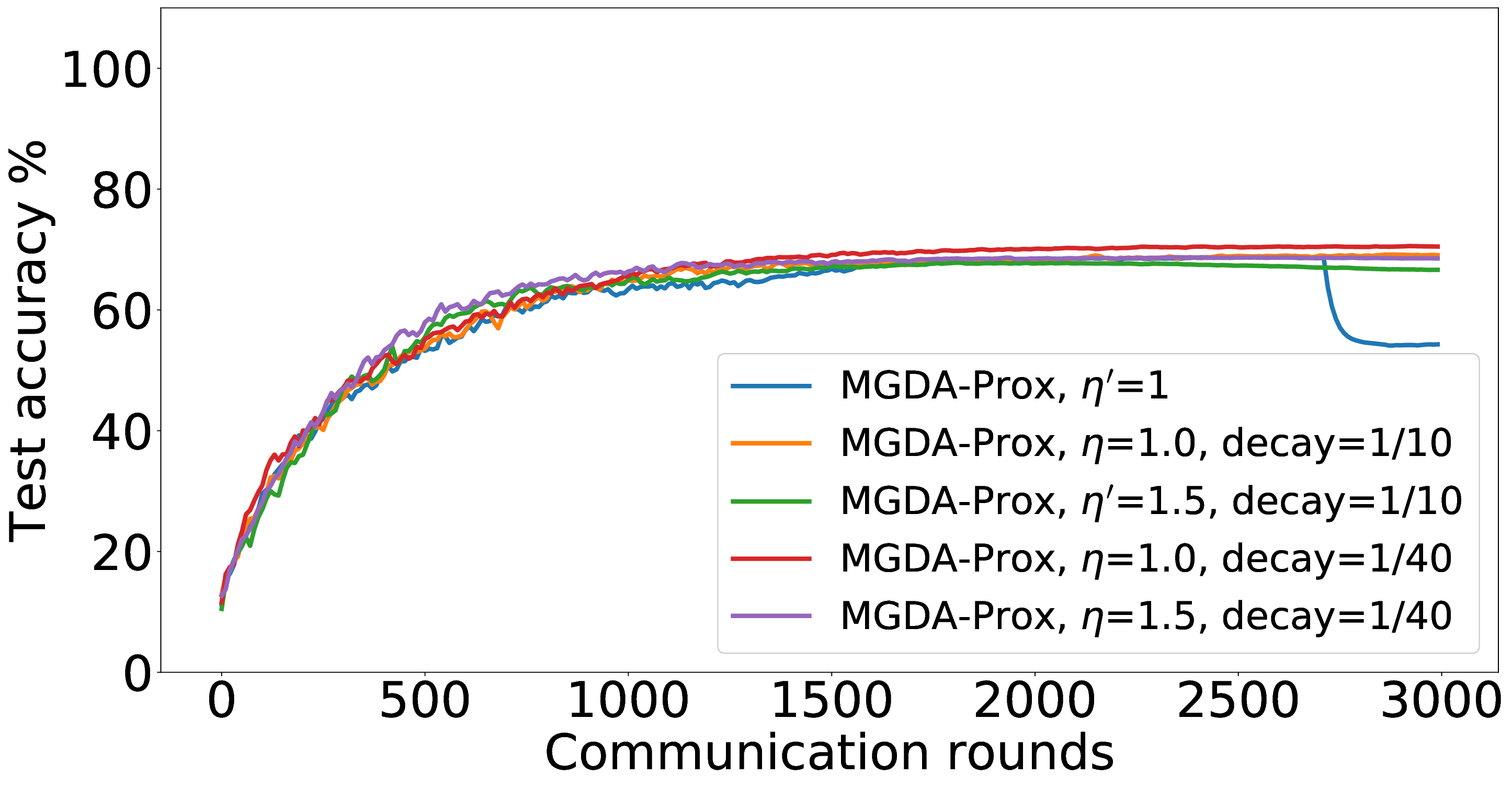}
\includegraphics[width=0.46\columnwidth]{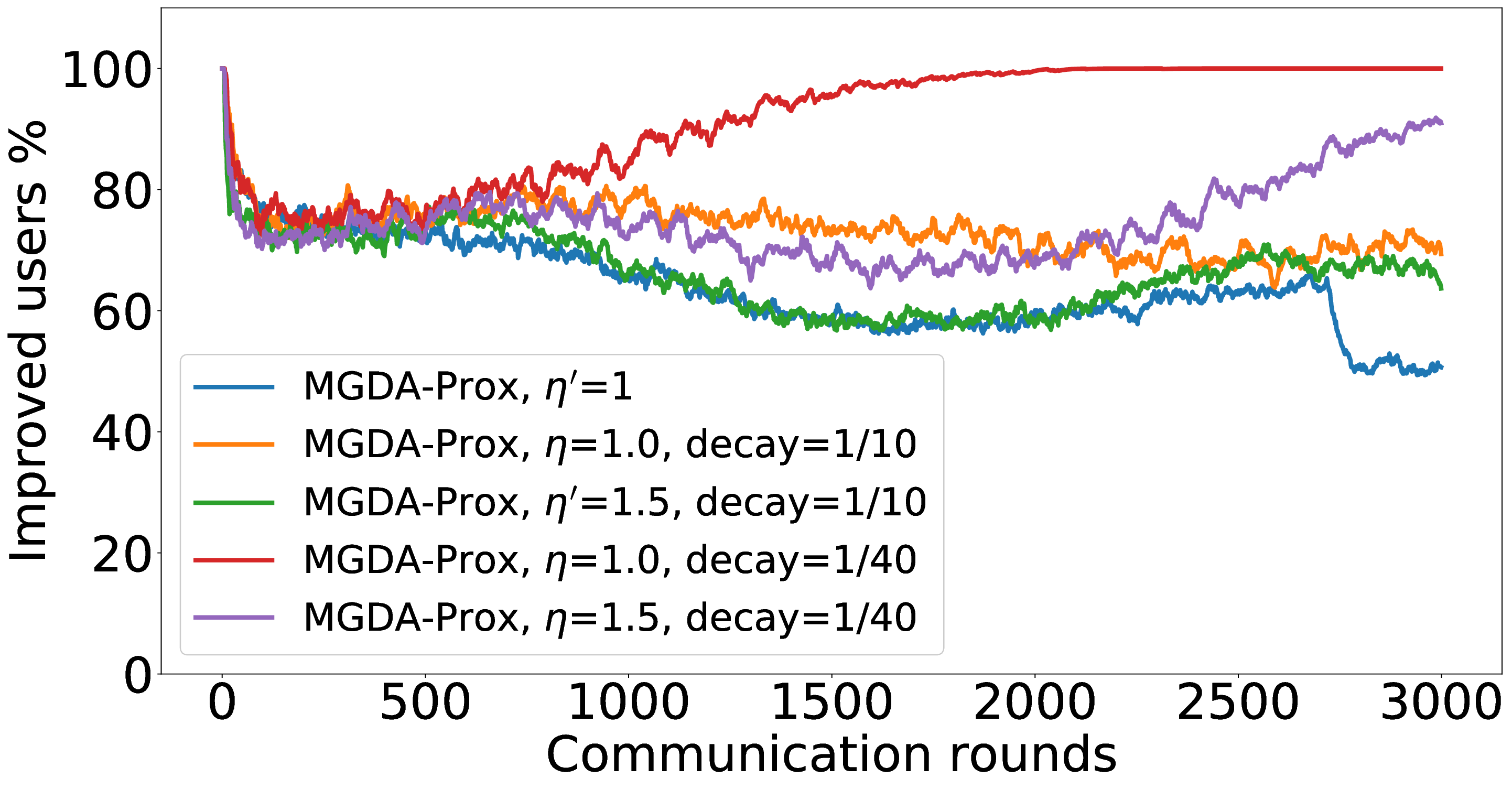}
\includegraphics[width=0.46\columnwidth]{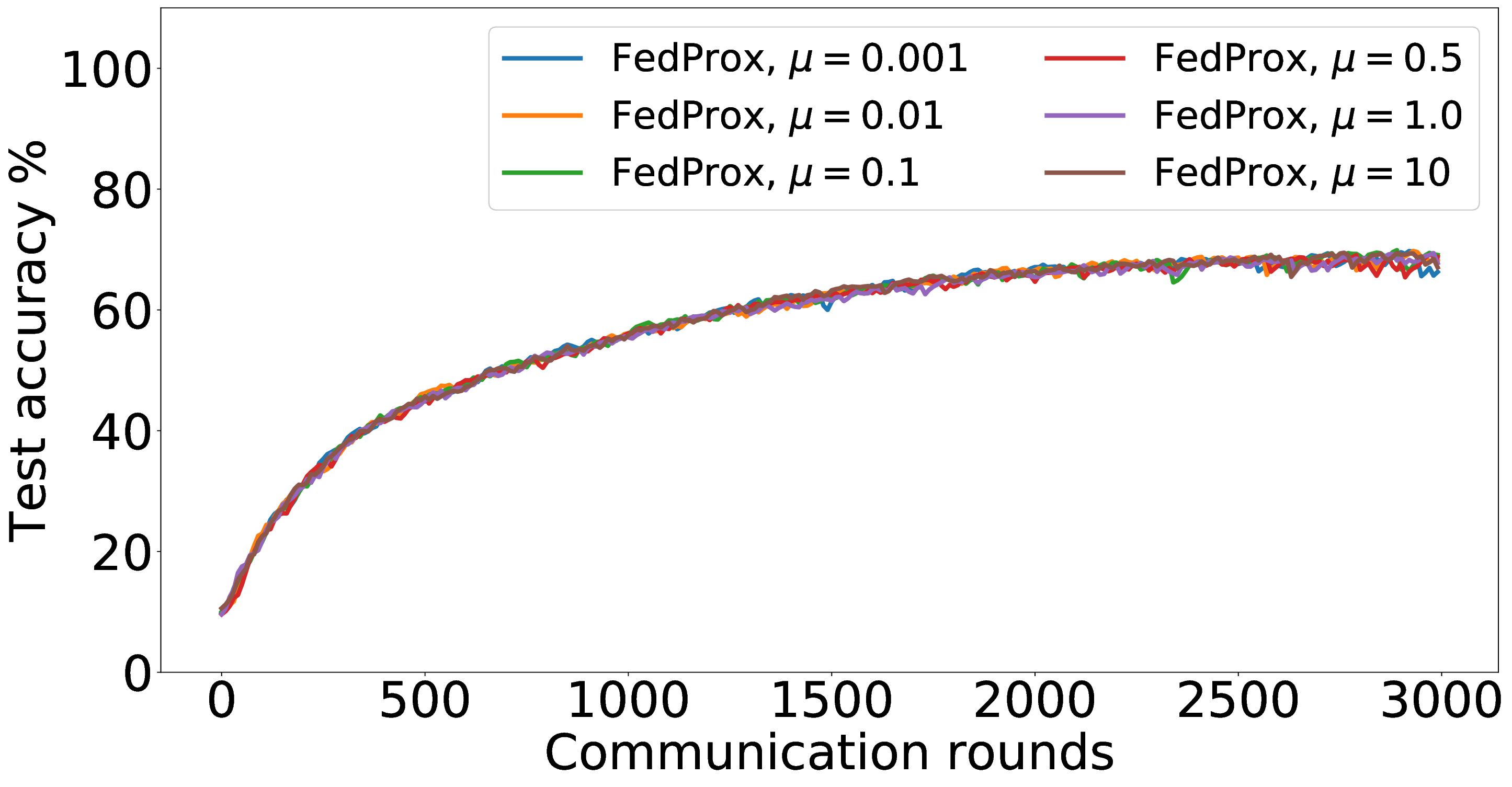}
\includegraphics[width=0.46\columnwidth]{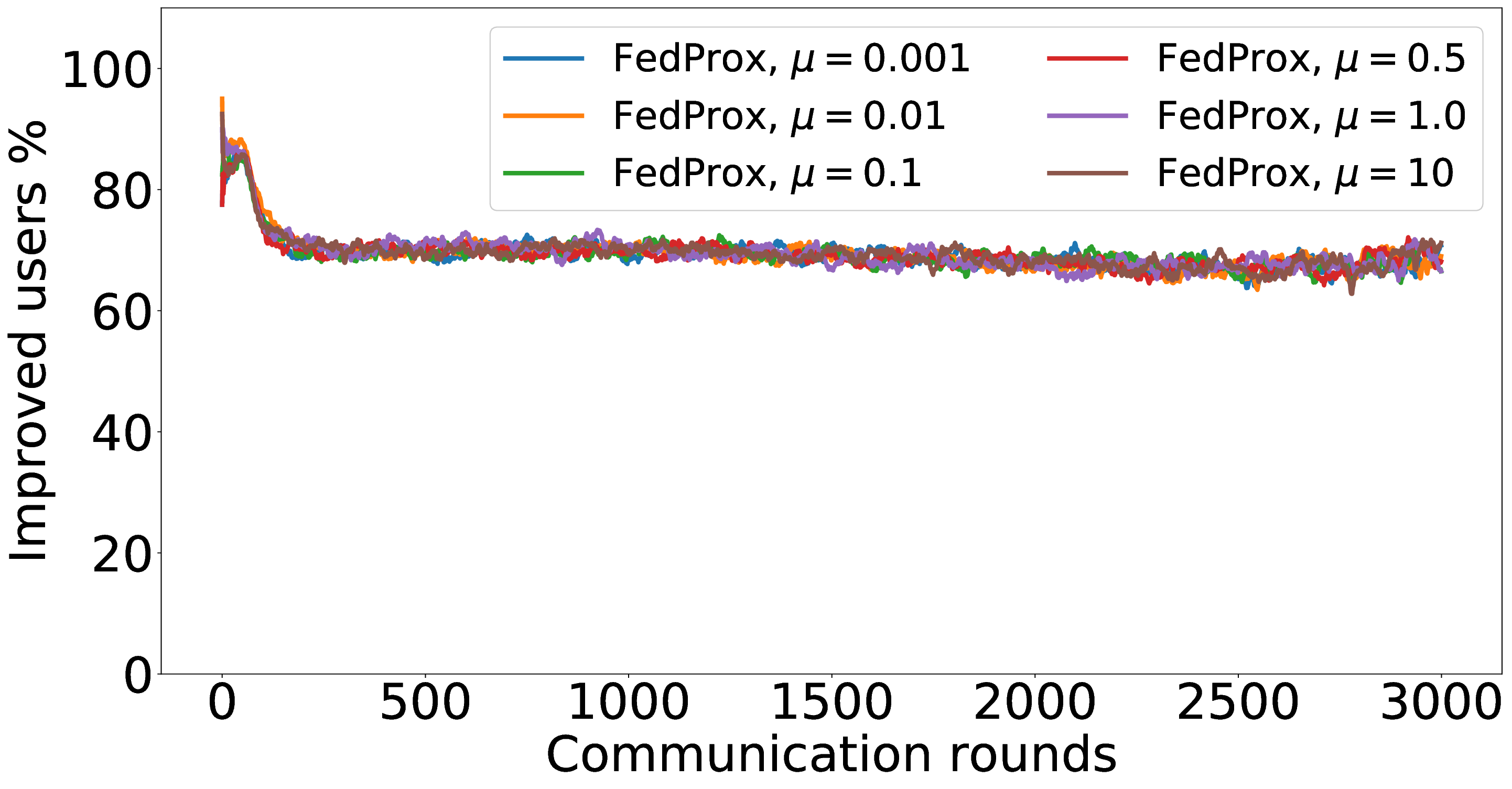}
\includegraphics[width=0.46\columnwidth]{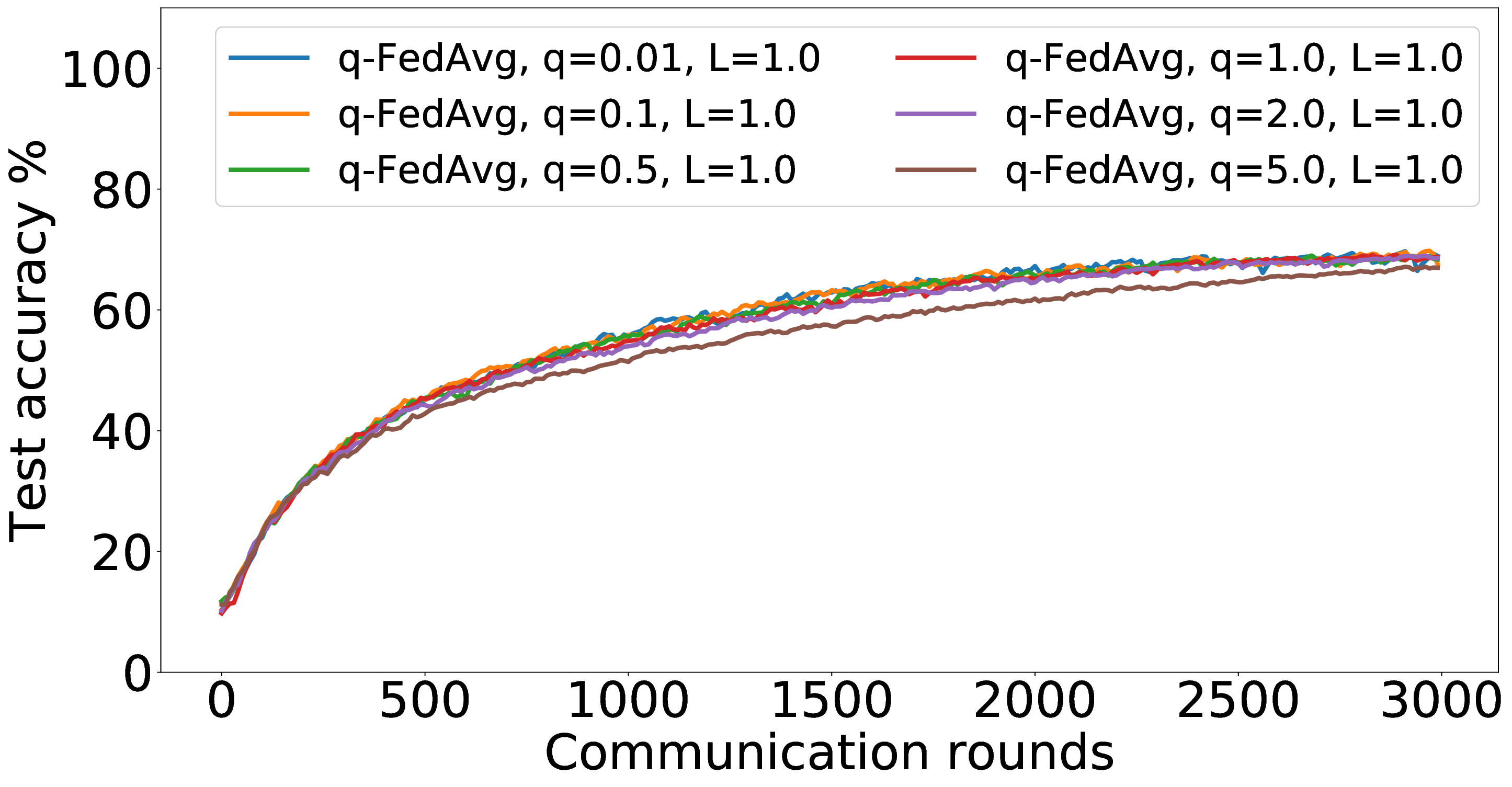}
\includegraphics[width=0.46\columnwidth]{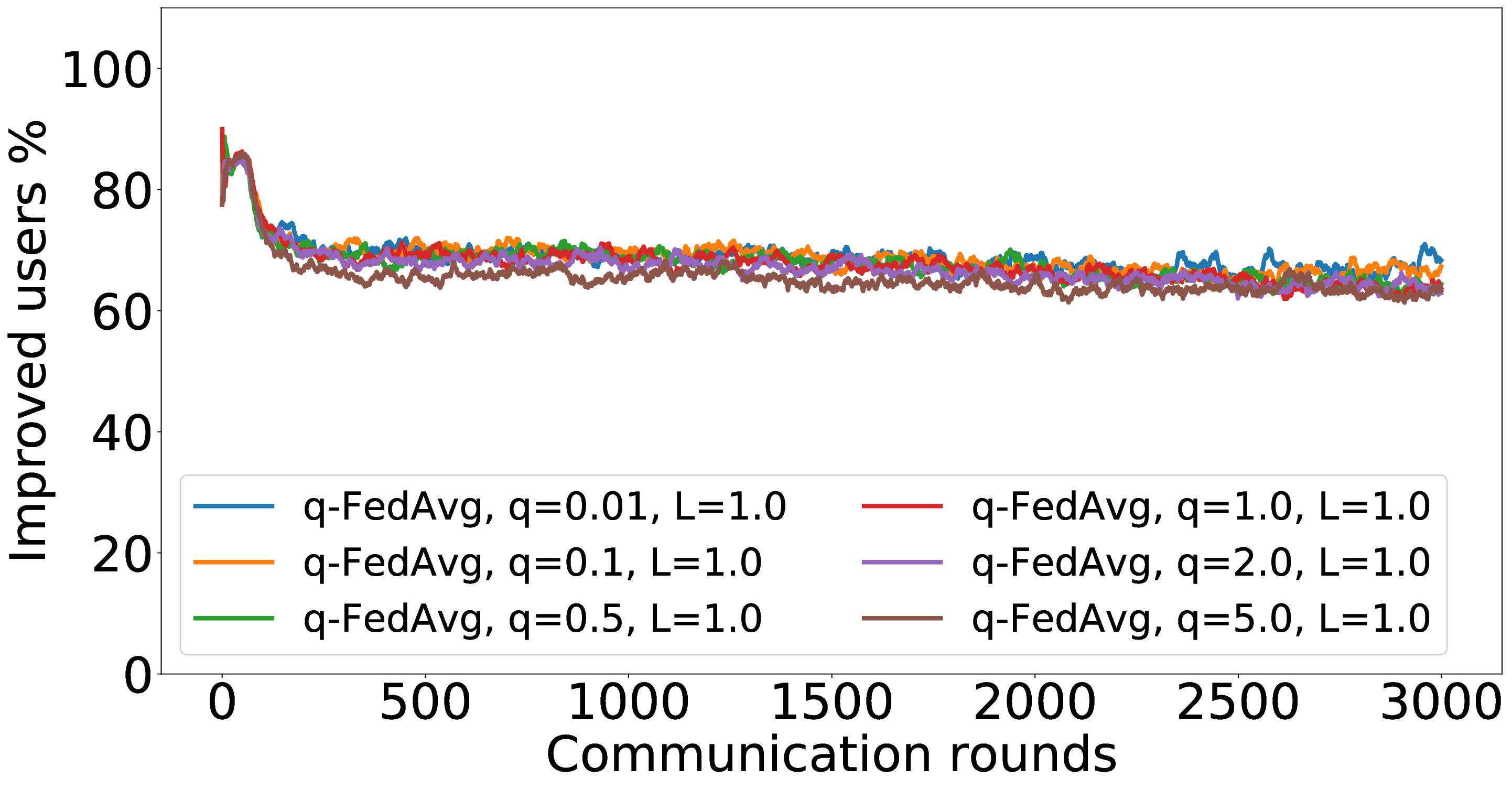}
\caption{The percentage of improved users in terms of training loss and the global test accuracy vs communication rounds on the CIFAR-10 dataset with $p=0.1$ and $b=400$. The results are averaged across $4$ runs with different random seeds.} 
\label{fig:cifar-percentage-2}
\end{figure}

\newpage
\subsection{Results: Shakespeare dataset}
\label{sec:shakespeare_results}
\begin{table}[hbt!]
\footnotesize
\centering
\caption{Test accuracies of users on Shakespeare, where Avg (client) represents average client accuracy while Avg (data) represents average accuracy w.r.t data points. Batch size $10$, full user participation, local learning rate $\eta=\{0.1.0.25,0.5,0.8,1\}$ ($\eta$ represents local learning rate; $\upeta$ represents global learning rate and is only relevant for algorithms with gradient normalization), total communication rounds $200$. \fedprox{}, $\mu=0.001$. The reported statistics are averaged across $4$ runs with different random seeds.
\label{table:shakespeare}}
\begin{tabular}{lll|llllll} \toprule
        \multicolumn{3}{c|}{Algorithm} &  Avg (client) (\%) &  Avg (data) (\%) &  Std. (\%)  & Worst $5$ (\%) & Best $5$ (\%) \\\midrule
        Name            & $\eta$   & momentum     &   \multicolumn{4}{c}{}\\\midrule
        % \fedavg{}         &  $0.1$         &     $0.5$       & $46.97 \pm 0.08$ & $6.40 \pm 0.19$ & $38.46 \pm 0.61$ & $49.77 \pm 1.12$ \\
        \fedavg{}      & $0.25$     & $0.5$ & $45.63 \pm 0.52$       & $48.55 \pm 0.15$ & $5.67 \pm 0.91$ & $39.51 \pm 0.68$ & $50.37 \pm 0.29$\\
        \fedavg{}      & $0.5$     & $0.5$ & $47.17 \pm 0.50$     & $50.11 \pm 0.21$ & $6.22 \pm 0.43$ & $40.37 \pm 1.67$ & $51.89 \pm 0.37$ \\
        \fedavg{}      & $0.8$     & $0$  & $47.16 \pm 0.32$    & $49.95 \pm 0.21$ & $6.90 \pm 1.34$ & $41.04 \pm 0.70$ & $52.15 \pm 0.71$ \\
        \fedavg{}      & $0.8$     & $0.5$  & $46.73 \pm 0.21$    & $49.30 \pm 0.11$ & $6.15 \pm 0.57$ & $40.70 \pm 0.76$ & $51.51 \pm 1.08$ \\
        \fedavg{}      & $1.0$     & $0.5$  & $46.04 \pm 0.27$    &  $48.40 \pm 0.16$ & $6.12 \pm 0.52$ & $39.98 \pm 0.59$ & $50.87 \pm 0.60$ \\\midrule
        Name            & $\eta$     &  momentum         &   \multicolumn{4}{c}{}\\ \midrule
        % \fedprox{}         &  $0.1$         &     $0.5$       & $46.58 \pm 0.20$ & $7.16 \pm 1.42$ & $38.16 \pm 0.69$ & $49.90 \pm 0.71$ \\
        \fedprox{}      & $0.25$     & $0.5$  & $45.72 \pm 0.50$      & $47.28 \pm 0.28$ & $5.75 \pm 0.36$ & $38.95 \pm 1.30$ & $51.00 \pm 1.01$\\
        \fedprox{}      & $0.5$     & $0.5$  & $44.89 \pm 0.48$    & $46.87 \pm 0.10$ & $6.78 \pm 0.54$ & $37.18 \pm 1.58$ & $50.36 \pm 1.29$ \\
        \fedprox{}      & $0.8$     & $0$  & $45.02 \pm 0.25$    & $47.38 \pm 0.20$ & $6.66 \pm 0.89$ & $38.73 \pm 1.93$ & $49.53 \pm 0.87$ \\
        \fedprox{}      & $0.8$     & $0.5$ & $44.52 \pm 0.32$     & $46.25 \pm 0.03$ & $6.08 \pm 0.23$ & $38.30 \pm 1.04$ & $49.02 \pm 1.02$ \\
        \fedprox{}      & $1.0$     & $0.5$ & $43.59 \pm 0.29$     &  $45.98 \pm 0.42$ & $5.99 \pm 0.84$ & $37.75 \pm 1.41$ & $48.41 \pm 0.33$ \\\midrule
        Name            & $\upeta$       & decay       &   \multicolumn{4}{c}{}\\ \midrule
        \fedMGDAn{}         & $10$     & $1$ & $44.29 \pm 0.64$    & $46.11 \pm 0.28$ & $6.88 \pm 0.65$ & $38.42 \pm 1.05$ & $49.82 \pm 0.79$\\
        \fedMGDAn{}         & $12.5$     & $1$ & $44.68 \pm 0.50$    &  $46.71 \pm 0.19$ & $6.53 \pm 0.99$ & $37.93 \pm 1.14$ & $50.36 \pm 1.12$\\ 
        \fedMGDAn{}         & $20$     & $0.8$ & $44.93 \pm 0.42$    &  $47.06 \pm 0.20$ & $8.36 \pm 0.44$ & $37.34 \pm 0.66$ & $52.01 \pm 0.79$\\ 
        \fedMGDAn{}         & $20$     & $1$ & $44.50 \pm 0.61$    &  $47.34 \pm 0.14$ & $8.26 \pm 0.26$ & $37.28 \pm 1.66$ & $51.41 \pm 1.71$\\ \midrule
        Name            & $\upeta$       & decay       &   \multicolumn{4}{c}{}\\ \midrule
        \fedavgn{}         & $2$     & $1$  & $44.32 \pm 0.19$   & $46.06 \pm 0.04$ & $5.73 \pm 0.40$ & $38.76 \pm 0.53$ & $48.79 \pm 0.75$\\
        \fedavgn{}         & $2.5$     & $1$ & $44.33 \pm 0.41$    & $46.53 \pm 0.24$ & $5.78 \pm 0.45$ & $39.09 \pm 0.79$ & $48.58 \pm 0.22$\\
        \fedavgn{}         & $5$     & $0.8$ & $46.07 \pm 0.25$    & $48.01 \pm 0.06$ & $5.39 \pm 0.41$ & $40.26 \pm 0.21$ & $50.68 \pm 0.22$\\
        \fedavgn{}         & $10$     & $0.8$ & $46.93 \pm 0.35$    & $49.01 \pm 0.10$ & $5.85 \pm 0.04$ & $41.00 \pm 0.82$ & $51.37 \pm 0.86$\\
        \bottomrule

\end{tabular}
\vspace{5pt}
\end{table}

\end{document}